\documentclass[twoside,11pt]{article}

\usepackage{jmlr2e}

\usepackage{xcolor}
\usepackage{bm}
\usepackage{marginnote}

\usepackage{amsmath}
\usepackage{amsfonts}
\usepackage{algpseudocode}
\usepackage{bbm}
\usepackage{nicefrac}
\usepackage{extarrows}
\usepackage{mathtools}
\usepackage{xfrac}

\usepackage{algpseudocode}
\usepackage{algorithm}
\usepackage{setspace}
\usepackage{lscape}
\usepackage{subcaption}
\usepackage{tikz}
\usetikzlibrary{intersections}
\usepackage{comment}

\DeclareMathOperator{\E}{\mathbb{E}}

\DeclareMathOperator*{\diag}{diag}

\DeclareMathOperator*{\tr}{tr}
\DeclareMathOperator*{\Tr}{Tr}

\newcommand{\norm}[1]{\left\| #1 \right\|_2}
\DeclarePairedDelimiter{\abs}{\lvert}{\rvert}

\makeatletter
\DeclareRobustCommand{\element}[1]{\@element#1\@nil}
\def\@element#1#2\@nil{%
  #1%
  \if\relax#2\relax\else\MakeLowercase{#2}\fi}
\makeatother

\def \C {\mathbb{C}}
\def \N {\mathbb{N}}
\def \P {\mathbb{P}}
\def \R {\mathbb{R}}

\def \nc {s}

\def \cumu {\kappa}
\def \icacumu {c}
\def \alg {\mathcal{X}}
\def \alge {x}
\def \fnl {\varphi}
\def \mixmq {\bm{Q}}
\def \mixmqq {\bm{q}}

\def \clc {c}

\def \orth {\bm{O}}
\def \ps {{ps}}

\def \permm {\bm{P}}
\def \signm {\bm{S}}
\DeclareMathOperator*{\argmax}{arg\,max}
\def \argkurt {\bm{w}}
\def \Argkurt {\bm{W}}
\def \Argent {\bm{W}}
\def \ent {\chi}
\def \eg{\lambda}
\def \cov {\bm{C}}

\def \cond {\mathcal D}

\ShortHeadings{Free Component Analysis: Theory, Algorithms \& Applications}{WU and NADAKUDITI}
\firstpageno{1}

\singlespace
\begin{document}

\title{Free Component Analysis: Theory, Algorithms \& Applications}

\author{\name Hao Wu \email lingluan@umich.edu \\
       \addr Department of Mathematics\\
       University of Michigan\\
       Ann Arbor, MI 48109, USA
       \AND
       \name Raj Rao Nadakuditi  \email rajnrao@umich.edu \\
       \addr Electrical Engineering  and Computer Science\\
       University of Michigan\\
       Ann Arbor, MI 48109, USA   
       }

\editor{}

\maketitle

\begin{abstract}

We describe a method for unmixing mixtures of freely independent random variables in a manner analogous to the independent component analysis (ICA) based method for unmixing independent random variables from their additive mixtures. Random matrices play the role of free random variables in this context so the method we develop, which we call Free component analysis (FCA), unmixes matrices from additive mixtures of matrices.  Thus, while the mixing model is  standard, the novelty and difference in unmixing performance  comes from the introduction of a new statistical criteria, derived from free probability theory, that quantify freeness analogous to how kurtosis and entropy quantify independence. We describe the theory, the various algorithms, and compare FCA to vanilla ICA which does not account for spatial or temporal structure. We highlight why the statistical criteria make FCA also vanilla despite its matricial underpinnings and show that FCA performs comparably to, and sometimes better than, (vanilla) ICA in every application, such as image and speech unmixing, where ICA has been known to succeed.   Our computational experiments suggest that not-so-random matrices, such as images and short time fourier transform matrix of waveforms are (closer to being) freer ``in the wild'' than we might have theoretically expected.
 
\end{abstract}


\section{Introduction}
Principal component analysis
(PCA) \citep{pearson_1901} is a widely used dimensionality reduction technique in statistical machine learning. The principal components learned by PCA are the directions that maximize the variance, subject to a set of orthogonality constraints. Mathematically speaking, given a (centered) data matrix $\bm{Y} = \begin{bmatrix} \bm{y}_1 & \ldots & \bm{y}_{\nc}\end{bmatrix}^T$, the $i$-th principal component is the solution to the manifold optimization problem
\begin{equation}\label{eq:pca_opt}
  \bm{w}^{{\sf pca}}_{i} = \argmax_{||\bm{w}||_2 = 1} \textrm{variance}(\bm{w}^T \bm{Y}) \textrm{ subject to } \bm{w} \perp \bm{w}_1^{{\sf pca}}, \ldots, \bm{w}^{{\sf pca}}_{i-1}.   
\end{equation}

\subsection{From PCA to ICA via cumulants} 

The variance or the second  cumulant \citep{cornish1938moments} of a random variable $x$ is defined as 
\begin{equation}\label{eq:c2}
\clc_2(x) = \textrm{variance}(x) := \mathbb{E}[x^2] - \left(\mathbb{E}[x]\right)^2
\end{equation}
Substituting (\ref{eq:c2}) into (\ref{eq:pca_opt}) allows us to cast PCA as a maximization of the second cumulant:
\begin{equation}
  \bm{w}^{{\sf pca}}_{i} = \argmax_{||\bm{w}||_2 = 1}  \clc_2(\bm{w}^T \bm{Y}) \textrm{ subject to }  \bm{w} \perp \bm{w}_1^{{\sf pca}}, \ldots, \bm{w}^{{\sf pca}}_{i-1}.   \label{eq::pca_opt2}
\end{equation}

Independent component analysis (ICA) \citep{comon_1994, hyvarinen_2004} is a signal separation and dimensionality reduction technique that is obtained by replacing (in our notation) $\clc_2(\cdot)$ on the right hand side of  (\ref{eq::pca_opt2}) by the fourth cumulant $\clc_4(\cdot)$,  thereby yielding the optimization problem 
\begin{equation} \label{eq:ica_opt} 
  \bm{w}^{{\sf ica}}_{i} = \argmax_{ \norm{\bm{w}} = 1}  \left \lvert \clc_4(\bm{w}^T \bm{Y}) \right \rvert \textrm{ subject to }  \bm{w} \perp \bm{w}_1^{{\sf ica}}, \ldots, \bm{w}^{{\sf ica}}_{i-1}.      
\end{equation}
The fourth cumulant $\clc_4(\cdot)$ of a scalar random variable $x$ is equivalent to its kurtosis \citep{chissom1970interpretation,cornish1938moments}, and when $\mathbb{E}[x] = 0$ it is given by 
\citep[Eq. (6)]{smith1995recursive}
\begin{equation}
\clc_4(x) = \textrm{kurtosis}(x) := \mathbb{E}[x^4] - 3 \left(\mathbb{E}[x^2]\right)^2.
\end{equation}
We refer to the formulation in (\ref{eq:ica_opt}) as kurtosis, or $\clc_4$-ICA in short.  Replacing $\clc_4(\cdot)$ on the right hand side of (\ref{eq:ica_opt}) with the $\clc_j(\cdot)$ for integer $j \geq 3$ yields  $\clc_j$-ICA.  There are other formulations of ICA involving different objective functions, such as for example any non-quadratic, well-behaving even function as in \citep{hyvarinen1997one,hyvarinen1997family}; see 
\citep{comon_1994} for a discussion on other such contrast functions.

\subsection{Known result: ICA  unmixes mixtures of independent random variables} 

Suppose we are given a multivariate vector $\bm{z}$ modeled as 
\begin{equation}\label{eq:general model ica}
    \underbrace{\begin{bmatrix} z_1  \\ \vdots \\ z_s \end{bmatrix}}_{=:\bm{z}} =\underbrace{\begin{bmatrix}\bm{a}_1 & \cdots & \bm{a}_\nc \end{bmatrix}}_{=: \bm{A}} \underbrace{\begin{bmatrix} x_1 \\ \vdots \\ x_s   \end{bmatrix}}_{=: \bm{x}}, 
    \end{equation}
where $\bm{A}$ is a non-singular $s \times s$  mixing matrix and $\bm{{x}}$ is a vector of independent scalar-valued random variables. Assume, without loss of generality, that  $\mathbb{E}[\bm{x}] = \bm{0}$ and  $\mathbb{E}[\bm{x}\bm{x}^T] = \bm{I}$. Let $\bm{A} = \bm{U} \bm{\Sigma} \bm{V}^T$ be the singular value decomposition (SVD) of the mixing matrix. Then, we have that
$$
\cov_{\bm{z}\bm{z}} := \mathbb{E}[\bm{z}\bm{z}^T] = \bm{A} \mathbb{E}[\bm{x}\bm{x}^T] \bm{A}^T = \bm{A}\bm{A}^T = \bm{U} \bm{\Sigma}^2 \bm{U}^T.
$$
The whitened vector $\bm{y} = \cov_{\bm{z} \bm{z}}^{-1/2} \bm{z}$ has identity covariance and can be rewritten in terms of the SVD of $\bm{A}$ as 
\begin{equation}\label{eq:ica model 2}
    \bm{y} = \cov_{\bm{z} \bm{z}}^{-1/2} \bm{z} =  \bm{U} \bm{\Sigma}^{-1} \bm{U}^T  \bm{U} \bm{\Sigma} \bm{V}^T \bm{x} = \underbrace{\left(\bm{U} \bm{V}^T\right)}_{=:    \mixmq} \bm{x}.
\end{equation}
Note that $\mixmq = \bm{U} \bm{V}^T$ in (\ref{eq:ica model 2}) is an orthogonal matrix, because $\bm{U}$ and $\bm{V}$ are orthogonal matrices derived from the SVD of $\bm{A}$. Equation (\ref{eq:ica model 2}) thus reveals that the whitened vector $\bm{y}$ is related to the latent independent random variables that we wish to unmix via an orthogonal transformation. If we can estimate $\mixmq$ from $\bm{y}$, we can unmix the independent random variables by computing $\widehat{\bm{W}}^T \bm{y}$ provided $\widehat{\bm{W}} = \mixmq \bm P \bm S$ where $\bm P$ is a permutation matrix and $\signm$ is a diagonal matrix with $\pm 1$ as diagonal elements.

It is a remarkable fact \citep{comon_1994, hyvarinen_2004}   that, generically, for $\bm{y}$ modeled as in (\ref{eq:ica model 2}),  $\clc_4$-ICA as in (\ref{eq:ica_opt}) returns $\bm{W}_{\sf ica}$ such that $\bm{W}_{\sf ica}^T \bm{y}$ unmixes the mixed independent random variables. Thus ICA can be viewed as a procedure for unmixing sums of independent random variables from each other. 

The caveat of $c_4$-ICA is that no more than one of the independent random variables is Gaussian, and that the random variables do not all have a kurtosis identically equal to zero. The latter condition rules out the use of $\clc_k$-ICA for odd $k>2$ because the cumulants of a symmetric random variable are identically equal to zero, so that we would not be able to unmix a large class of random variables. 

Replacing $\clc_4$-ICA with $\clc_k$-ICA for even $k > 4$ would still not allow us to unmix more than one Gaussian random variable: this is a fundamental limit of ICA \citep[Section 2]{comon_1994}. \citep{cardoso1999high} discusses aspects related to the use of higher order contrast functions for ICA while \cite{chen2006efficient} address the important issue of the statistical efficiency of ICA estimators in the presence of limited samples.

\subsection{Our contribution: From ICA to FCA via free cumulants}

Free probability theory  is a mathematical theory developed by Voiculescu \cite{voiculescu1993analogues, Voiculescu_1994, Voiculescu_1995, Voiculescu_1997} that is a counterpart of classical probability theory, except that the random variables are  non-commutative in a manner that scalar random variables are not. 

As a probability theory, it is not surprising that free probability possesses (i) a linear functional $\varphi(\cdot)$ mapping the random variables to scalar, which plays the same role as the expectation operator $\mathbb{E}[\cdot]$ in classical probability theory,
(ii) an analog of the notion of classical independence, which is so-called ``freeness'' or free independence. 
Since free probability is designed for non-commutative random variables, intuitions that is natural in classical probability is not longer granted: for free independent $x_1$ and $x_2$, the mixed moments 
$$\varphi[(x_1 x_2)^2] = \varphi[x_1 x_2 x_1 x_2]$$
is not necessarily equal to $\varphi(x_1^2) \cdot \varphi(x_2^2)$, since $x_1 x_2 x_1 x_2 \neq x_1^2 x_2^2$ whenever $x_1$ and $x_2$ are assumed to be non-commutative. Readers are referred to Appendix \ref{sec:math_prelim} for a self-contained introduction to free probability and how it differs from classical probability.

Free probability, via free independence, provides a recipe for computing such mixed moments of freely independent random variables in a manner that is analogous to but different from classical probability theory. For our purpose here, there is a notion of free cumulants $\cumu(\cdot)$ for integer $m$ which exhibit the same properties as the classical cumulants (see Theorem \ref{thm:vanish_mix_cumu} and  \eqref{eq:freeadditivity} in Appendix \ref{subsec:prelim self adjoint}). This allows us to  cast FCA analogous to the ICA in (\ref{eq:ica_opt})  as a fourth free cumulant maximization problem of the form
\begin{equation} \label{eq:fca_opt} 
  \bm{w}^{{\sf fca}}_{i} = \argmax_{\norm{\bm{w}} = 1} \left \lvert \kappa_4(\bm{w}^T \bm{y}) \right \rvert \textrm{ subject to }  \bm{w} \perp \bm{w}_1^{{\sf fca}}, \ldots, \bm{w}^{{\sf fca}}_{i-1},   
\end{equation}
where $\kappa_4(\cdot)$ is the fourth free cumulant.  We can similarly formulate $\kappa_m$-FCA for $m \geq 3$ as we did for ICA. 

This is also where we depart from ICA in another crucial sense. We can model the random variables as self-adjoint (or symmetric) or non-self adjoint (or rectangular/non-symmetric), which give us self-adjoint and rectangular variants of FCA, respectively. \cite{voiculescu1993analogues} developed free probability theory for self-adjoint random variables; \cite{ florent_free_probability} extended it to rectangular random variables.  

In the self-adjoint setting $\kappa_4(\cdot)$ is given by  (\ref{eq:free_kurtosis_expfor}) while in the non-self adjoint (or rectangular, in a sense we shall shortly see) setting $\kappa_4(\cdot)$ is given by   (\ref{eq:def_kurt_rec}).

The development and analysis of algorithms for self-adjoint and rectangular FCA is the main contribution of this paper.

\subsection{Our main finding: FCA unmixes mixtures of free random variables} 

If we whiten the vector $\bm{z}$ as in (\ref{eq:ica model 2}) with the covariance matrix defined via the $\varphi(\cdot)$ operator as in Definition \ref{def:cov}, then we show that $\kappa_4$-FCA, just as $c_4$-ICA, returns $\bm{W}_{\sf fca} = \mixmq \bm P \bm S$ (see Theorem \ref{thm:kurt3}), and thus $\bm{W}_{\sf fca}^T \bm{y}$ unmixes the mixed free random variables.

The caveat of $\cumu_4$-FCA, analogous to the $c_4$-ICA algorithm, is that  no more than one of the free random variables can be the free probabilistic equivalent of the classical Gaussian random variable, and that the random variables do not all have a free kurtosis equal to zero. In the self-adjoint setting, the free analog of the Gaussian is the free semi-circular element \citep{free_entropy_HP} while in the rectangular setting, it is the free Poisson element \citep{florent_free_probability}.

Just as for ICA, the condition that the free kurtosis of the free random variables cannot all  equal to zero rules out the use of $\kappa_m$-FCA for odd valued $m \geq 3$  in the self-adjoint setting, because the free cumulants  of a symmetric free random variable are identically equal to zero and so we would not be able to unmix a large class of free random variables with symmetric distribution. 
On the other hand, for rectangular free random variables, cumulants odd orders are zeros by default (see \cite[(b), pp. 6]{florent_free_entropy}).

We will prove that just as in the ICA setting, replacing $\kappa_4$-FCA with $\kappa_m$-FCA for even valued $m \geq 4$ would still not allow us to unmix more than one Gaussian analog free random variable: this is a fundamental limit of FCA. Thus FCA fails whenever we have more than one free Gaussian analogs mixed together. This is the fundamental limit of FCA.

The free semi-circular element in the self-adjoint setting, and  the Poisson element in the rectangular case, are the only non-commutative random variables with higher order kurtosis equal to zero, analogous to the Gaussian in the scalar setting. Thus, we might say that FCA finds directions that maximize deviation from the semi-circularity (or Poissonity) when the random variables are self-adjoint (or rectangular, respectively).

We also develop an algorithm for FCA based on the maximization of the free entropy for both the self-adjoint \citep{voiculescu1993analogues, free_entropy_HP} and rectangular settings \citep{florent_free_entropy}, and show that FCA successfully unmixes free random variables in a similar way.  Table \ref{table:fcathems} summarizes our results. 

\begin{table}[h]
\begin{center}
\caption{FCA algorithms and their limits.}
\label{table:fcathems}
\begin{tabular}{|c|c|c|c|c|}
\hline
& \shortstack{\\ self-adjoint\\ FCA \\ (free kurtosis) } & \shortstack{self-adjoint\\ FCA \\ (free entropy) } & \shortstack{rect. FCA \\ (free rect.\\ kurtosis) } &\shortstack{rect. FCA \\ (free rect.\\ entropy) } \\
                                                      \hline
\hline
\shortstack{\\ Recovery \\ Guarantee} &  \shortstack{Theorem \ref{thm:kurt3}\\ \vspace{0.5ex} }& \shortstack{Theorem \ref{thm:ent} \\ \vspace{0.5ex}} & \shortstack{ Theorem \ref{thm:kurt3} \\ \vspace{0.5ex}}  &   \shortstack{Theorem \ref{thm:ent}\\ \vspace{0.5ex} }      \\
\hline
\shortstack{Identifiability\\ Condition \\ \vspace{0.5ex}} &   \shortstack{\\At most one \\ component\\ with $\cumu_4 = 0$} &\shortstack{\\ At most one \\ free semicircular\\ element}  & \shortstack{At most one \\ component\\ with $\cumu_4 = 0$} &   \shortstack{At most one\\  free Poisson\\  element}           \\ 
\hline
\end{tabular}
\end{center}
\end{table}

\subsection{Insight: FCA unmixes mixtures of 
(asymptotically) free random matrices} 
Voiculescu \citep{voiculescu1991limit, mingo2017free} showed that symmetric random matrices are good models for asymptotically free self-adjoint random variables (also see Appendix \ref{sec:connection sym}). The non-commutativity comes in because matrix multiplication is non-commutative. Florent \citep{florent_free_probability} showed that rectangular random matrices are good models for asymptotically free rectangular random variables (also see Appendix \ref{subsubsec:recmat}). 

In the self-adjoint setting, Voiculescu showed that random matrices $\bm{X}_1$ and $\bm{X}_2$  are asymptotically free whenever $\bm{X}_1$ and $\bm{X}_2$ are independent of each other if one, or both, of the random matrices have isotropically random (or Haar distributed) eigenvectors. In the non-self-adjoint or rectangular setting, Benaych-Georges showed analogous that rectangular random matrix $\bm{X}_1$ and $\bm{X}_2$ are free whenever they are independent of each other and if the singular vectors of one or both of the random matrices are Haar distributed. Since these pioneering works, many authors have relaxed the conditions and broadened the class of random matrices that we now know to be asymptotically free   -- see, for example the work of \cite{male2011traffic} and \cite{anderson2014asymptotically}. 
We can thus consider the matrix mixing model 
\begin{equation}\label{eq:general matrix model}
\underbrace{\begin{bmatrix} \bm{Z}_1  \\ \vdots \\ \bm{Z}_s \end{bmatrix}}_{=:\bm{Z}}
     =
     \underbrace{\begin{bmatrix}a_{11} \bm{I} & \ldots & a_{1 \nc } \bm{I} \\ \vdots & \ldots & \vdots \\ a_{\nc 1} \bm{I} & \ldots & a_{\nc \nc} \bm{I} \end{bmatrix}}_{=: \bm{A} \otimes \bm I }
     \underbrace{\begin{bmatrix} \bm{X}_1 \\ \vdots \\ \bm{X}_s   \end{bmatrix}}_{=: \bm{X}}, 
\end{equation}

When the matrices $\bm{X}_1, \ldots, \bm{X}_s \in \R^{N \times N}$ are symmetric or Hermitian, then we are in the self-adjoint setting. Voiculescu \citep{voiculescu1991limit} showed that the appropriate linear function $\varphi(\cdot)$ is exactly the normalized trace function. That is,
\begin{equation}\label{eq:sym varphi}
\varphi(\bm{X}_i) = \dfrac{1}{N} \textrm{Tr}(\bm{X}_i)
\end{equation}
and
\begin{equation}\label{eq:rect varphi}
\varphi(\bm{X}_i \bm{X}_{j}) =  \dfrac{1}{N} \textrm{Tr}(\bm{X}_i \bm{X}_j).
\end{equation}
Replacing this with their sample analogs gives us a concrete algorithm for self-adjoint FCA; see Algorithm \ref{alg:free_whiten} and Algorithm \ref{alg:fcf_pro}.

When the matrices $\bm{X}_1, \ldots, \bm{X}_s \in \R^{N \times M} (N \leq M)$ are square but not self-adjoint or just non-square (we will call this type of matrices rectangular for the rest of paper), we are in the non-self-adjoint setting \citep{florent_free_probability}. Then the appropriate pair of the linear functionals $\varphi_1(\cdot)$ and $\varphi_2(\cdot)$ are exactly the normalized trace functions in $\R^{N \times N}$ and $\R^{M \times M}$:
$$\varphi_1(\bm{X}_i \bm{X}_j^H) =  \dfrac{1}{N} \textrm{Tr}(\bm{X}_i \bm{X}_j^H)$$
and
$$\varphi_2(\bm{X}_i^H \bm{X}_j) = \dfrac{1}{M} \textrm{Tr}(\bm{X}_i^H \bm{X}_j).$$

Thus we expect that asymptotically, FCA should unmix asymptotically free random matrices. In the setting where the random matrices are large but finite, we expect FCA to approximately unmix the asymptotically free random matrices, with some non-zero but small unmixing error, analogous to the finite sample unmixing performance of ICA \citep{ilmonen2010new,frieze1996learning,arora2012provable}. We will use numerical simulations to demonstrate that FCA can near perfectly unmix mixtures of large, finite sized (asymptotically free) matrices - see Sections \ref{sec:sim self adjoint} and \ref{sec:rec_rm_separation}.
    
\begin{figure}[t]
          \centering
            \begin{subfigure}[t]{0.24\textwidth}
              \includegraphics[width = \textwidth,trim={0cm 0cm 0cm 0cm},clip]{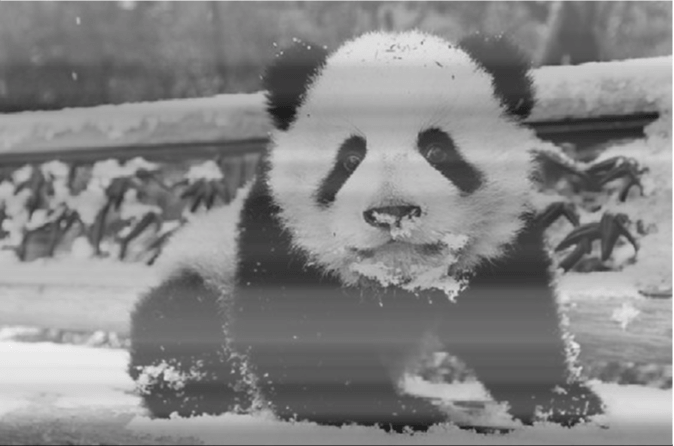}
          \caption{Panda}
            \label{fig:panda}
            \end{subfigure}
            \begin{subfigure}[t]{0.24\textwidth}
              \includegraphics[width = \textwidth,trim={0cm 0cm 0cm 0cm},clip]{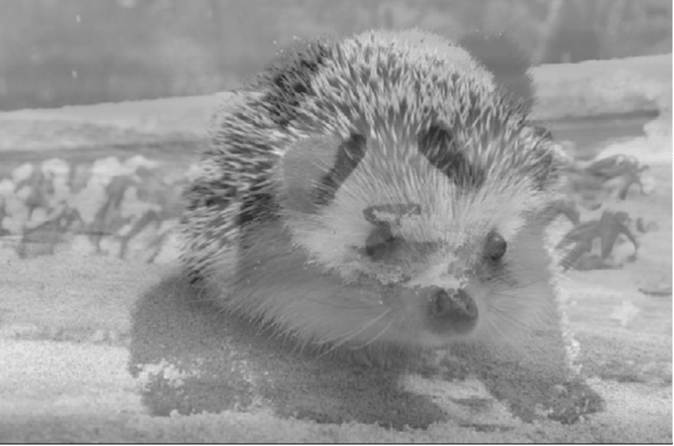}
          \caption{Mixed Image1}
            \label{fig:panmix1}
            \end{subfigure}  
            \begin{subfigure}[t]{0.24\textwidth}
              \includegraphics[width = \textwidth,trim={0cm 0cm 0cm 0cm},clip]{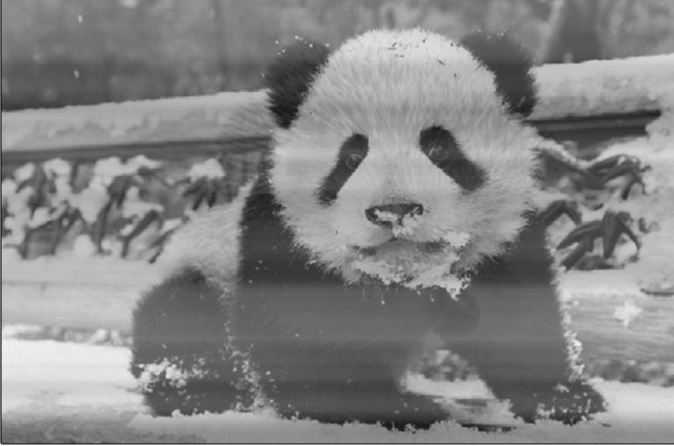}
          \caption{Image1 via ICA}
            \label{fig:pandaica}
            \end{subfigure}
            \begin{subfigure}[t]{0.24\textwidth}
              \includegraphics[width = \textwidth,trim={0cm 0cm 0cm 0cm},clip]{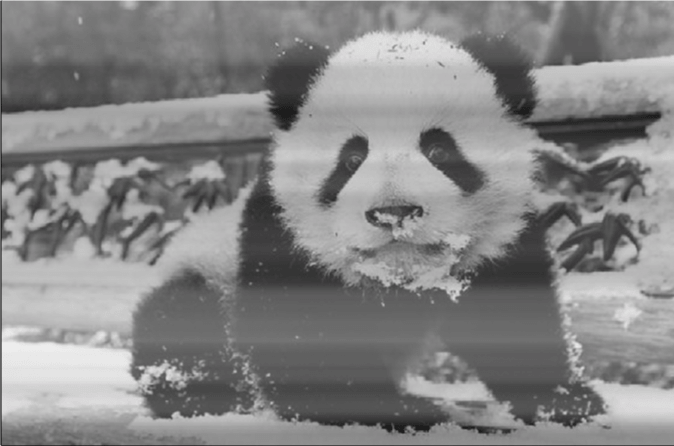}
          \caption{Image1 via FCA}
            \label{fig:pandafca}
            \end{subfigure}
            \begin{subfigure}[t]{0.24\textwidth}
              \includegraphics[width = \textwidth,trim={0cm 0cm 0cm 0cm},clip]{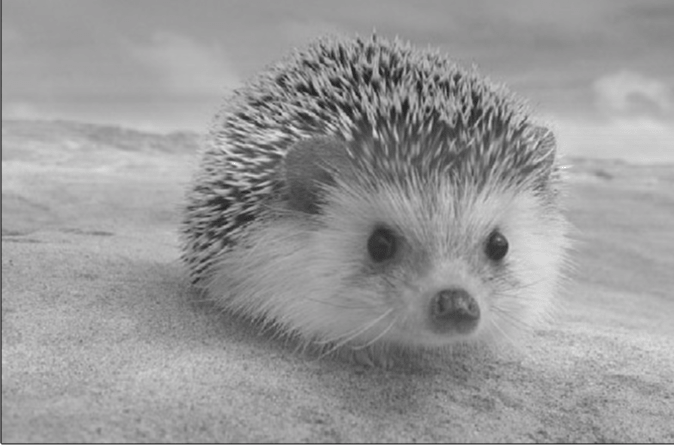}
          \caption{Hedgehog}
            \label{fig:hedgehog}
            \end{subfigure} 
            \begin{subfigure}[t]{0.24\textwidth}
              \includegraphics[width = \textwidth,trim={0cm 0cm 0cm 0cm},clip]{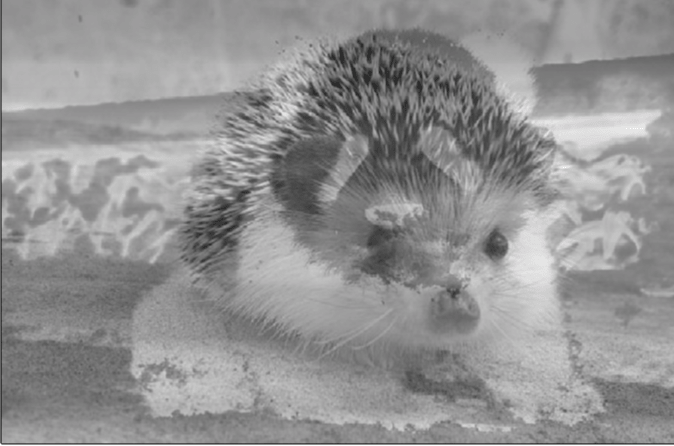}
          \caption{Mixed Image2}
            \label{fig:pandamix2}
            \end{subfigure}
            \begin{subfigure}[t]{0.24\textwidth}
              \includegraphics[width = \textwidth,trim={0cm 0cm 0cm 0cm},clip]{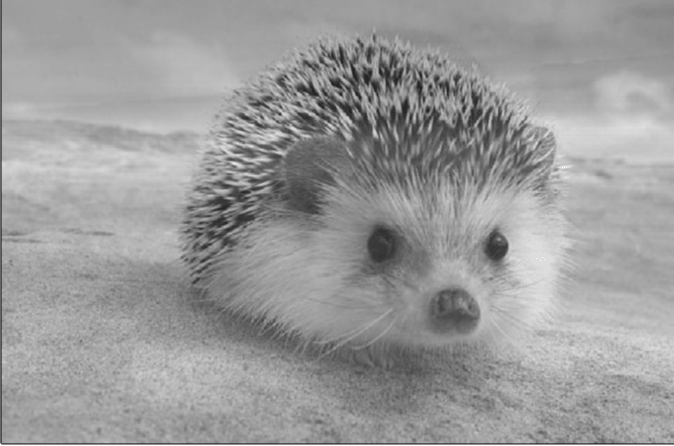}
          \caption{Image2 via ICA}
            \label{fig:hedgeica}
            \end{subfigure}
            \begin{subfigure}[t]{0.24\textwidth}
              \includegraphics[width = \textwidth,trim={0cm 0cm 0cm 0cm},clip]{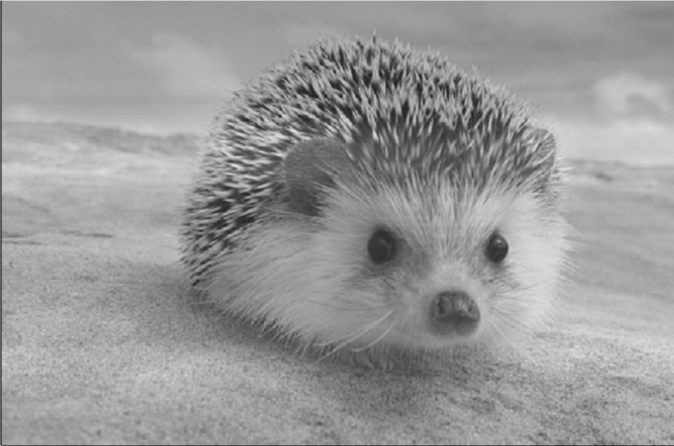}
          \caption{Image2 via FCA}
            \label{fig:hedgefca}
            \end{subfigure}
            \caption{An experiment in image separation using ICA and FCA. Note that subplots (c), (g) (unmixed images via ICA) and (d), (h) (unmixed images via FCA) both recover (a), (e) respectively. Here, $\bm{A} = [\sqrt{2}, \sqrt{2}; -\sqrt{2}, \sqrt{2}]/2$ in \eqref{eq:general matrix model}.
            The error of ICA is $6.08\times 10^{-2}$ while the error of FCA is $2.69\times 10^{-2}$. See \eqref{eq:error} for the definition of the error.
            }
            \label{fig:panhogmix}
            \end{figure}

\begin{figure}[t]
  \centering
    \begin{subfigure}[t]{0.24\textwidth}
      \includegraphics[width = \textwidth,trim={0cm 0cm 0cm 0cm},clip]{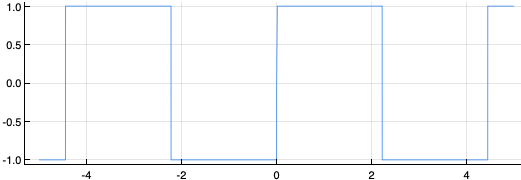}
  \caption{Square wave}
    \label{fig:wave1x1}
    \end{subfigure}
    \begin{subfigure}[t]{0.24\textwidth}
      \includegraphics[width = \textwidth,trim={0cm 0cm 0cm 0cm},clip]{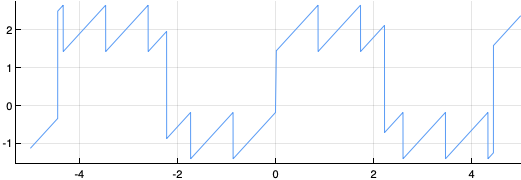}
  \caption{Mixed wave1}
    \label{fig:wave1z1}
    \end{subfigure}  
    \begin{subfigure}[t]{0.24\textwidth}
      \includegraphics[width = \textwidth,trim={0cm 0cm 0cm 0cm},clip]{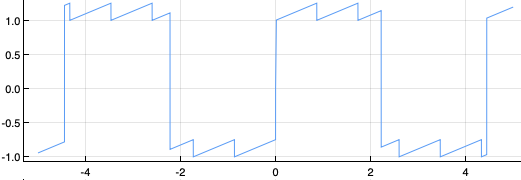}
  \caption{Wave1 via ICA}
    \label{fig:wave1ica1}
    \end{subfigure}
    \begin{subfigure}[t]{0.24\textwidth}
      \includegraphics[width = \textwidth,trim={0cm 0cm 0cm 0cm},clip]{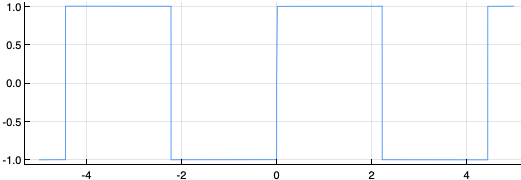}
  \caption{Wave1 via FCA}
    \label{fig:wave1fca1}
    \end{subfigure}
    \begin{subfigure}[t]{0.24\textwidth}
      \includegraphics[width = \textwidth,trim={0cm 0cm 0cm 0cm},clip]{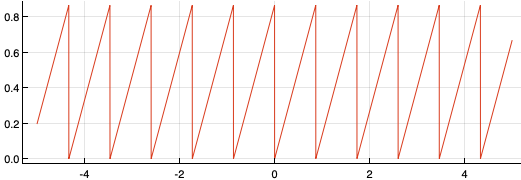}
  \caption{Sawtooth wave}
    \label{fig:wave1x2}
    \end{subfigure} 
    \begin{subfigure}[t]{0.24\textwidth}
      \includegraphics[width = \textwidth,trim={0cm 0cm 0cm 0cm},clip]{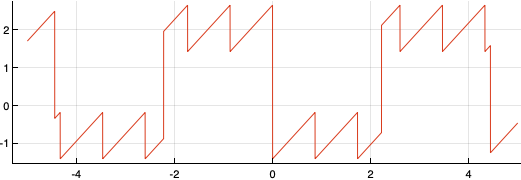}
  \caption{Mixed wave2}
    \label{fig:wave1z2}
    \end{subfigure}
    \begin{subfigure}[t]{0.24\textwidth}
      \includegraphics[width = \textwidth,trim={0cm 0cm 0cm 0cm},clip]{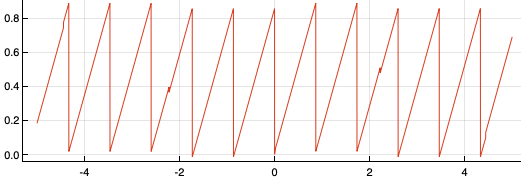}
  \caption{Wave2 via ICA}
    \label{fig:wave1ica2}
    \end{subfigure}
    \begin{subfigure}[t]{0.24\textwidth}
      \includegraphics[width = \textwidth,trim={0cm 0cm 0cm 0cm},clip]{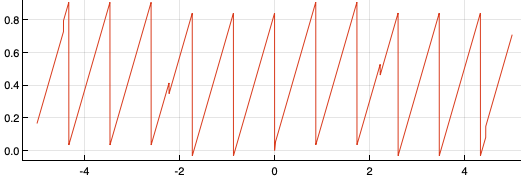}
  \caption{Wave2 via FCA}
    \label{fig:wave1fca2}
    \end{subfigure}
    \caption{An experiment in waveform separation using ICA and FCA. Note that subplots (c), (g) (unmixed waves via ICA) and (d), (h) (unmixed waves via FCA) both recover (a), (e). Visually, FCA performs better in this experiment. In this simulation, $\bm{A} =[\sqrt{2}, \sqrt{2}; -\sqrt{2}, \sqrt{2}]/2$ in (\ref{eq:general matrix model}). 
    The errors for ICA and FCA are $9.95\times 10^{-2}$ and $6.66\times 10^{-2}$ respectively. 
    }
\label{fig:wave1}
\end{figure}

\subsection{Observation: FCA sometimes better unmixes random variables than ICA}

In the  examples in Figures \ref{fig:panhogmix} and \ref{fig:spectrogram_sep}, FCA did better than ICA in a quantitative sense. Figure \ref{fig:locustdenoise} shows a setting where we are unmixing two mixed images and where one of the images corresponds to a Gaussian random matrix. In this setting, FCA performs better than ICA in a visually perceptible way. We have observed that FCA usually does at least as well as ICA and sometimes better. 

In a similar setting, we replace the locust image by a matrix $\bm U\bm D\bm V^T$ in SVD form, where $\bm U$, $\bm V$ are Discrete Cosine Transformation (DCT) matrices and $\bm D$ is a diagonal matrix (see Section \ref{sec:rec_rm_separation}).
This matrix model enables us to increase the dimension and compare the asymptotic behavior of ICA and FCA.
Our numerical simulations show that $\kappa_4$-FCA and $c_4$-ICA perform similarly. 
However, we observe that  free entropy based FCA significantly outperforms ICA (see Figure \ref{fig:fca comparison}) at the cost of increased  computational complexity, since estimating the free entropy involves eigenvalue (or singular value) computation, which are of order $O(N^3)$.

\subsection{Insight: FCA can be applied wherever vanilla ICA has been applied}

ICA has been successfully applied  to image unmixing, audio separation and waveform unmixing problem  \citep{lee1998independent, mitsui2017blind}. Here we show  that \emph{FCA can be successfully applied wherever ICA has succeeded}, including in settings where there are seemingly no matrices in sight.

Figure \ref{fig:panhogmix} showcases the successful use of FCA for unmixing mixed images. This is a natural place to apply FCA because (grayscale) images are matrices. Applying ICA to unmix the images involves vectorizing the images, and treating them as mixed scalar random variables in a way that ignores the spatial matrix information that FCA uses. Perhaps it is therefore not surprising that FCA can outperforms ICA.

What \emph{is} surprising is that the images in Figure \ref{fig:panhogmix} \emph{are not}  textbook examples of asymptotically free random matrices. By this we mean that would not have predicted that the panda and hedgehog matrices are free according to the definition in Appendix \ref{sec:math_prelim}. One might even argue that they are not really random matrices.  And yet, FCA unmixes them as though they are free. For this and many, many other examples of mixed natural images. It is as though matrices in the wild are free-er than we might expect. We hope that experiments with FCA and computational reasoning on its unexpected successes can guide  free probabilists looking to expand the class of matrix models for which freeness holds.

Figures \ref{fig:wave1}  and \ref{fig:spectrogram_sep} show examples where we are trying to unmix mixed deterministic waveforms and audio signals respectively. ICA is known to succeed in these examples, and it is natural to apply ICA here since the latent variables are scalar valued. FCA seems unnatural because there are no matrices in sight, let alone mixed matrices!

The surprising insight is that if we compute the short time fourier transform (STFT) matrix of the mixed signals, then the matrix mixing model in  (\ref{eq:general matrix model}) is with respect to the STFT matrix of the mixed signals: we can use FCA to unmix the signals! Here, FCA on the STFT embdedding outperforms ICA. We might compute other matrix embeddings (say via the short time wavelet transform) and apply FCA there. We do not (yet) have a theory to predict which embedding would lead to better unmixing; nonetheless, the important point is that by embedding scalar valued signals as matrices, we can apply FCA wherever ICA has been applied, and that we can also possibly get better (or worse -- see Figure \ref{fig:wave2}) unmixing performance by varying the matrix embedding. 

Figure \ref{fig:diagmram} summarizes our worldview on this and our sense that there is a theory waiting to be fully revealed on the relation between non-asymptotic recovery of mixed variables and a to-be-defined notion of distance to the various notions of freeness and independence that can provide a principled way to reason about whether ICA or FCA will better unmix the mixed variables.  What we wish to emphasize is that by simply changing the statistical criterion to one that is more matrix-centric (via free probability), one is accessing different embeddings than vanilla ICA can/does.


\subsection{Discussion: Why do we only compare FCA with vanilla ICA?}
There are variants of ICA that explicitly account for spatial and temporal structures in a way that vanilla ICA does not -- see for example, \cite[Chapters 2, 7, 10, 11]{comon2010handbook}.  A natural question arises: should we be comparing FCA to the not-so-vanilla flavors of ICA instead? To that end, we begin by noting that FCA as developed here relies on a statistical criteria quantifying freeness that account for matricial structure via the computation of the empirical free cumulants and empirical free entropy. The criteria, as summarized in Algorithm \ref{alg:fcf_pro}, involve matrix computations in Table \ref{table:fcf_Fhat}.

However, it too is vanilla in the sense that it does not account for the spatio-temporal structure in the matrices. To note this observe that in (\ref{eq:general matrix model}), if we were to transform $\mathbf{Z}_i \mapsto \mathbf{P}_1 \mathbf{Z}_i \mathbf{P}_1^*$ ($\mathbf{Z}_i \mapsto \mathbf{P}_1 \mathbf{Z}_i \mathbf{P}_2$ for non-self adjoint case), where $\mathbf{P}_1$ and $\mathbf{P}_2$ are appropriately sized, arbitrary permutation or orthogonal  matrices  then the mixing model still holds with $\mathbf{X}_i \mapsto \mathbf{P}_1 \mathbf{X}_i \mathbf{P}_1^*$ ($\mathbf{X}_i \mapsto \mathbf{P}_1 \mathbf{X}_i \mathbf{P}_2$ for non-self adjoint case) and FCA will still succeed even though the spatio-temporal structure has been altered because the quantities being computed in Table \ref{table:fcf_Fhat} will be invariant to these transformations.

Thus vanilla FCA, as presented here, also \textbf{does not} exploit the spatio-temporal structure either the way vanilla ICA does not. Fundamentally, vanilla FCA is different from vanilla ICA because of the difference between independence and freeness as captured in the statistical criterion. Hence comparing FCA to vanilla ICA is appropriate. Our goal is to bring into sharper focus a new statistical criterion for unmixing variables, derived from a different probabilistic model/embedding, and highlight its applicability in settings where vanilla ICA has succeeded while emphasizing how it is the  different statistical criterion and embedding that result in the improvement in performance relative to vanilla ICA (e.g. see Figure \ref{fig:locustdenoise}).  Just as in vanilla ICA, the surprise and delight of FCA is that they both work well out-of-the-box and it is the statistical criterion quantifying the degree of independence (for ICA) or degree of freeness (for FCA) that is making the difference. That they succeed when they do with the bare minimum of orthogonality constraints is what makes them magical -- our goal is simply to add FCA and the underlying statistical criterion to the ICA list.

\begin{figure}[t]
          \centering
            \begin{subfigure}[t]{0.24\textwidth}
              \includegraphics[width = \textwidth,trim={0cm 0cm 0cm 0cm},clip]{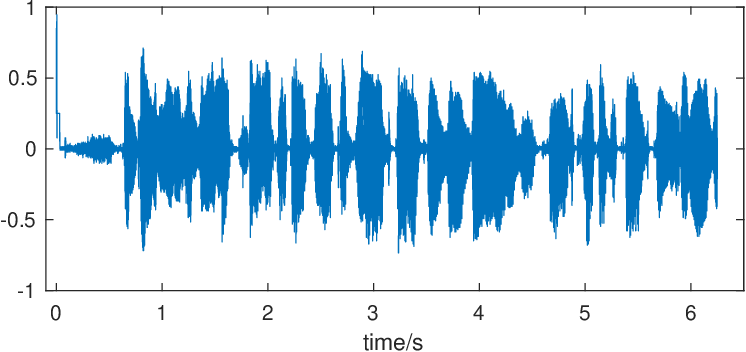}
          \caption{Audio 1}
            \label{fig:x1_sw}
            \end{subfigure}
            \begin{subfigure}[t]{0.24\textwidth}
              \includegraphics[width = \textwidth,trim={0cm 0cm 0cm 0cm},clip]{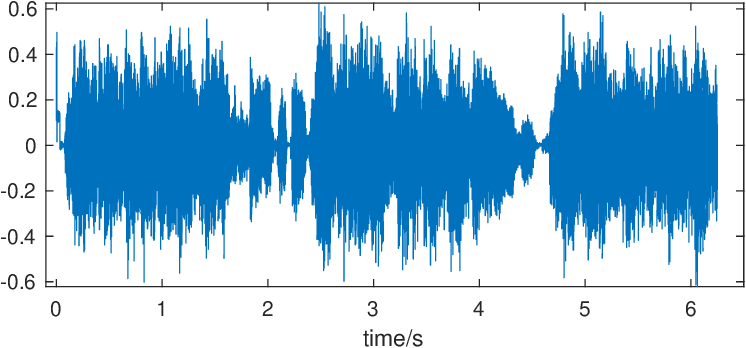}
          \caption{Mixed audio1}
            \label{fig:z1_sw}
            \end{subfigure}  
            \begin{subfigure}[t]{0.24\textwidth}
              \includegraphics[width = \textwidth,trim={0cm 0cm 0cm 0cm},clip]{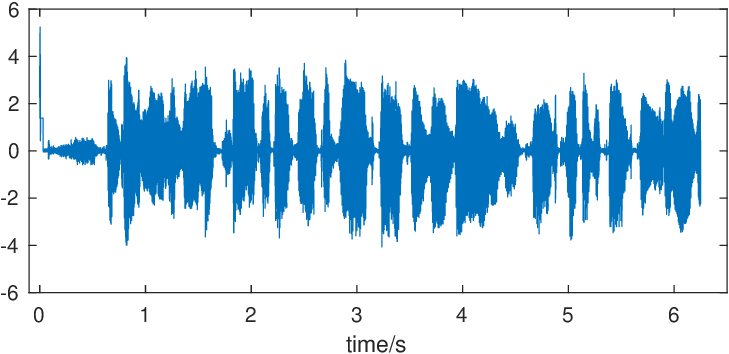}
          \caption{Audio 1 via ICA}
            \label{fig:xhat1_ica_sw}
            \end{subfigure}
            \begin{subfigure}[t]{0.24\textwidth}
              \includegraphics[width = \textwidth,trim={0cm 0cm 0cm 0cm},clip]{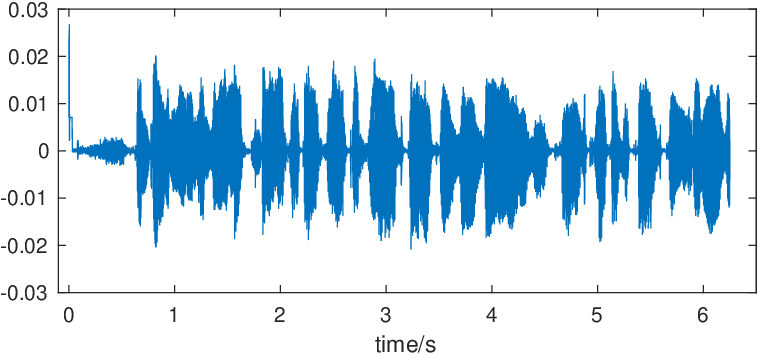}
          \caption{Audio 1 via FCA}
            \label{fig:xhat1_fca_sw}
            \end{subfigure}
            \begin{subfigure}[t]{0.24\textwidth}
              \includegraphics[width = \textwidth,trim={0cm 0cm 0cm 0cm},clip]{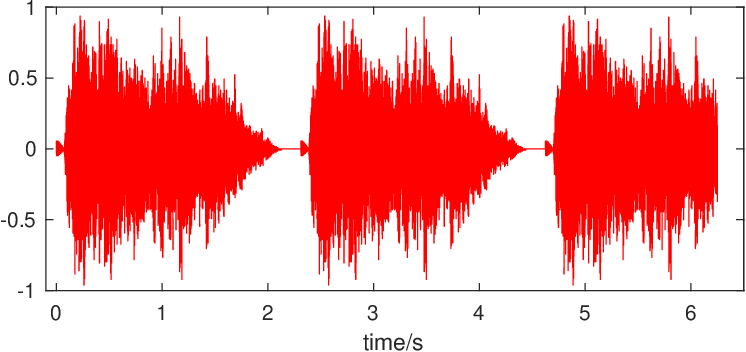}
          \caption{Audio 2}
            \label{fig:x2_sw}
            \end{subfigure} 
            \begin{subfigure}[t]{0.24\textwidth}
              \includegraphics[width = \textwidth,trim={0cm 0cm 0cm 0cm},clip]{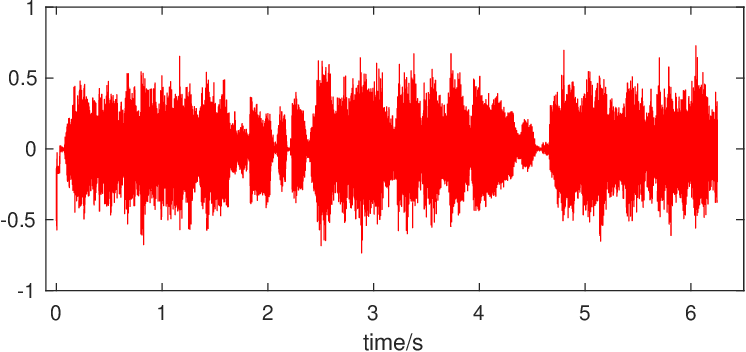}
          \caption{Mixed Audio 2}
            \label{fig:z2_sw}
            \end{subfigure}
            \begin{subfigure}[t]{0.24\textwidth}
              \includegraphics[width = \textwidth,trim={0cm 0cm 0cm 0cm},clip]{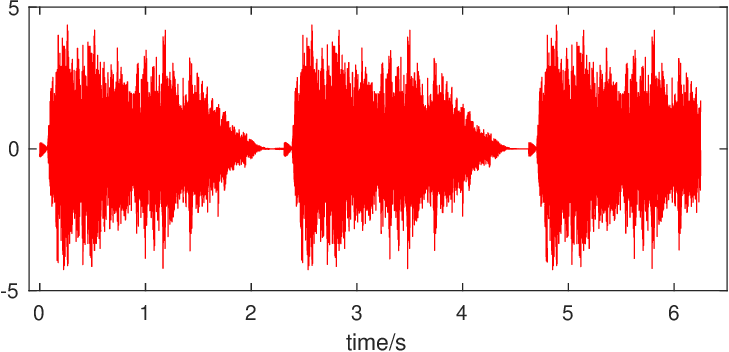}
          \caption{Audio 2 via ICA}
            \label{fig:xhat2_ica_sw}
            \end{subfigure}
            \begin{subfigure}[t]{0.24\textwidth}
              \includegraphics[width = \textwidth,trim={0cm 0cm 0cm 0cm},clip]{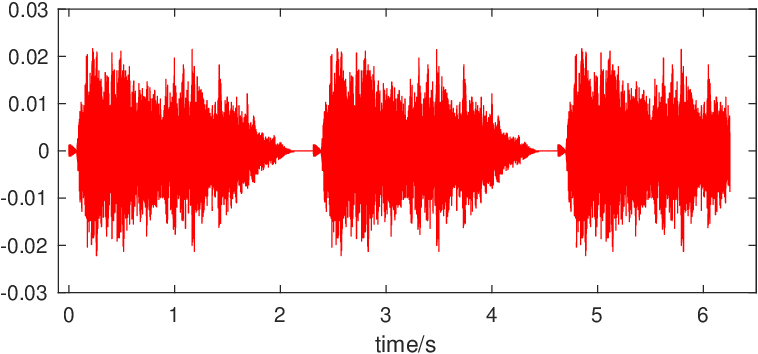}
          \caption{Audio 2 via FCA}
            \label{fig:xhat2_fca_sw}
            \end{subfigure}
            \caption{An experiment in audios separation via ICA and FCA: Note that subplots (c), (g) (unmixed audio signals via ICA) and (d), (h) (unmixed audio signals via FCA) both recover (a), (e).
            In this experiment, $\bm{A} = [\sqrt{2}, \sqrt{2}; -\sqrt{2}, \sqrt{2}]/2$ in \eqref{eq:general matrix model}.
           The errors for ICA and FCA are $1.47\times 10^{-2}$ and $1.79\times 10^{-2}$ respectively.}
           \label{fig:spectrogram_sep}
            \end{figure}

  \begin{figure}[t]
          \centering
            \begin{subfigure}[t]{0.23\textwidth}
              \includegraphics[width = \textwidth,trim={0cm 0cm 0cm 0cm},clip]{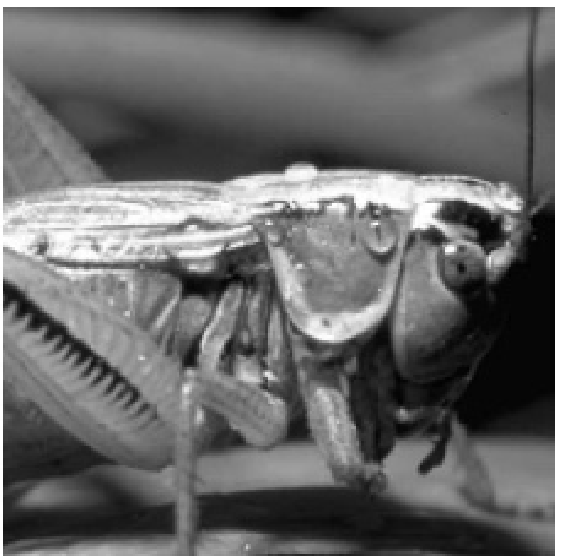}
          \caption{Original Image}
            \label{fig::locust}
            \end{subfigure}
            \begin{subfigure}[t]{0.23\textwidth}
              \includegraphics[width = \textwidth,trim={0cm 0cm 0cm 0cm},clip]{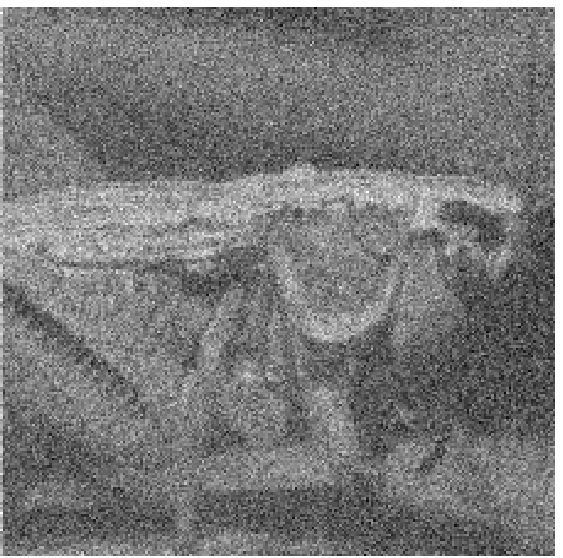}
          \caption{Mixed Image 1}
            \label{fig::mix1}
            \end{subfigure}           
            \begin{subfigure}[t]{0.23\textwidth}
              \includegraphics[width = \textwidth,trim={0cm 0cm 0cm 0cm},clip]{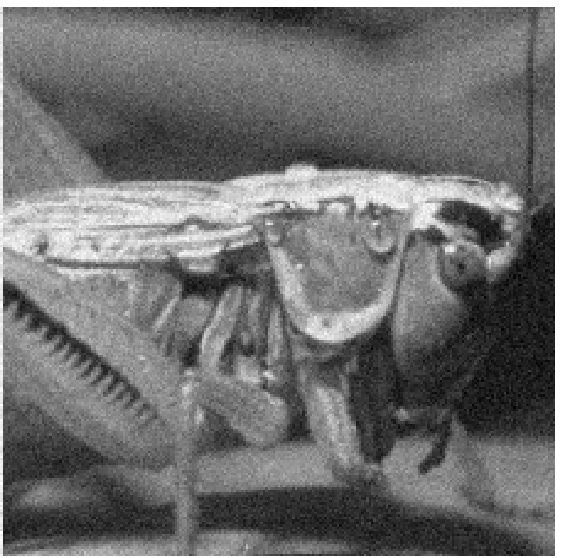}
          \caption{Image via ICA}
            \label{fig::ica_recover}
            \end{subfigure}
            \begin{subfigure}[t]{0.23\textwidth}
              \includegraphics[width = \textwidth,trim={0cm 0cm 0cm 0cm},clip]{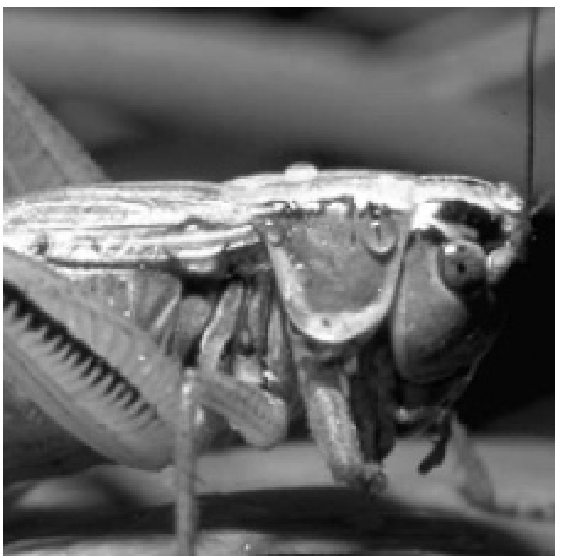}
          \caption{Image via FCA}
            \label{fig::fca_recover}
            \end{subfigure}
            \begin{subfigure}[t]{0.23\textwidth}
              \includegraphics[width = \textwidth,trim={0cm 0cm 0cm 0cm},clip]{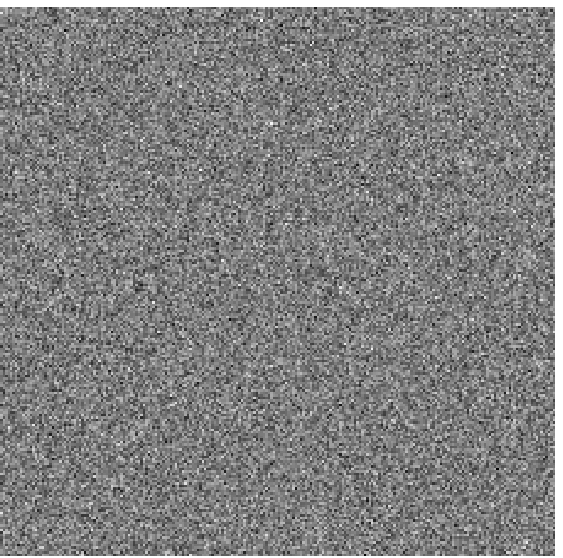}
          \caption{Gaussian Noise}
            \label{fig::Gau_noise}
            \end{subfigure}            
            \begin{subfigure}[t]{0.23\textwidth}
              \includegraphics[width = \textwidth,trim={0cm 0cm 0cm 0cm},clip]{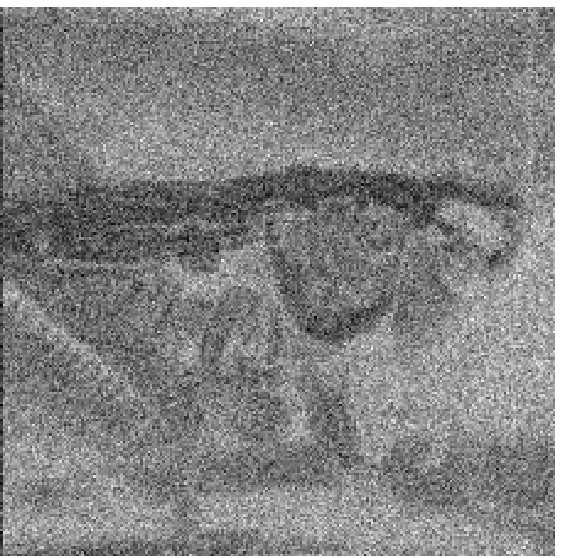}
          \caption{Mixed Image 2}
            \label{fig::mix2}
            \end{subfigure}
            \begin{subfigure}[t]{0.23\textwidth}
              \includegraphics[width = \textwidth,trim={0cm 0cm 0cm 0cm},clip]{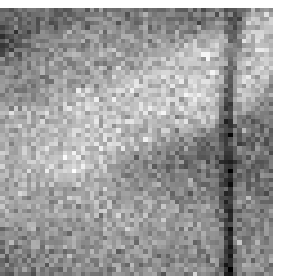}
          \caption{Image via ICA (zoom in)}
            \label{fig::ica_zoom_in}
            \end{subfigure}
            \begin{subfigure}[t]{0.23\textwidth}
              \includegraphics[width = \textwidth,trim={0cm 0cm 0cm 0cm},clip]{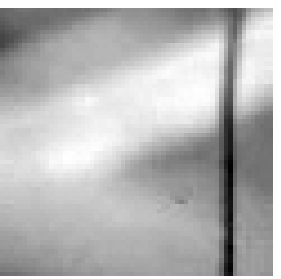}
          \caption{Image via FCA (zoom in)}
            \label{fig::fca_zoom_in}
            \end{subfigure}
            \caption{
            An experiment in image denoising via ICA and (kurtosis-based) FCA: 
            Comparing (g) and (f), we observe that FCA does a better job then ICA in this experiment. Here $\bm{A} = [\sqrt{2}, \sqrt{2}; -\sqrt{2}, \sqrt{2}]/2$  in (\ref{eq:general matrix model}). 
            The variance of whitened Gaussian noise is set to equal the empirical variance of original image.}
            \label{fig:locustdenoise}
            \end{figure}

\begin{figure}
\centering
\begin{tikzpicture}[scale=0.65,>=latex]
\path[draw,name path=line1] (-6,0)to[out=220,in=80] (-8,-3);
\path[draw,name path=border1] (8,-1)to[out=150,in=-10](-6,0);
\draw[draw,thick,name path=line2](7,-4.5) to[out = 110,in=230](8,-1);
\draw[draw,thick,name path=border2] (-8,-3)to[out = 0,in=150](7,-4.5);
\shade[left color=gray!10,right color=gray!70,opacity=.50] 
  (-6,0)to[out=220,in=80] (-8,-3)to[out = 0,in=150](7,-4.5) to[out = 110,in=230](8,-1)to[out=150,in=-10](-6,0);
\node at (-7.5,0.5) (label1) {Independence};
\draw[->] (label1) -- +(-40:2);

\path[draw,name path=line3] (0,3)to[out=50,in=200] (3,5);
\path[draw,name path=border3] (3,5)to[out=250,in=110] (3,-5);
\path[draw,name path=line4] (3,-5)to[out=200,in=50] (0,-7);
\path[draw,name path=border4] (0,-7)to[out=110,in=250] (0,3);
\shade[top color=blue!10,bottom color=blue!70,opacity=.50] 
  (0,3)to[out=50,in=200] (3,5)to[out=250,in=110] (3,-5)to[out=200,in=50] (0,-7)to[out=110,in=250] (0,3);
\node at (5,-5.5) (label1) {Rectangular Freeness};
\draw[->] (label1) -- +(165:3.35);

\path[draw,name path=line5] (-5,2)to[out=50,in=200] (-2,4);
\path[draw,name path=border5] (-2,4)to[out=250,in=110] (-1,-4.5);
\path[draw,name path=line6] (-1,-4.5)to[out=200,in=50] (-4,-6.5);
\path[draw,name path=border6] (-4,-6.5)to[out=110,in=250] (-5,2);
\shade[top color=red!10,bottom color=red!70,opacity=.50] 
  (-5,2)to[out=50,in=200] (-2,4)to[out=250,in=110] (-1,-4.5)to[out=200,in=50] (-4,-6.5)to[out=110,in=250] (-5,2);
\node at (-7,-5.5) (label1) {Self-Adjoint Freeness};
\draw[->] (label1) -- +(30:3.35);

\path[name intersections={of=border1 and border3,by={a}}];
\path[name intersections={of=border2 and border4,by={b}}];

\draw[thick,dashed] (a) to[out=200,in=50] (b);

\path[name intersections={of=border1 and border5,by={c}}];
\path[name intersections={of=border2 and border6,by={d}}];

\draw[thick,dashed] (c) to[out=200,in=50] (d);

\draw[fill] (3,3) circle [radius=0.1];
\node at (4.3,3.3) {$(X^{(1)}_1,X^{(1)}_2)$};

\draw[fill] (1.8,3.3) circle [radius=0.05];
\draw[dashed] (3,3) to (1.8,3.3);

\draw[fill] (-3.2,1.8) circle [radius=0.05];
\draw[dashed] (3,3) to (-3.2,1.8);

\draw[fill] (3,0) circle [radius=0.05];
\draw[dashed] (3,3) to (3,0);

\draw[fill] (6,0.5) circle [radius=0.1];
\node at (7.3,0.8) {$(X^{(2)}_1,X^{(2)}_2)$};

\draw[fill] (1,1) circle [radius=0.05];
\draw[dashed] (6,0.5) to (1,1);

\draw[fill] (-3.2,-1.8) circle [radius=0.05];
\draw[dashed] (6,0.5) to (-3.2,-1.8);

\draw[fill] (5.5,-1) circle [radius=0.05];
\draw[dashed] (6,0.5) to (5.5,-1);
\end{tikzpicture}
\caption{We can regard ICA and FCA with various embedding as ``projections" onto corresponding manifolds. Here, the gray surface denotes the manifold of independent pairs. The red and blue surfaces stand for self-adjoint free pairs and rectangular free pairs respectively. In order to achieve the best performance, one shall pick the projection into the closest manifold. For example, if the latent data is $(X^{(1)}_1,X^{(1)}_2)$, then rectangular FCA should have the best performance when separating them from the additive mixture. In contrast, for the underlying data $(X^{(2)}_1,X^{(2)}_2)$, one should pick ICA.}
\label{fig:diagmram}
\end{figure}
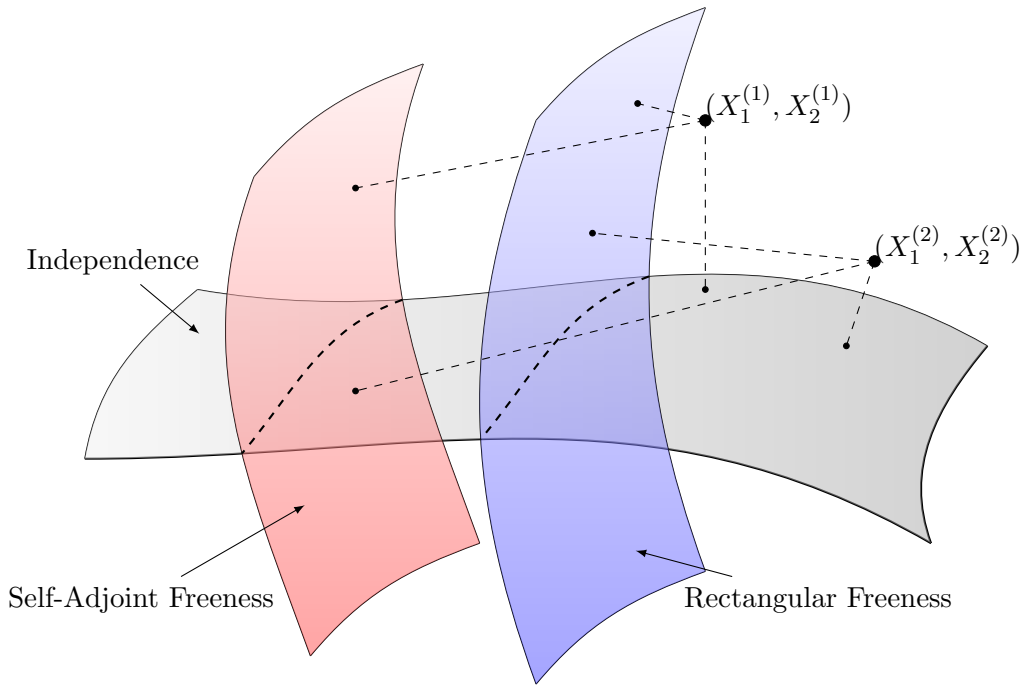

\subsection{Organization}
The remainder of the paper is organized as follows. 
We will develop FCA for self-adjoint and rectangular non-commutative random variables (corresponding to self-adjoint and rectangular random matrices) in Section \ref{sec:main result} by describing the objective functions whose maximization, analogous to the ICA setting, leads to successful unmixing of the `free' components from their additive mixture. Then we describe FCA based algorithms for factorizing data matrices in Section \ref{sec:algo}.
We illustrate  our theorems and ability of FCA to successfully unmix real-world images using numerical simulation in Section \ref{sec:num_simu}. We present some concluding remarks and highlight some open problems in Section \ref{sec:conclusions}

A self-contained introduction to the free probability is given in Section \ref{subsec:prelim self adjoint} and \ref{subsec:prelim rec} for self-adjoint and rectangular random variables respectively. 
We build the connection between non-commutative random variables and random matrices in Section \ref{sec:connection}.


\section{Main result: Recovery guarantees for FCA}  \label{sec:main result}

\subsection{Setup and assumptions under an orthogonal mixing model}

\subsubsection{The self-adjoint setting}
\label{sec:res_sym1}
Given a probability space $(\alg, \fnl)$, let $\alge_1, \ldots, \alge_\nc$ be $\nc$ self-adjoint and free random variables (see Appendix \ref{subsec:prelim self adjoint}). Let $\bm{y}$ denote the vector which contains as its elements the various additive mixtures of $\alge_1, \ldots, \alge_\nc$. We model $\bm{y}$ as  
\begin{equation}\label{eq:sym fca model}
    \underbrace{\begin{bmatrix} y_1  \\ \vdots \\ y_s \end{bmatrix}}_{=:\bm{y}} =\underbrace{\begin{bmatrix}\mixmqq_1 & \cdots & \mixmqq_\nc \end{bmatrix}}_{=: \mixmq{}} \underbrace{\begin{bmatrix} x_1 \\ \vdots \\ x_s  \end{bmatrix}}_{=: \bm{x}}, 
    \end{equation}
where $\mixmq$ is a $s \times s$ orthogonal matrix.

For self-adjoint FCA, we assume that the variables $\alge_i$ are centered and have unit variance, \textit{i.e.} for  $i = 1, \cdots, \nc$, we have that   $\fnl(\alge_i) = 0$ and that $\fnl(\alge_i^2) = 1$ .

\subsubsection{The non-self-adjoint setting}
\label{sec:res_rec}
Given a $(\rho_1, \rho_2)$-rectangular probability space $(\alg, p_1, p_2, \fnl_1, \fnl_2)$ -- (see Appendix \ref{subsec:prelim rec}) -- we consider a setup similar to that in \eqref{eq:sym fca model} where we model $\bm{y}$ as 
\begin{equation}
\label{model_eq_rec}
\bm{y} = \bm{Q} \bm{x},    
\end{equation}
where $\bm{Q}$ is an $\nc \times \nc$ orthogonal matrix. 
We assume that for  $i = 1, \cdots, \nc$, $x_i, y_i$ are rectangular random variables (i.e., $x_i,y_i \in \alg_{12} := p_1 \alg p_2$) and $\fnl_1(x_ix_i^*) = 1$. Note that $\fnl(x_i) = 0$ by default.
The fundamental assumption is now that $(x_i)_{i=1}^\nc$ are free with amalgamation over the linear span of $p_1$ and $p_2$. We will simply say that $(x_i)_{i=1}^\nc$ are free if there is no ambiguity.

\subsection{Free kurtosis based FCA}
The \emph{free kurtosis} of a centered self-adjoint random variable $x \in \alg$ is defined as 
\begin{equation}
\label{eq:free_kurtosis_expfor}
    \cumu_4(x) = \fnl(x^4) - 2\,\fnl(x^2)^2.
\end{equation}
The  \emph{rectangular free kurtosis} of a rectangular random variable $x \in \alg_{12}$ is defined as 
\begin{equation}
\label{eq:def_kurt_rec}
    \cumu_4(x) = \fnl_1((xx^*)^2) - \left(1 + \frac{\fnl(p_1)}{\fnl(p_2)}\right) \left(\fnl_1(xx^*)\right)^2.
\end{equation}
We now state a result on the largest free component.

\begin{theorem}[Largest free component] 
\label{thm:kurt1}
Assume $\bm{x}$ and $\bm{y}$ are related either via \eqref{eq:sym fca model} in the self-adjoint setting or via \eqref{model_eq_rec} in the non-self-adjoint setting. Suppose, additionally, without of loss of generality,  that
\begin{equation}    \label{eq:kurt1_main}
      \abs*{ \cumu_4(x_1) } \geq \abs*{\cumu_4(x_2)} \geq \cdots \geq \abs*{\cumu_4(x_\nc)}.
\end{equation}
Let $\argkurt^{(1)}$ denote the solution of the manifold optimization problem
\begin{equation}
\label{eq:opt_kurt}
\argkurt^{(1)}= \mathop{\arg\max}_{\argkurt} \abs*{\cumu_4(\bm w^T \bm y)}\textrm{ subject to } \norm{\argkurt} = 1
\end{equation}
\begin{enumerate}
\item[(a)] 
Suppose 
\begin{equation} \label{eq:strictdeckurt}
  \left \lvert \cumu_4(x_1) \right \rvert > \abs{\cumu_4(x_2)} \geq \cdots \geq \abs{\cumu_4(x_\nc)}.
\end{equation}
Then
\begin{equation}
    \argkurt^{(1)} = \pm \, \mixmqq_1
\end{equation}
\item[(b)] 
Suppose there is an integer $r \in [2..\nc]$, such that 
\begin{equation} \label{eq:firstrequal}
\abs*{\cumu_4(x_1)} = \cdots = \abs*{\cumu_4(x_r)} > \abs*{\cumu_4(x_{r + 1})} \geq \cdots \geq \abs*{\cumu_4(x_\nc)}.
\end{equation}
Then 
\begin{equation} \label{eq:opt_kurt_r}
    \argkurt^{(1)}  \in \{\pm\,\mixmqq_1, \ldots, \pm \,\mixmqq_r \}    .
\end{equation}
\end{enumerate}
\end{theorem}
\begin{remark}
(b) of the above theorem considers the case where there are multiple indexes corresponding to the largest absolute kurtosis. In contrast to the principal component analysis, the maximizers of \eqref{eq:opt_kurt} (and also of \eqref{eq:opt_kurt_r}) only contains corresponding columns of $\mixmq$, and not their general linear combinations. This is a consequence of that we are using the fourth order statistics of random variables.
\end{remark}

\begin{theorem}[The $k$-th largest free component] 
\label{thm:kurt2}
Assume that $\bm{x}$ and $\bm{y}$ are related as in Theorem \ref{thm:kurt1}. Let $\argkurt^{(k)}$ denote the solution to the manifold optimization problem 
\begin{equation}
\label{eq:kurt2_a}
\argkurt^{(k)} =  \mathop{\mathop{\arg \max}_{\argkurt}}\left \lvert\cumu_4(\bm w^T\bm y)\right \rvert \textrm { subject to } \norm{\argkurt} = 1, \argkurt \perp \argkurt^{(1)}, \cdots, \argkurt^{(k - 1)}
\end{equation}
  
Suppose
$$
  \abs*{\cumu_4(x_1)} > \abs*{\cumu_4(x_2)} > \cdots > \abs*{\cumu_4(x_\nc)}.$$ 
  
\flushleft Then 
\begin{equation}
    \argkurt^{(k)}  = \pm\,\mixmqq_k.
\end{equation}
\end{theorem}

\begin{theorem}[Principal free components]
\label{thm:kurt3}
Assume that $\bm{x}$ and $\bm{y}$ are related as in Theorem \ref{thm:kurt1}. Let $\orth(\nc)$ denote the set of $\nc \times \nc$ orthogonal matrices, and consider the manifold optimization problem 
\begin{equation}
\label{eq:kurt3}
  \mathop{\max}_{\Argent} \sum_{i = 1}^\nc \left \lvert\cumu_4\left([\Argkurt^T\bm{y}]_i\right)\right \rvert \textrm{ subject to } \Argkurt \in \orth(\nc),
\end{equation}
where $[\Argkurt^T  \bm{y}]_i$ denotes the $i$-th  element of $\Argkurt^T  \bm{y}$. 
Suppose that all components have non-zero kurtosis, i.e., 
\begin{equation}
\label{eq:kurt3_assumption}
\abs*{\cumu_4(x_i)} > 0 \quad \text{for $i = 1, \cdots, s$.} 
\end{equation}
Then $\bm{W}$ is an optimum if and only if :
\begin{equation}
\bm{W}= \mixmq \permm\signm,
\end{equation}
for some $\permm$ and $\signm$ where 
where $\permm$  is a permutation matrix and $\signm$ is a diagonal matrix with $\pm 1$ as diagonal elements.
\end{theorem}

\begin{remark}
\label{rmk:onezerok4}
Theorem \ref{thm:kurt3} still holds if there is at most one components with zero free kurtosis.
\end{remark}

\begin{remark}
\label{rmk:nonconvex}
In ICA, it is known that similar optimization problem based on cumulants are non-concave (see e.g. Fig 2 and Fig 3 of \cite{cruces2002robust}).
On the other hand, free cumulants are analogous to classic cumulants; they both has free/independent additivity and thus the corresponding kurtosis optimization problems share the similar nature.
Therefore, the object functions in \eqref{eq:opt_kurt} and \eqref{eq:kurt3} are non-concave. 
For further discussion, see Section \ref{sec:op_conv}.
\end{remark}

\subsubsection{Higher-order free cumulant based FCA}

\begin{remark}
It can be shown that above theorems still hold with $\cumu_4(\cdot)$ replaced by any $\cumu_{2m}(\cdot)$, for $m \geq 3$.
\end{remark}

\begin{remark}
The maximizer of \eqref{eq:kurt3} is not guaranteed to recover $\mixmq$ when there are multiple components of $x$ with zero free kurtosis. 
In this case, one may try to use optimization problem based on $\cumu_{2m}(\cdot)$, $m \geq 3$.
However, the semicircle elements (for the self-adjoint case; see Appendix \ref{subsubsec:semicirclr}) and the Poisson elements (for the non-self-adjoint case; see Appendix \ref{subsubsec:poisson}) have all vanishing free cumulants of order higher than $2$.
In Theorem \ref{thm:identifiability}, we will prove that $\mixmq$ can be recovered whenever $\bm x$ contains at most one semicircular element or free Poissonian element for the self-adjoint or non-self-adjoint settings, respectively.
\end{remark}

\subsubsection{Free-entropy based FCA}

The free entropy $\ent(a_1,\cdots,a_n)$ (see Sections \ref{subsubsec:free entropy} and \ref{subsubsec:free entropy rec} for the definitions in the self-adjoint and non-self-adjoint settings) of a tuple of free random variables encodes the dependence between the variables $a_i$. Analogous to the scalar setting,  the free entropy is maximized when the random variables are freely independent. Thus we can pose FCA as an entropy maximization problem as stated next. 

\begin{theorem}[FCA based on free entropy]
\label{thm:ent}
Assume that $\bm{x}$ and $\bm{y}$ are related as in Theorem \ref{thm:kurt1} and at most one component of $x$ is semicircular in the self-adjoint setting or a free Poisson in the non-self adjoint setting.  Let $\orth(\nc)$ denote the set of $\nc \times \nc$ orthogonal matrices.  
Suppose that 
$$\ent(\alge_i) > -\infty \textrm{ for }  i = 1,\cdots, s.$$ 
Consider the manifold optimization problem 
\begin{equation}
\label{eq:ent1}
 \max_{\Argent} \sum_{i = 1}^s -\ent\left([\Argkurt^T \bm{y}]_i\right) \textrm{ subject to } \Argent \in \orth(\nc),
\end{equation}
where $[\Argkurt^T  \bm{y}]_i$ denotes the $i$-th  element of $\Argkurt^T  \bm{y}$.
Then $\bm W$ is an optimum if and only if:
\begin{equation}
 \bm{W}= \mixmq \permm\signm,
 \end{equation}
for some $\permm$ and $\signm$ 
where $\permm$  is a permutation matrix and $\signm$ is a diagonal matrix with $\pm 1$ as diagonal elements.
\end{theorem}

\subsubsection{FCA identifiability condition}

In the self-adjoint setting $c_4$-FCA will fail when $\bm x$ contains semicircular elements because free semi-circular elements have a free kurtosis identically equal to zero. Moreover, suppose $\bm{\alge} = (\alge_1, \alge_2)^T$ where $\alge_i$ are free semicircular elements with $\fnl(\alge_i) = 0$ and $\fnl(\alge_i^2) = 1$. Then, it can be shown that for any $\mixmq \in \orth(2)$,  the components of $\mixmq \bm{\alge}$ are still free semicircular elements. Therefore, if there are more two components are  semicircular elements, it is impossible to identify $\mixmq$ with the mere knowledge of free independence between the components of $\bm x$. The analog of this holds for the non-self-adjoint setting as well.

We now state an FCA identiability condition based on this observation.

\begin{theorem}[Identifiability Condition]
\label{thm:identifiability}
Consider \(\bm \alge\) and \(\bm y\) and $\mixmq \in \orth(\nc)$ such that $\bm \alge$ and $\bm y$ are related as in Theorem \ref{thm:kurt1}. 
Assume $\bm \alge$ has free elements. Assume that at most one component of $\bm{x}$ is semicircular in the self-adjoint setting or free Poisson in the non-self-adjoint setting. 

Now, if there is a $\Argkurt \in \orth(\nc)$ such that $\Argkurt^T \bm y$ has free components, then 
\begin{equation}
\Argkurt = \mixmq \permm \signm.
\end{equation}
for some $\permm$ and $\signm$ 
where $\permm$  is a permutation matrix and $\signm$ is a diagonal matrix with $\pm 1$ as diagonal elements.
That is, $\Argkurt$ can be obtained by permuting the columns of $\mixmq$ with possible sign flips and vice versa. 
\end{theorem}

\begin{remark}[Weakness of FCA condition relative to ICA]
Here we only establish the FCA identifiability condition for orthogonal mixing matrix. With help of whitening process (see Section \ref{sec:free_whiten}), this can be easily extended to invertible mixing matrix. However, above invertible condition of FCA is weaker than the corresponding condition for ICA, which applies to rectangular mixing matrix under mild condition. See Section \ref{subsec:over_under_fca} for further discussion. 
\end{remark}

\subsection{Setup and assumptions under a non-orthogonal mixing model}
\label{sec:free_whiten}

\subsubsection{The self-adjoint setting}

Given a probability space $(\alg, \fnl)$, let $\alge_1, \ldots, \alge_\nc$ be $\nc$ self-adjoint and free random variables. Let $\bm{{x}}$ be a vector of free, but not necessarily centered random variables. Then the  variable $\widetilde{x}_i$ defined as 
$$\widetilde{x}_i = x_i -  \underbrace{\fnl(x_i) \, 1_\alg}_{=:\overline{x}_i},$$
is centered since $\fnl(\widetilde{x}_i) = 0$.  Substituting $x_i = \widetilde{x}_i + \fnl(x_i)$ in \eqref{eq:sym fca model} we obtain the mixed model
\begin{equation}\label{eq:sym fca model 2}
    \underbrace{\begin{bmatrix} z_1  \\ \vdots \\ z_s \end{bmatrix}}_{=:\bm{z}} =\underbrace{\begin{bmatrix}\bm{a}_1 & \cdots & \bm{a}_\nc \end{bmatrix}}_{=: \bm{A}} \underbrace{\begin{bmatrix} \widetilde{x}_1 + \overline{x}_1  \\ \vdots \\ \widetilde{x}_s + \overline{x}_s    \end{bmatrix}}_{=: \bm{\widetilde{x}}+ \bm{\overline{x}}} = \bm{A} \bm{\widetilde{x}} + \bm{A} \overline{\bm{x}}, 
\end{equation}
In this general, non-orthogonal mixing setup, we assume, without loss of generality, that $\varphi(\widetilde{x}_i^2) = 1$ and covariance $\cov_{\bm{x} \bm{x}} = \bm I$, where the covariance matrix $\cov_{\bm{x} \bm{x}} $ is defined as following.

\begin{definition}[Covariance matrix of self-adjoint random variables]
\label{def:cov}
Let \\
$\bm{z} = \begin{bmatrix} z_1 & \ldots & z_{\nc} \end{bmatrix}^T$ be a vector of self-adjoint random variables.  The covariance $\cov_{\bm{z} \bm{z}}$ matrix of $\bm{z}$ is the $\nc \times \nc$ matrix given by: 
\begin{equation}
    [\cov_{\bm{z}\bm{z}}]_{ij} = \fnl\left(\widetilde z_i \widetilde z_j\right) \quad \text{for $i, j = 1, \ldots, \nc$,}
\end{equation}
where $\widetilde{z}_i$ is the centered random variable

$$\widetilde z_i = z_i - \fnl(z_i) 1_\alg.$$
\end{definition}

\subsubsection{The non-self-adjoint setting}
Given a rectangular probability space $(\alg, p_1, p_2, \fnl_1, \fnl_2)$, let  $\alge_1, \ldots, \alge_\nc$ be $\nc$ self-adjoint and free random variables. We assume that $\bm{z}$ is modeled as in (\ref{eq:sym fca model 2}). In the non-self-adjoint setting, the variables are centered by construction -- we assume additionally that for all  $x_i \in \alg_{12}$,  $\varphi_1(x_i x_i^*) = 1$ and 
covariance $\cov_{\bm{x} \bm{x}} = \bm I$, where the covariance matrix $\cov_{\bm{x} \bm{x}} $ is defined as following.
 
\begin{definition}[Covariance matrix of non-self-adjoint random variables] \label{def:cov_rec}
For an arbitrary random vector $\bm{z}^T = \begin{bmatrix} z_1 & \ldots & z_{\nc} \end{bmatrix}$ of rectangular random variables from $\alg_{12}$, note that $\varphi(z_i) = 0$ by default, the covariance matrix of $z$ is defined by a $\nc \times \nc$ matrix $\cov_{\bm z \bm z}$ where
\begin{equation}
    [\cov_{\bm z \bm z}]_{ij} = \fnl_1\left( z_i  z_j^*\right).
\end{equation}
\end{definition}

\subsection{Unmixing mixed free random variables using FCA}
We first establish some properties of the covariance matrices thus computed.

\begin{proposition}\label{prop:positive cov}
The covariance matrix as in Definitions \ref{def:cov} and \ref{def:cov_rec} is  positive semi-definite. 
\end{proposition}

For the covariance of $\bm{z}$ satisfying (\ref{eq:sym fca model 2}), we have the following stronger result.

\begin{proposition}\label{prop:real positiv cov}
The vector of mixed variables $\bm{z}$ modeled as in (\ref{eq:sym fca model 2}) has covariance 
$\cov_{\bm z \bm z}$ that is real and positive definite.
\end{proposition} 

This proposition allows us to formulate FCA on the whitened vector and prove a recovery result as stated next. 
\begin{theorem} \label{thm:general kurt}
Assume $\bm{x}$ and $\bm{z}$ are related as in (\ref{eq:general model ica}). Let $\bm A = \bm{U} \bm{\Sigma} \bm{V}^T$ be the singular value decomposition of $\bm{A}$. 
Consider the manifold optimization problem
\begin{equation}
\label{eq:general kurt}
 \max_{\Argent} \sum_{i = 1}^s \left \lvert \cumu_4\left([\Argkurt^T \bm{y}]_i\right) \right \rvert \textrm{ subject to } \Argent \in \orth(\nc),
\end{equation}
where $\bm{y}$ is the whitened and centered vector given by:
\begin{equation}\label{eq:ztil ytil}
\bm{y} = \bm{C}_{\bm{z}\bm{z}}^{-\sfrac{1}{2}} \bm{\widetilde{z}},
\end{equation}
where $\bm{C}_{\bm{z}\bm{z}}^{-\sfrac{1}{2}} = \bm{U} \bm{\Sigma}^{-1} \bm{U}^T$ is the inverse of the square root of the covariance matrix $\bm{C}_{\bm{z}\bm{z}}$ and $\bm{z}$ is the centered vector whose $i$-th element is given by 
$$\widetilde{z}_i = z_i - \varphi(z_i) 1_\alg .$$

Suppose that all components have non-zero kurtosis, i.e., 
\begin{equation}
\abs*{\cumu_4(x_i)} > 0 \quad \text{for $i = 1, \cdots, s$.} 
\end{equation}
\flushleft Then $\bm{W}$ is an optimum if and only if:
\begin{equation}
\bm{W} = \left(\bm{U}\bm{V}^T\right) \permm\signm,
\end{equation}
for some $\permm$ and $\signm$ where 
where $\permm$  is a permutation matrix and $\signm$ is a diagonal matrix with $\pm 1$ as diagonal elements.
\end{theorem}
\begin{proof}
It suffices to observe via (\ref{eq:ica model 2}) that
\begin{equation}\label{eq:ztil zxil}
\bm{y} = \left(\bm{U}\bm{V}^T\right) \widetilde{\bm{x}}.
\end{equation}

The matrix $\bm{U}\bm{V}^{T}$ is an orthogonal matrix because $\bm{U}$ and $\bm{V}$ are orthogonal matrices and so we can recover $\bm{U}\bm{V}^T$ from the stated manifold optimization problem via an application of Theorem \ref{thm:kurt3}.
\end{proof}

\begin{theorem} \label{thm:general ent}
Suppose $\bm{x}$ and $\bm{z}$ are related as in Theorem \ref{thm:general kurt}. Let $\bm A = \bm{U} \bm{\Sigma} \bm{V}^T$ be the singular value decomposition of $\bm{A}$.  
Also suppose at most one element of $\bm x$ is semicircular in the self-adjoint setting and free Poissonian in the non-self-adjoint setting and that
\begin{equation}    \label{eq:finite ent}
      \ent(x_i) > -\infty \qquad \text{for $i = 1, \cdots, \nc$.}
\end{equation}
Consider the manifold optimization problem
\begin{equation}
\label{eq:general ent}
 \max_{\Argent} \sum_{i = 1}^s -\ent\left([\Argkurt^T \bm{y}]_i\right) \textrm{ subject to } \Argent \in \orth(\nc),
\end{equation}
where $\bm{y}$ is the whitened and centered vector given by \eqref{eq:ztil ytil}.
\flushleft Then $\bm{W}$ is an optimum if and only if:
\begin{equation}
\bm{W} = \left(\bm{U}\bm{V}^T\right) \permm\signm,
\end{equation}
where $\permm$ is a permutation matrix and $\signm$ is a diagonal matrix with $\pm 1$ diagonal elements.
\end{theorem}

\begin{proof}
The proof is exactly same as the the proof of the Theorem \ref{thm:general kurt}, except for our application of Theorem \ref{thm:ent} to \eqref{eq:ztil zxil} instead of Theorem \ref{thm:kurt3}.
\end{proof}

\begin{corollary}[Unmixing via FCA]
\label{coro:selffcf}
Suppose $\bm{x}$ and $\bm{z}$ are related as in Theorem \ref{thm:general kurt} and that the $x_i$'s satisfy the conditions in Theorem \ref{thm:general kurt} or \ref{thm:general ent}.
Let $\Argkurt_{\sf opt}$ denote an optimum of the optimization problem in (\ref{eq:general kurt}) or \eqref{eq:general ent}. Consider the factorization
\begin{equation}\label{eq:unmix free}
    \bm{z} = \widehat{\bm{A}} \, \widehat{\bm{x}},
\end{equation}
\textrm{where }
\begin{subequations}
\begin{equation}
\widehat{\bm{A}} = \label{eq: Ahatunmix free} \bm{C}_{\bm{zz}}^{\sfrac{1}{2}} \, \Argkurt_{\sf opt},
\end{equation}
\textrm{and }
\begin{equation}\label{eq:xhat unmix free}
\widehat{\bm{x}} = \widehat{\bm{A}}^{-1} \bm{z}.
\end{equation}
\end{subequations}
Then $\widehat{\bm A} = \bm A \bm P \bm S$ for some $\permm$ is a permutation matrix and $\signm$ is a diagonal matrix with $\pm 1$ diagonal elements. Therefore, $\widehat{\bm x}$ recovers $\bm x$ up to permutation and sign flips.

\end{corollary}
\begin{proof}
As $\bm{C}_{\bm{z}\bm{z}}^{\sfrac{1}{2}} = \bm{U} \bm{\Sigma}^{-1} \bm{U}^T$, given an optimum $\bm W$ satisfying $\bm{W} = \left(\bm{U}\bm{V}^T\right) \permm\signm$,  
\begin{equation}
    \bm{C}_{\bm{z}\bm{z}}^{\sfrac{1}{2}} \bm W = \bm U \bm \Sigma \bm V^T \bm P \bm S = \bm A \bm P \bm S. 
\end{equation}
That is, we recover mixing matrix $\bm A$ up to column permutation and column sign flips. This completes the proof.
\end{proof}
\subsection{Overdetermined and underdetermined FCA}
\label{subsec:over_under_fca}

We now consider same model as in \eqref{eq:general matrix model} for the settings where the mixing matrix $\bm{A}$ is rectangular. For the over-determined setting where  $\bm{A}$ is a $p \times s$ mixing matrix with $p > s$, it can be shown that FCA applied to $\bm{y} = \bm{\Sigma}_s^{-1} \bm{U}_{s}^T \bm{z}$  will unmix the free random variables. Here $\bm{U}_s$ is a $p \times s$ matrix and $\bm{\Sigma}_s$ is an $s \times s$ diagonal matrix of the singular values of $\bm{A}$.
These matrices can obtained by using eigenvalue decomposition $\bm{C}_{zz} = \bm{U}_{s} \bm{\Sigma}_s \bm{\Sigma}_s^T \bm{U}_s^T$. 

In the under-determined setting where $p < s$, it is established in ICA that mixing matrix is unique up to column permutation and scaling under mild condition \cite[Theorem 3]{eriksson2004identifiability}. This result is a consequence of scalar Cram\'ers Lemma \cite[Lemma 9, pp. 294]{comon_1994} and the Lemma of Marcinkiewicz-Dugue \cite[Lemma 10, pp. 294]{comon_1994}. And the estimation of mixing matrix and separation of the sources can be done via decomposition of a higher order tensor in a sum of rank-1 terms (see \cite{de2007fourth} and references therein). We don't have analogous result for uniqueness of mixing matrix in FCA yet due to missing analog of these lemmas in free probability. Acutally, the analog of Cram\'ers Lemma in the free probability does not hold \citep{lehner2004cumulants, marcinkiewicz_free_prob}, thus we expect the uniquenss result (if exists) will be weaker. On the other hand, it seems possible to develop analog of tensor method for free component separation since it mainly depends the free/indpendent additivity of the cumulants.


\subsection{Unmixing mixtures of matrices using FCA}
\label{sec:algo}
The multiplication of matrices is non-commutative, therefore we can consider the mixing model in (\ref{eq:general matrix model}) where $\bm{X}_1, \ldots, \bm{X}_{s}$ are  finite dimensional (asymptotically) free self-adjoint or rectangular matrices (see Definitions \ref{def:asy_ind} and \ref{def:asy_ind_rec}). The goal is to unmix $\bm{X}_1, \ldots, \bm{X}_{s}$ from their additive mixtures $\bm{Z}_1, \ldots, \bm{Z}_s$. 

Corollary \ref{coro:selffcf} provides a recipe for unmixing the mixture of matrices by factorizing $\bm{Z}$ into the matricial analog of (\ref{eq:unmix free}). In the matricial setting, this is equivalent to factorizing $\bm{Z} = (\widehat{\bm A} \otimes \bm{I}_N) \widehat{\bm X}$. We shall refer to this factorization of an array of matrices as Free Component Factorization (FCF). 

To compute $\widehat{\bm A}$ in FCF as prescribed by Corollary \ref{coro:selffcf} we must compute the matricial analog $\bm{Y}$ of $\bm{y}$ in (\ref{eq:ztil ytil}). This involves first computing the matricial $\nc \times \nc$ covariance matrix  analog as in Algorithm \ref{alg:free_whiten} where we have replaced the $\varphi(\cdot)$ and $\varphi_1(\cdot)$ in the self-adjoint and the rectangular settings with their matricial analogs as in (\ref{eq:sym varphi}) and (\ref{eq:rect varphi}), respectively. 

Having computed the whitened array of matrices $\bm{Y}$ we can compute the matrix $\widehat{\bm A}$ via Algorithm \ref{alg:fcf_pro} where the dot operator is as defined next.

\begin{definition}[Dot operator]
\label{def:dot_operator}
Let $\bm Y = [\bm Y_1^T, \cdots, \bm Y_\nc^T]^T$ be an array of matrices where $\bm{Y}_i \in \mathbb{R}^{N \times M}$. Given a function $F: \mathbb{R}^{N \times M} \mapsto \mathbb{R}$, define
\begin{equation}
    F_{\cdot}(\bm Y) := \begin{bmatrix} F(\bm Y_1) \\ \vdots \\ F(\bm Y_\nc)\end{bmatrix}
\end{equation}
For later convenience, we denote the sum of all elements of $F_{\cdot}(\bm Y)$ by $\sum F_{\cdot}(\bm Y)$, i.e.,
\begin{equation}
\label{eqn:dot_sum}
    \sum F_{\cdot}(\bm Y) = \sum_{i = 1}^\nc \left[F_{\cdot}(\bm Y)\right]_i = \sum_{i = 1}^\nc F(\bm Y_i)
\end{equation}
\end{definition}

\begin{algorithm}[H]
\caption{Free whitening for random matrices}
\label{alg:free_whiten}
\begin{spacing}{1.2}
        \textbf{Input:} $\bm{Z} = [\bm{Z}_1^T, \cdots, \bm{Z}_\nc^T]^T$ where $\bm{Z}_i$ are $N \times M$ matrices. $M = N$ if $\bm Z_i$ are self-adjoint.\\
        1. Compute $\overline{\bm{Z}} = [\overline{\bm{Z}}^T_1, \cdots, \overline{\bm{Z}}^T_\nc]^T$, where
        $$
        \overline{\bm{Z}}_i = 
        \begin{cases}
        \frac1N \Tr(\bm{Z}_i) \bm I_N & \text{if $\overline{\bm{Z}}_i$ are self-adjoint,} \\
         \mathrm{mean}(vec(\bm{Z}_i)) \times \mathrm{ones}(N,M) & \text{if $\overline{\bm{Z}}_i$ are rectangular.}
        \end{cases}
        $$
        \\
        2. Compute $\widetilde{\bm{Z}} = \bm{Z} - \overline{\bm{Z}}$ and  $\nc \times \nc$ empirical covariance matrix $\bm{C}$ where for $i, j = 1, \ldots, \nc$:
        \begin{flalign*}
            \bm C_{ij} = \frac1N \Tr(\bm{\widetilde{Z}}_i \bm{\widetilde Z}^H_j).
        \end{flalign*}
        \vspace{-.5cm}
        \\
        3. Compute eigen-decomposition , $\bm{C} = \bm{U} \bm{\Sigma}^2 \bm{U}^T$.\\
        4. Compute $\bm{Y} = ((\bm{U} \bm{\Sigma}^{-1}\bm{U}^T) \otimes \bm I_N) \widetilde{\bm Z}$.\\
        5. \textbf{return:} $\bm{Y}, \bm{\Sigma},$ and $\bm{U}$. 
        \vspace{1ex}
\end{spacing}
\end{algorithm}

\begin{algorithm}[H]
    \caption{Free Component Factorization (FCF) of an array of matrices}
    \label{alg:fcf_pro}
\begin{spacing}{1.5}
    \textbf{Input}: Array of matrices $\bm{Z} = [\bm Z_1^T, \cdots, \bm Z_\nc^T]^T$ where $\bm{Z}_i$ are $N \times M$ matrices.
    1. Compute $\bm Y, \bm \Sigma, \bm U$ by applying Algorithm \ref{alg:free_whiten} to $\bm Z$. \\
    2. Compute \footnotemark
    \vspace{0.2cm}
    $$\bm {\widehat W} = \mathop{\text{arg min}}_{\bm W \in O(n)} \sum \widehat F_{\cdot} \left(\bm{\widetilde{W}}^T \bm Y\right), \quad \text{where $\bm {\widetilde{W} }= \bm W \otimes \bm I_N$}.$$
    \vspace{-.8cm}

    3. Compute $\bm{\widehat A} = \bm U \bm \Sigma \bm U^T \bm {\widehat W}$ and $\bm {\widehat X} = (\bm{ \widehat A}^{-1} \otimes \bm{I}_N) \bm{Z}$.\\
    4. Sort components of $\widehat{\bm X}$ by magnitude of $\widehat{F}(\widehat{\bm X}_i)$. \\
    5. Permute the columns of $\widehat{\bm A}$ such that $\bm{Z} = (\bm{\widehat A} \otimes \bm{I}_N) \bm{\widehat X}$. \\
    6. \textbf{return:} $\bm{\widehat A}$ and $\bm{\widehat X}$ 
    \vspace{.2cm}
\end{spacing}
\end{algorithm}

\footnotetext{Here  $\widehat F(\cdot)$ is either the (self-adjoint or rectangular) free kurtosis, the  free entropy or a higher (than fourth) order (even valued) free cumulant. See Table \ref{table:fcf_Fhat}.}  

\subsection{Numerical algorithms for Free Component Factorization}
\label{sec:opt_summary}

The manifold optimization problem in FCF can be solved using a gradient descent with retraction  method \citep{manopt, mogensen2018optim}. 

\begin{theorem}[Gradient of the objective function]    \label{thm:grad_F}

Let $\bm{Y} = [\bm{Y}_1^T,\cdots, \bm{Y}_\nc^T]^T$ and $\bm{W} = [\bm{w}_1, \cdots, \bm{w}_\nc] \in \R^{\nc \times \nc}$. Suppose 
$$\widetilde{\bm{W}} = \bm{W}\otimes \bm{I}_N,$$
we are interested in the gradient
$$\partial_{\bm{W}_{k\ell}}  \sum \widehat F_{\cdot} \left(\widetilde{\bm W}^T \bm Y\right),$$
which depends on whether $\bm{Y}$ is an array of self-adjoint or rectangular matrices. 

Suppose $\widehat{F}(\cdot)$ is chosen to be    free kurtosis or free entropy for the self-adjoint or rectangular setting as in Table \ref{table:fcf_Fhat}, then the gradient is given by the corresponding expression in Table \ref{table:fcf_gradF} where 
$$\bm{X}_\ell = \widetilde{\bm{w}}_\ell^T \bm{Y},$$
and $\widetilde{\bm{w}}_\ell = \bm{w}_\ell\otimes \bm{I}_N$.
\end{theorem}

Armed with these gradients we can compute the free component factorization of an array of matrices using numerical solvers for manifold optimization, such as for example the  \texttt{manopt} \citep{manopt} package (for \texttt{MATLAB}) or the  \texttt{Optim.jl} \citep{mogensen2018optim}  package for \texttt{Julia}.  Our  software implemntation via the \texttt{FCA.jl} package \citep{fca} does precisely this.

\begin{table}[]
\centering
\caption{Formulas for $\widehat F(\cdot)$ in Algorithm \ref{alg:fcf_pro}. Here $\bm X$ is either a self-adjoint or a rectangular matrix.}
\label{table:fcf_Fhat}
\begin{tabular}{|c|c|c|}
\hline
& self-adjoint FCF & rectangular FCF \\
\hline \hline 
\shortstack{\\ free \\ kurtosis}&      
\begin{minipage}[c]{.35\linewidth}
\begin{equation*}
\begin{aligned}
&\widehat F(\bm{X})  = -\left \lvert \widehat \cumu_4(\bm{X}) \right \rvert, \text{where} \\
&\widehat \cumu_4(\bm{X}) = \frac{1}{N} \Tr(\bm{X}^4) \\
& \qquad \qquad - 2\left[\frac1N\Tr(\bm{X}^2)\right]^2 \\
\end{aligned}
\end{equation*}
\vspace{1ex} 
\end{minipage}  
&    
\begin{minipage}[c]{.43\linewidth}
\vspace{0.5ex}
\begin{equation*}
\begin{aligned}
& \widehat F(\bm{X})  = -\left \lvert \widehat \cumu_4(\bm{X}) \right \rvert, \quad \text{where} \\
& \widehat \cumu_4(\bm{X}) =  \frac{1}{N} \Tr((\bm{X}\bm{X}^H)^2) \\
& \quad -  \left(1 + \frac{N}{M}\right)\left[\frac1N\Tr(\bm{X}\bm{X}^H)\right]^2
\end{aligned}
\end{equation*}
\vspace{0.5ex} 
\end{minipage}             
\\
\hline 
\shortstack{\\ free \\ entropy} &       
\begin{minipage}[c]{.35\linewidth}
\begin{equation*}
\begin{aligned}
& \text{Let $\eg_i$ be eigenvalues of $\bm{X}$}\\
& \widehat F(\bm{X})  = \widehat \ent(\bm{X}), \quad \text{where} \\
& \widehat \ent(\bm{X}) = \sum_{i \neq j} \frac{\log \abs*{\eg_i - \eg_j}}{N(N - 1)}
\end{aligned}
\end{equation*}
\vspace{1ex} 
\end{minipage}           
&
\begin{minipage}[c]{.43\linewidth}
\begin{equation*}
\begin{aligned}
&\text{Let $\eg_i$ be eigenvalues of $\bm{X}\bm{X}^H$,}\\
&\text{set $\alpha = \frac{N}{N + M}$ and $\beta = \frac{M}{N + M}$}, \\
& \widehat F(\bm{X})  = \widehat \ent(\bm{X}), \quad \text{where} \\
& \widehat \ent(\bm{X})  = \alpha^2 \sum_{i\neq j} \frac{\log \abs*{\eg_i - \eg_j}}{N(N - 1)} \\
& \quad \quad \quad + (\beta - \alpha)\alpha \sum_{i = 1}^N \frac{\log \eg_i}{N}
\end{aligned}
\end{equation*}
\vspace{0.5ex}
\end{minipage}   
\\ \hline             
\end{tabular}
\end{table}

\begin{table}
\centering
\caption{Euclidean gradients for the setting in Theorem \ref{thm:grad_F}}
\label{table:fcf_gradF}
\begin{tabular}{|c|c|c|}
\hline
& self-adjoint FCF & rectangular FCF \\
\hline \hline 
\shortstack{\\ free \\ kurtosis} &      
\begin{minipage}[c]{.40\linewidth}
\begin{equation*}
\begin{aligned}
& \partial_{\bm{W}_{k\ell}} \sum \widehat F_{\cdot}\left(\widetilde{\bm W}^T \bm Y\right) = \\
 & - \mathrm{sign}(\widehat{\cumu}_4(\bm{X}_\ell)) \times \Bigg[\frac4N \Tr(\bm{Y}_k\bm{X}_\ell^3) \\
&  - \frac{8}{N^2} \Tr(\bm{X}_\ell^2) \Tr(\bm{Y}_k\bm{X}_\ell)\Bigg]
\end{aligned}
\end{equation*}
\vspace{1ex} 
\end{minipage}  
&    
\begin{minipage}[c]{.45\linewidth}
\vspace{1ex}
\begin{equation*}
\begin{aligned}
& \partial_{\bm{W}_{k\ell}}  \sum \widehat F_{\cdot} \left(\widetilde{\bm W}^T \bm Y\right) = \\
& - \mathrm{sign}(\widehat{\cumu}_4(\bm{X}_\ell)) \times \Bigg[\frac{4}{N}\Tr(\bm{Y}_k\bm{X}_\ell^H\bm{X}_\ell\bm{X}_\ell^H) \\
& - \left(1 + \frac NM \right) \frac{4}{N^2} \Tr(\bm{X}_\ell\bm{X}_\ell^H) \Tr(\bm{Y}_k\bm{X}_\ell^H)\Bigg]
\end{aligned}
\end{equation*}
\vspace{1ex} 
\end{minipage}             
\\
\hline 
\shortstack{\\ free \\ entropy}  &       
\begin{minipage}[c]{.40\linewidth}
\begin{equation*}
\begin{aligned}
&\text{Denote the eigenvalues and }\\
&\text{eigenvectors  of $\bm{X}_\ell $ by $\eg^{(\ell)}_i$ and $v^{(\ell)}_i$,}
\end{aligned}
\end{equation*}
\begin{equation*}
\begin{aligned}
& \partial_{\bm{W}_{k\ell}}  \sum \widehat F_{\cdot}\left(\widetilde{\bm W}^T \bm Y\right) = \\
&\frac{1}{N(N - 1)} \sum_{i \neq j} \frac{\partial_{\bm{W}_{k\ell}} \eg^{(\ell)}_i - \partial_{\bm{W}_{k\ell}} \eg^{(\ell)}_j}{\eg^{(\ell)}_i - \eg^{(\ell)}_j}\\
& \text{with } \partial_{\bm{W}_{k\ell}} \eg^{(\ell)}_i = (v^{(\ell)}_i)^H \bm{Y}_k v^{(\ell)}_i
\end{aligned}
\end{equation*}
\vspace{2ex} 
\end{minipage}           
&
\begin{minipage}[c]{.45\linewidth}
\begin{equation*}
\begin{aligned}
&\text{Denote the eigenvalues and eigenvectors}\\
&\text{of $\bm{X}_\ell\bm{X}_\ell^H$ by $\eg^{(\ell)}_i$ and $v^{(\ell)}_i$,}
\end{aligned}
\end{equation*}
\begin{equation*}
\begin{aligned}
& \partial_{\bm{W}_{k\ell}} \sum \widehat F_{\cdot}\left(\widetilde{\bm W}^T \bm Y\right) = \\
 & \frac{\alpha^2}{N(N - 1)} \sum_{i \neq j} \frac{\partial_{\bm{W}_{k\ell}} \eg^{(\ell)}_i - \partial_{\bm{W}_{k\ell}} \eg^{(\ell)}_j}{\eg^{(\ell)}_i - \eg^{(\ell)}_j} \\
+ & \frac{(\beta - \alpha)\alpha}{N}\sum_{i = 1}^N \frac{\partial_{\bm{W}_{k\ell}}\eg^{\ell}_i}{\eg^{(\ell)}_i}\\
& \text{with } \partial_{\bm{W}_{k\ell}}\eg^{(\ell)}_i =  (v^{(\ell)}_i)^H (\bm{Y}_k\bm{X}_\ell^H  \\
& \qquad \qquad \qquad \qquad + \bm{X}_\ell\bm{Y}_k^H )v^{(\ell)}_i
\end{aligned}
\end{equation*}
\vspace{1ex}
\end{minipage}   
\\ \hline             
\end{tabular}
\end{table}

\section{Numerical Simulations}
\label{sec:num_simu}

We will now validate the unmixing performance of FCA on additive mixtures on random matrices and compare the unmixing performance with that of ICA. To that end, we first define a permutation invariant unmixing error metric that is also invariant to scaling and sign ambiguities.

\begin{definition}[Unmixing Error Metric]
Let $\bm{A}$ be the mixing matrix in (\ref{eq:general matrix model}) and $\widehat{\bm{A}}$ be an estimate of the mixing matrix. The scaling and permutation invariant unmixing error is defined as 
            \begin{equation}
            \label{eq:error}
              \mathcal{E}(\bm{A},\widehat{\bm A}) = \min_{\bm{D}\in \mathcal{D},\bm{P} \in \Pi}|| \bm{P}\bm{D}\widehat {\bm{A}}^{-1} \bm{A}-\bm{I}||_F,
            \end{equation}
where $\mathcal{D}$ denotes the set of non-singular diagonal matrices and $\Pi$ denotes the set of (square) permutation matrices. 
\end{definition}
We shall utilize this metric to compare FCA and ICA in what follows. 

\subsection{Unmixing of self-adjoint matrices using self-adjoint FCA}\label{sec:sim self adjoint}

We now verify  Theorems \ref{thm:general kurt}, \ref{thm:general ent} and Corollary \ref{coro:selffcf} by showing that self-adjoint FCA can successfully, while not perfectly, unmix mixtures of self-adjoint matrices.

Let $\bm {G}_1\in \R^{N \times N}$ and $\bm{G}_2 \in \R^{N \times M}$ be two independent Gaussian matrices composed of i.i.d. $\mathcal{N}(0,1)$ entries. Define 
\begin{equation}
\bm{X}_1 = \frac{\bm{G}_1 + \bm{G}_1^T}{\sqrt{2N}}\quad  \text{and}\quad \bm{X}_2 = \frac{\bm{G}_2\bm{G}_2^T}{M}.
\end{equation}
The matrices $\bm X_1, \bm X_2 $ are self-adjoint by construction, and their eigen-spectra are displayed in Figures \ref{fig:x1} and \ref{fig:x2} respectively. In the parlance of random matrix theory \citep{edelman2005random}, $\bm X_1$ is a matrix drawn from  the Gaussian orthogonal ensemble (GOE) and its limiting eigen-distribution obeys the semi-circle distribution, while  $\bm X_2$ is a matrix drawn from Laguerre orthogonal ensemble (LOE) and its limiting eigen-distribution obeys the Mar\v{c}enko-Pastur distribution. 

We now mix the matrices as in (\ref{eq:general matrix model}) for a non-singular  

$$\bm{A}= \begin{bmatrix} 0.5 & 0.5 \\ -0.5 & 0.6 \end{bmatrix}.$$
The eigen-spectra of the mixed matrices $\bm Z_1$ and $\bm Z_2$ are plotted in Figures \ref{fig:z1} and \ref{fig:z2}.

The distributions of $\bm X_1$ and $\bm X_2$ are orthogonally invariant, and according to the discussion following Definition \ref{def:asy_ind}, $\bm{X}_1$ and $\bm{X}_2$ are asymptotically free. Moreover, only one matrix ($\bm X_1$) has a limiting eigen-distribution that converges to that of an abstract free semicircular element. Hence, we can apply self-adjoint FCA to factorize $\bm{Z} = \begin{bmatrix} \bm{Z}_1 & \bm{Z}_2 \end{bmatrix}^T$  using Algorithm \ref{alg:fcf_pro}  and obtain estimates $\widehat{\bm A}$, $\widehat{\bm{X}}_1$ and $\widehat{\bm{X}}_2$ which should be good estimates of $\bm{A}$, $\bm{X}_1$ and $\bm{X}_2$ respectively. 

Figures \ref{fig:xhat1_kurt} and \ref{fig:xhat2_kurt} display the eigen-spectra of the matrices $\widehat{\bm{ X}}_1$ and $\widehat{\bm{X}}_2$ returned by self-adjoint  free kurtosis-based FCA. Comparing Figures \ref{fig:xhat1_kurt}, \ref{fig:xhat2_kurt} and  \ref{fig:x1}, \ref{fig:x2} reveals that free kurtosis-based self-adjoint FCA successfully unmixes the mixed matrices  well. Figures \ref{fig:xhat1_ent} and \ref{fig:xhat2_ent} show that free entropy based self-adjoint FCA successfully unmixes the mixed matrices. Both free kurtosis and free entropy based unmixing have comparably small but not zero error, which we compute over $200$ Monte-Carlo realizations. This is expected since the matrices are asymptotically free and the simulations are with finite dimensional matrices.  

  \begin{figure}
          \centering
            \begin{subfigure}[t]{0.24\textwidth}
              \includegraphics[width = \textwidth,trim={0cm 0cm 0cm 0cm},clip]{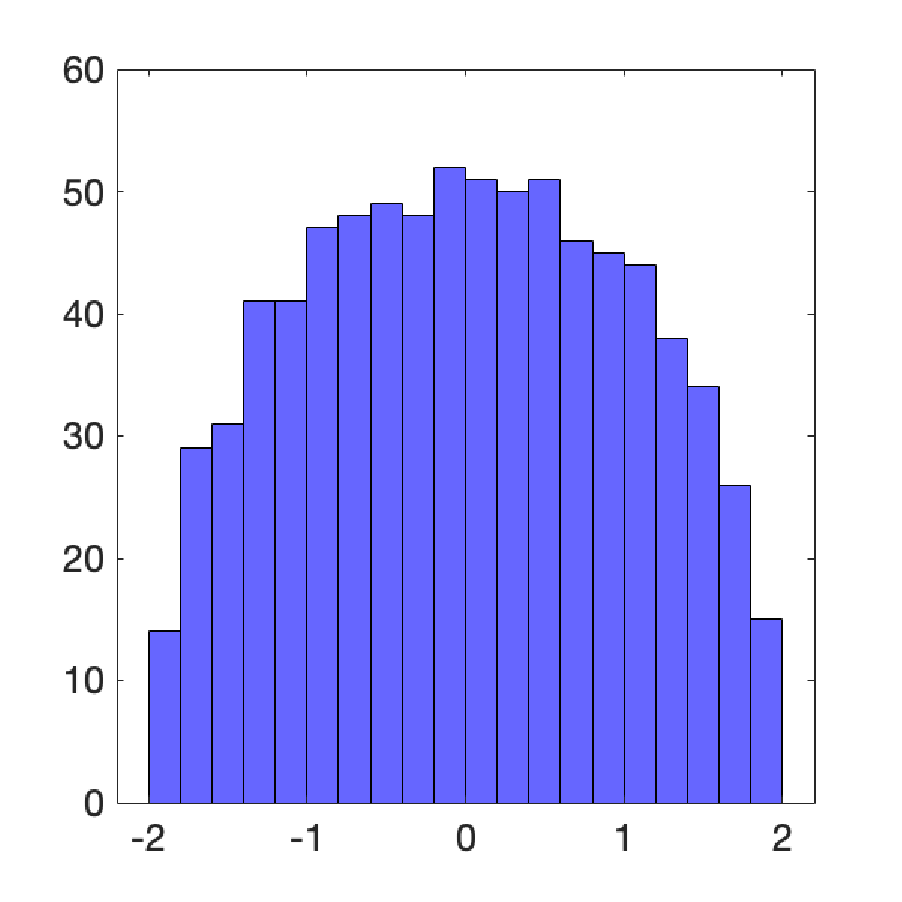}
          \caption{$\bm X_1$}
            \label{fig:x1}
            \end{subfigure}     
            \begin{subfigure}[t]{0.24\textwidth}
              \includegraphics[width = \textwidth,trim={0cm 0cm 0cm 0cm},clip]{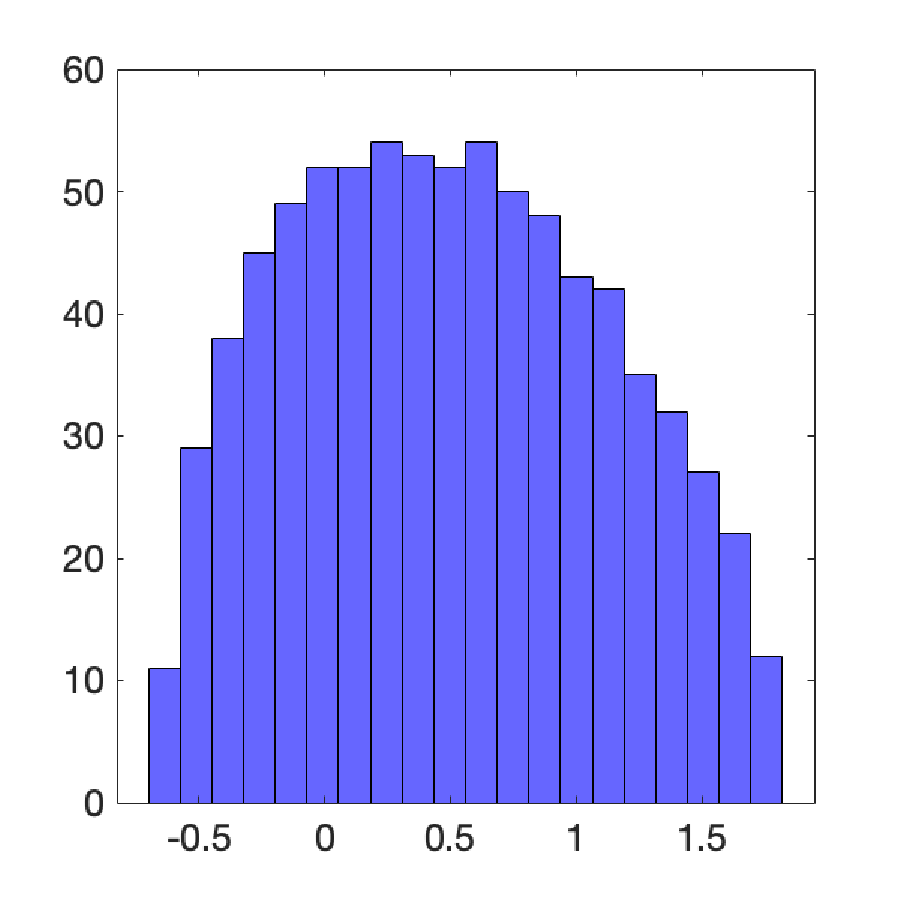}
          \caption{$\bm Z_1$}
            \label{fig:z1}
            \end{subfigure}           
            \begin{subfigure}[t]{0.24\textwidth}
              \includegraphics[width = \textwidth,trim={0cm 0cm 0cm 0cm},clip]{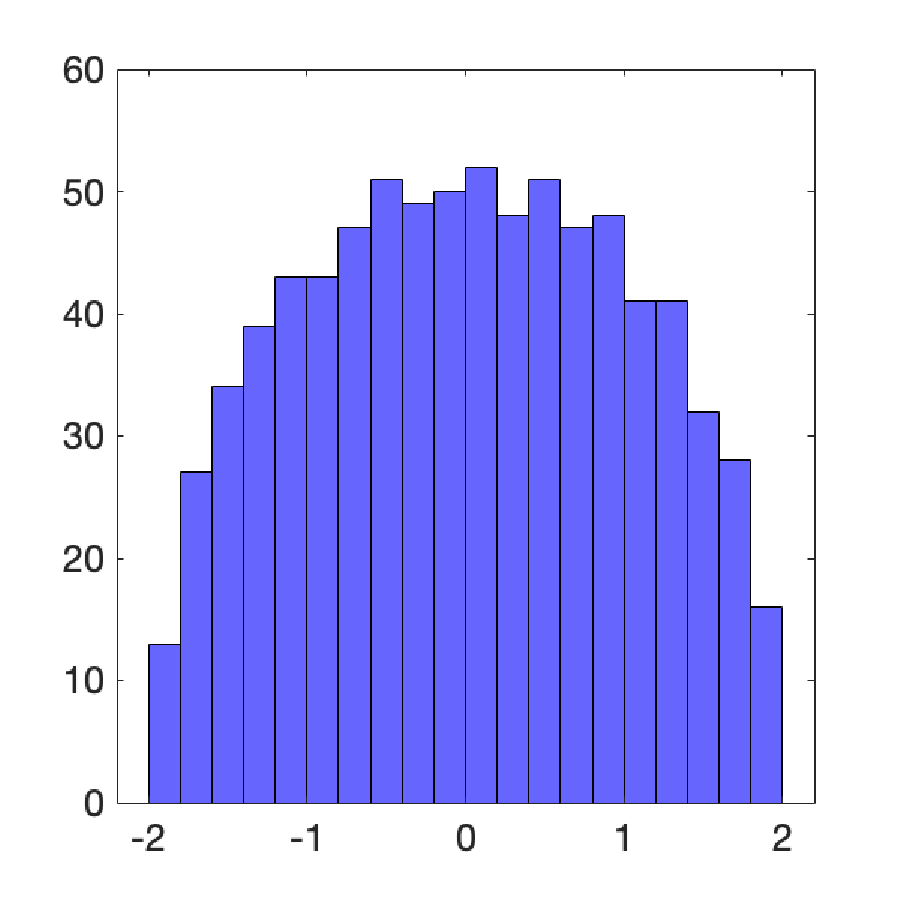}
          \caption{$\widehat{\bm X}_1$ via kurtosis-based FCA}
            \label{fig:xhat1_kurt}
            \end{subfigure}
             \begin{subfigure}[t]{0.24\textwidth}
              \includegraphics[width = \textwidth,trim={0cm 0cm 0cm 0cm},clip]{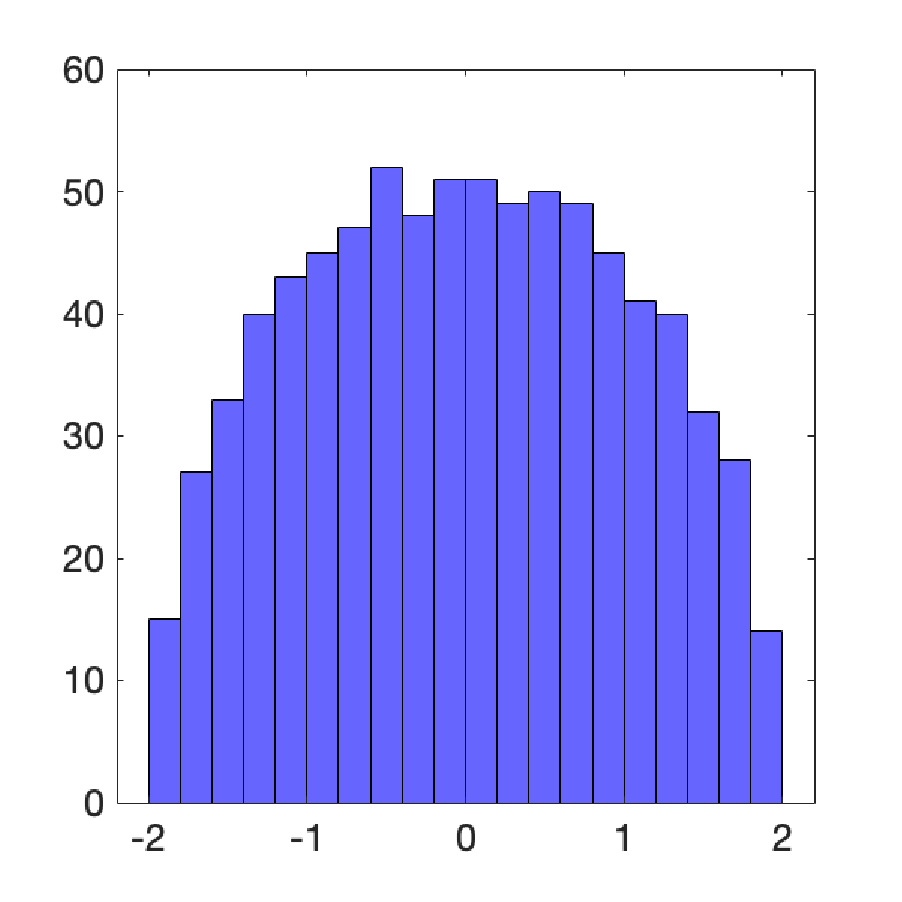}
          \caption{$\widehat{\bm X}_1$ via entropy-based FCA}
            \label{fig:xhat1_ent}
            \end{subfigure}
             \begin{subfigure}[t]{0.24\textwidth}
              \includegraphics[width = \textwidth,trim={0cm 0cm 0cm 0cm},clip]{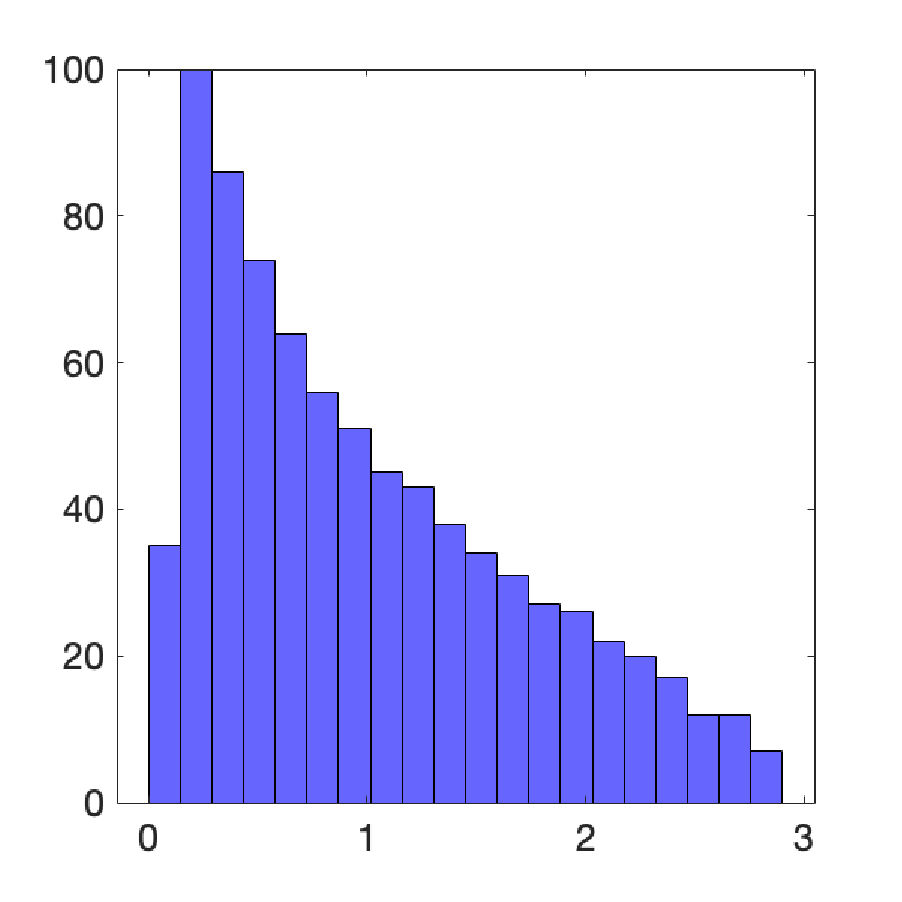}
          \caption{$\bm X_2$}
            \label{fig:x2}
            \end{subfigure}  
             \begin{subfigure}[t]{0.24\textwidth}
              \includegraphics[width = \textwidth,trim={0cm 0cm 0cm 0cm},clip]{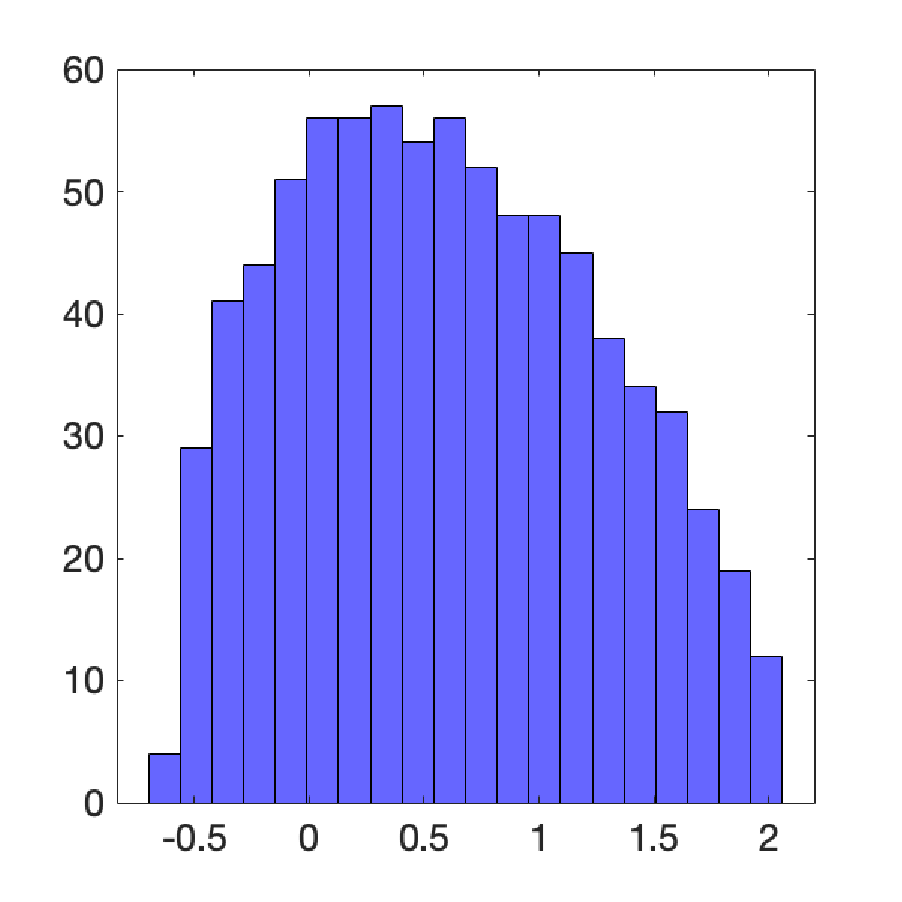}
          \caption{$\bm Z_2$}
            \label{fig:z2}
            \end{subfigure}
            \begin{subfigure}[t]{0.24\textwidth}
              \includegraphics[width = \textwidth,trim={0cm 0cm 0cm 0cm},clip]{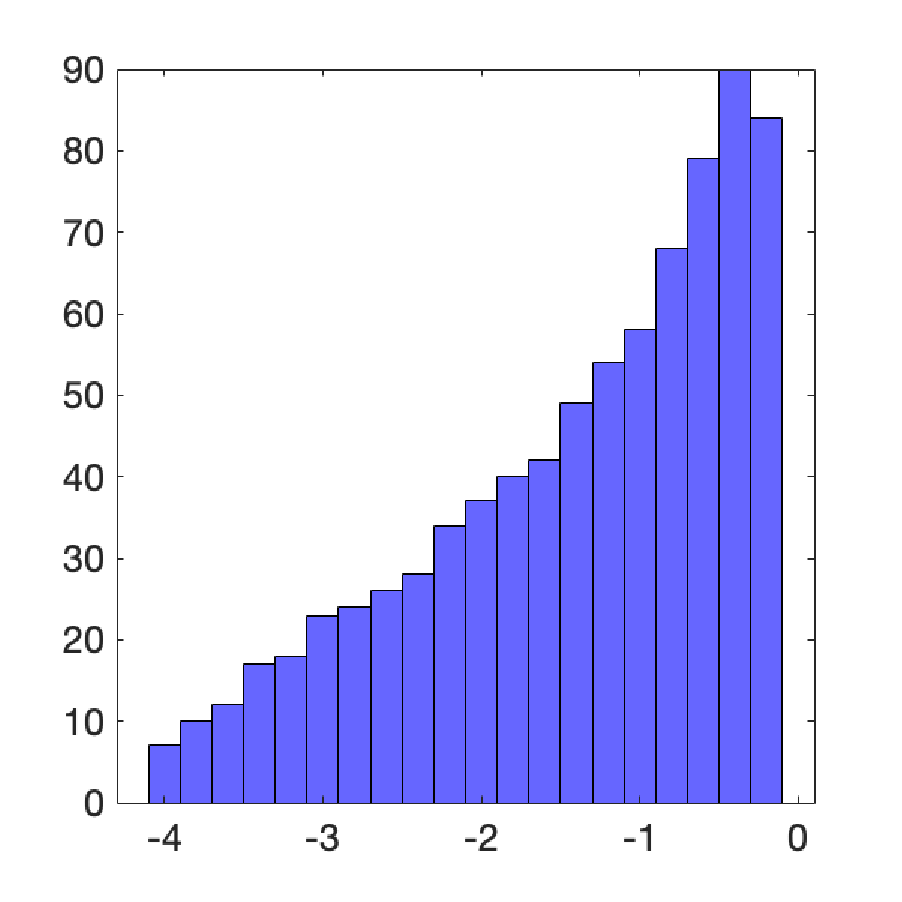}
          \caption{$\widehat{\bm X}_2$ via kurtosis-based FCA}
            \label{fig:xhat2_kurt}
            \end{subfigure}
            \begin{subfigure}[t]{0.24\textwidth}
              \includegraphics[width = \textwidth,trim={0cm 0cm 0cm 0cm},clip]{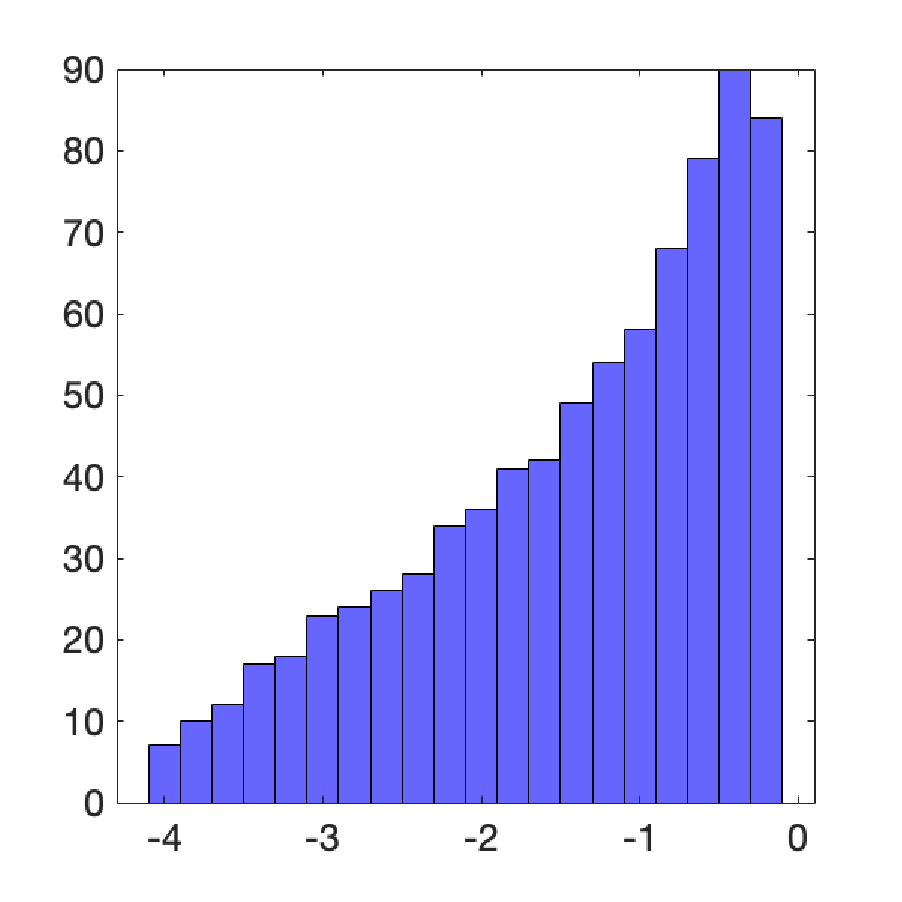}
          \caption{$\widehat{\bm X}_2$ via entropy-based FCA}
            \label{fig:xhat2_ent}
            \end{subfigure}
            \caption{An experiment in the separation of self-adjoint matrices. 
          The mixing matrix $\bm A = [0.5,0.5;-0.5,0.6]$, $N = 800, M = 1600$. 
          The average errors over 200 trials are $ 8.67\times 10^{-3}$  for kurtosis-based FCF and $ 6.44\times 10^{-3}$ for entropy-based FCF.
          }
           \label{fig:sym_rm_separation}
            \end{figure}

\subsection{Unmixing of rectangular matrices with rectangular FCA}
\label{sec:rec_rm_separation}

We now show that the rectangular FCA can successfully, while not perfectly, unmix mixtures of rectangular matrices. To that end, we let $\bm{X}_2$ be an $N \times M$ Gaussian matrix with i.i.d. $\mathcal{N}(0,1/M)$ entries and set \(\bm{X}_1 = \bm{UDV}^T\) where 
           \begin{equation*}
        \begin{aligned}
          \bm{U}(i,j) &= \frac{2}{N} \sin{[\frac{\pi}{2N}(i + 1)(2j+1)]} \qquad \text{for \(1\leq i,j\leq N\)} \\
            \bm{V}(i,j) &= \frac{2}{M} \cos{[\frac{\pi}{2M}i(2j+1)]} \qquad \text{for \(1\leq i,j\leq M\)},
        \end{aligned}
        \end{equation*}
 so that $\bm{U}$ and $\bm{V}$ thus constructed are orthogonal matrices. We 
 pick a `nice' function \(f(z)\) and set the diagonal matrix $\bm D \in \mathbb{R}^{N \times M}$ such that   \(\bm{D}(i,i) = f((i-1/2)/N)), i = 1, \cdots, N .\)
  
 The singular value spectra of $\bm{X}_1$ and $\bm{X}_2$ are plotted in Figures \ref{fig:x1_rec} and \ref{fig:x2_rec}. As before, we mix the matrices as in  \eqref{eq:general matrix model}. Figures \ref{fig:z1_rec} and \ref{fig:z2_rec} display the singular value spectra of $\bm{Z}_1$ and $\bm{Z}_2$.
 
 We now note that the singular value distributions of  $\bm{X}_1$ and $\bm{X}_2$ converge to a non-random limit and that the distribution of $\bm{X}_2$ is bi-orthogonally invariant. Thus, following the discussion after Definition \eqref{def:asy_ind_rec}, $\bm X_1$ and $\bm X_2$ are asymptotically free. 
 Moreover, only $\bm X_2$ has a limiting distribution which converges to that of an abstract free  Poisson rectangular element.
 
 Hence, we can apply rectangular FCA to factorize $\bm{Z} = \begin{bmatrix} \bm{Z}_1 & \bm{Z}_2 \end{bmatrix}^T$  using Algorithm \ref{alg:fcf_pro}  and obtain estimates $\widehat{\bm A}$, $\widehat{\bm{X}}_1$ and $\widehat{\bm{X}}_2$ which should be good estimates of $\bm{A}$, $\bm{X}_1$ and $\bm{X}_2$ respectively. 
 
 Figures \ref{fig:xhat1_kurt_rec} and \ref{fig:xhat2_kurt_rec} display the eigen-spectra of the matrices $\widehat{\bm{ X}}_1$ and $\widehat{\bm{X}}_2$ returned by rectangular free kurtosis-based FCA. Comparing Figures \ref{fig:xhat1_kurt_rec}, \ref{fig:xhat2_kurt_rec} and  \ref{fig:x1_rec}, \ref{fig:x2_rec} reveals that rectangular free kurtosis-based FCA successfully unmixes the mixed matrices  well. Figures \ref{fig:xhat1_ent_rec} and \ref{fig:xhat2_ent_rec} show that rectangular free entropy based FCA successfully unmixes the mixed matrices. Both free kurtosis and free entropy based unmixing have comparably small but not zero error, which we compute over $200$ Monte-Carlo realizations. This is expected since the matrices are asymptotically free but the simulations are with finite dimensional matrices.

 \begin{figure}
          \centering
            \begin{subfigure}[t]{0.24\textwidth}
              \includegraphics[width = \textwidth,trim={0cm 0cm 0cm 0cm},clip]{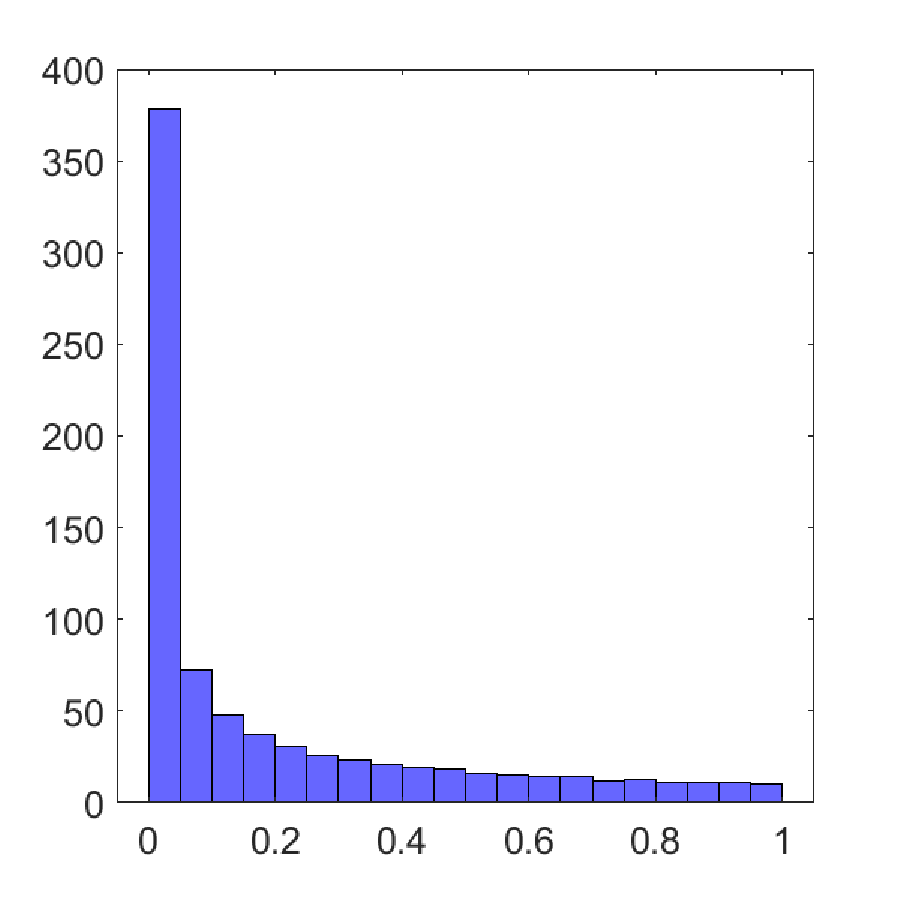}
          \caption{$\bm X_1$}
            \label{fig:x1_rec}
            \end{subfigure}
           \begin{subfigure}[t]{0.24\textwidth}
              \includegraphics[width = \textwidth,trim={0cm 0cm 0cm 0cm},clip]{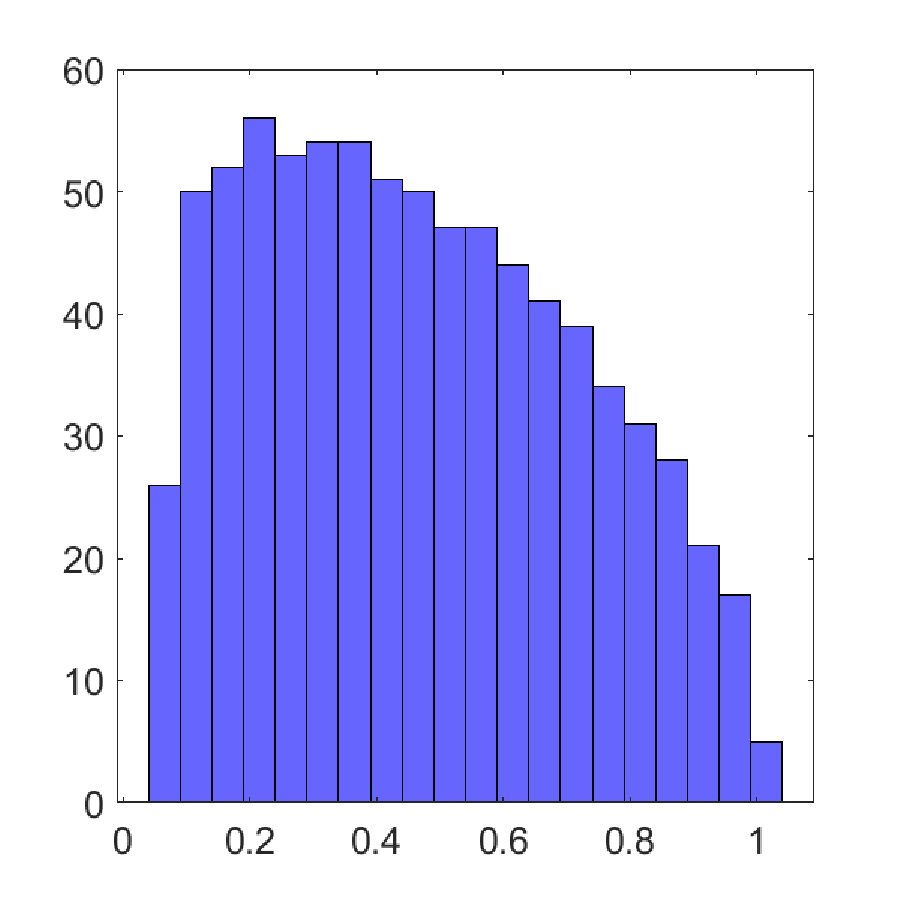}
          \caption{$\bm Z_1$}
            \label{fig:z1_rec}
            \end{subfigure}   
               \begin{subfigure}[t]{0.24\textwidth}
              \includegraphics[width = \textwidth,trim={0cm 0cm 0cm 0cm},clip]{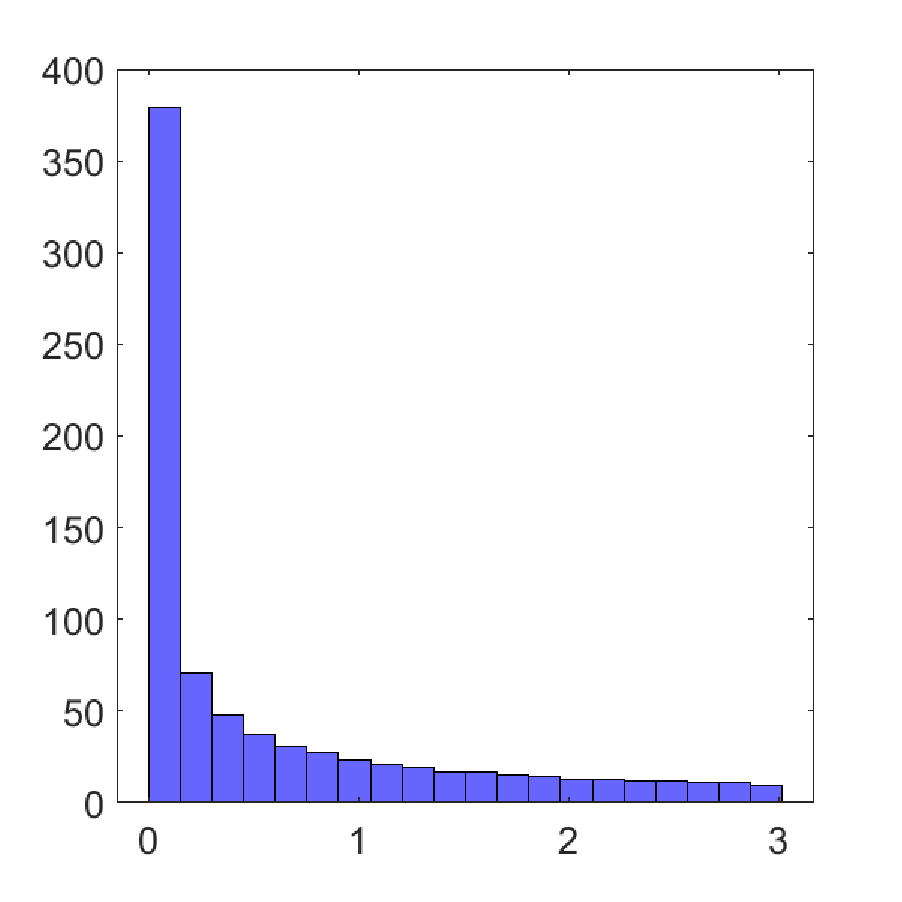}
          \caption{$\widehat{\bm X}_1$ via kurtosis-based FCA}
            \label{fig:xhat1_kurt_rec}
            \end{subfigure}
        \begin{subfigure}[t]{0.24\textwidth}
              \includegraphics[width = \textwidth,trim={0cm 0cm 0cm 0cm},clip]{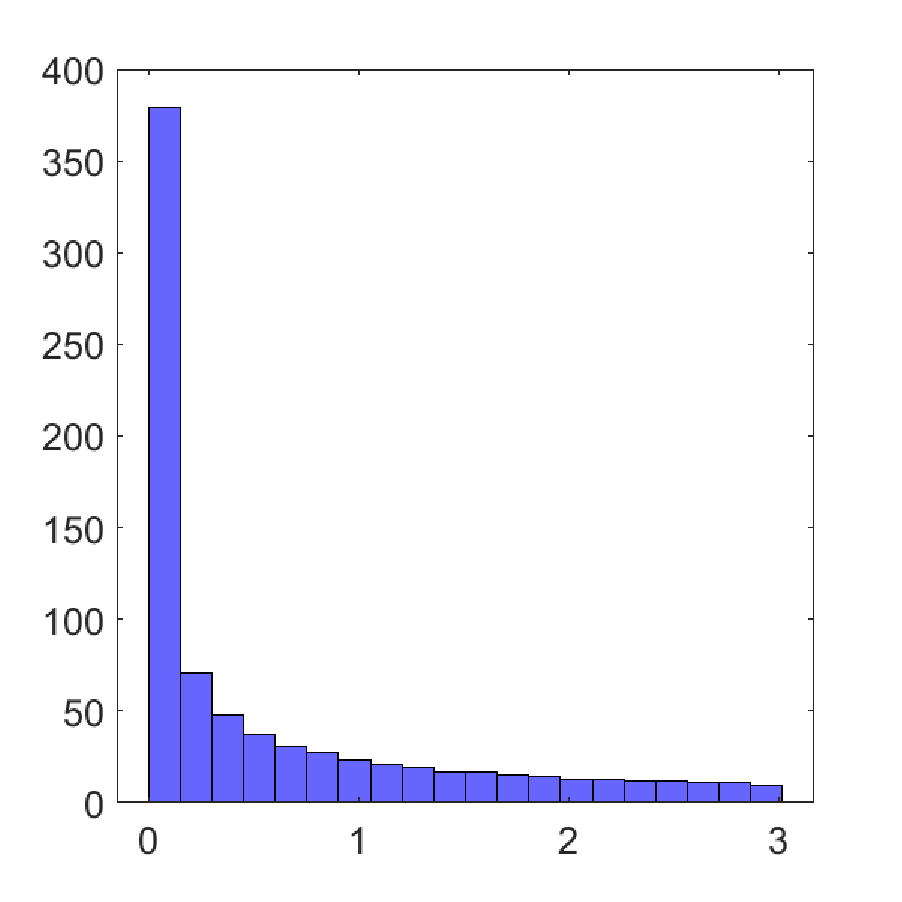}
          \caption{$\widehat{\bm X}_1$ via entropy-based FCA}
            \label{fig:xhat1_ent_rec}
            \end{subfigure}
            \begin{subfigure}[t]{0.24\textwidth}
              \includegraphics[width = \textwidth,trim={0cm 0cm 0cm 0cm},clip]{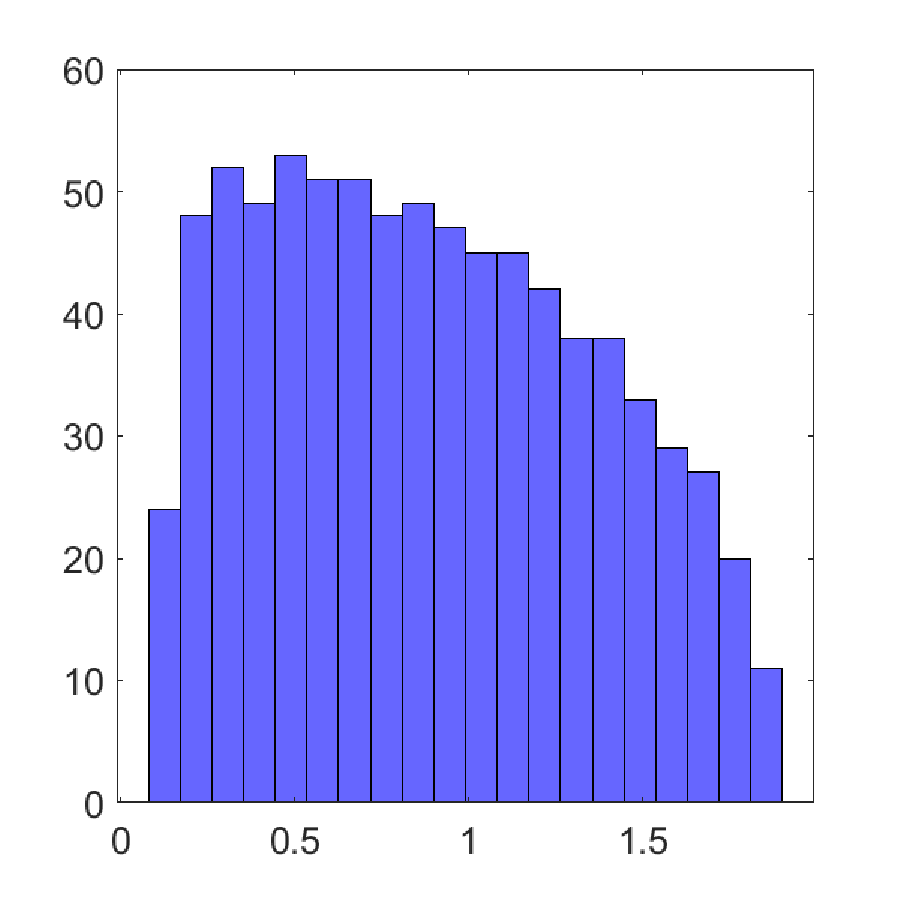}
          \caption{$\bm X_2$}
            \label{fig:x2_rec}
            \end{subfigure}                    
            \begin{subfigure}[t]{0.24\textwidth}
              \includegraphics[width = \textwidth,trim={0cm 0cm 0cm 0cm},clip]{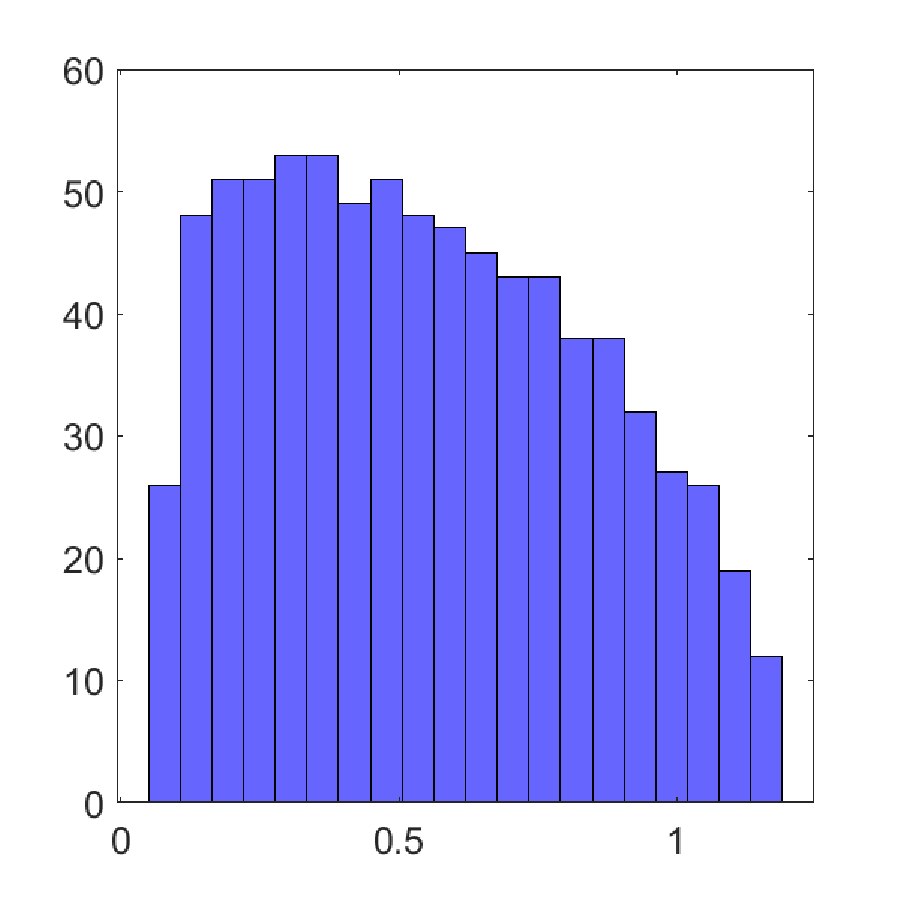}
          \caption{$\bm Z_2$}
            \label{fig:z2_rec}
            \end{subfigure}
            \begin{subfigure}[t]{0.24\textwidth}
              \includegraphics[width = \textwidth,trim={0cm 0cm 0cm 0cm},clip]{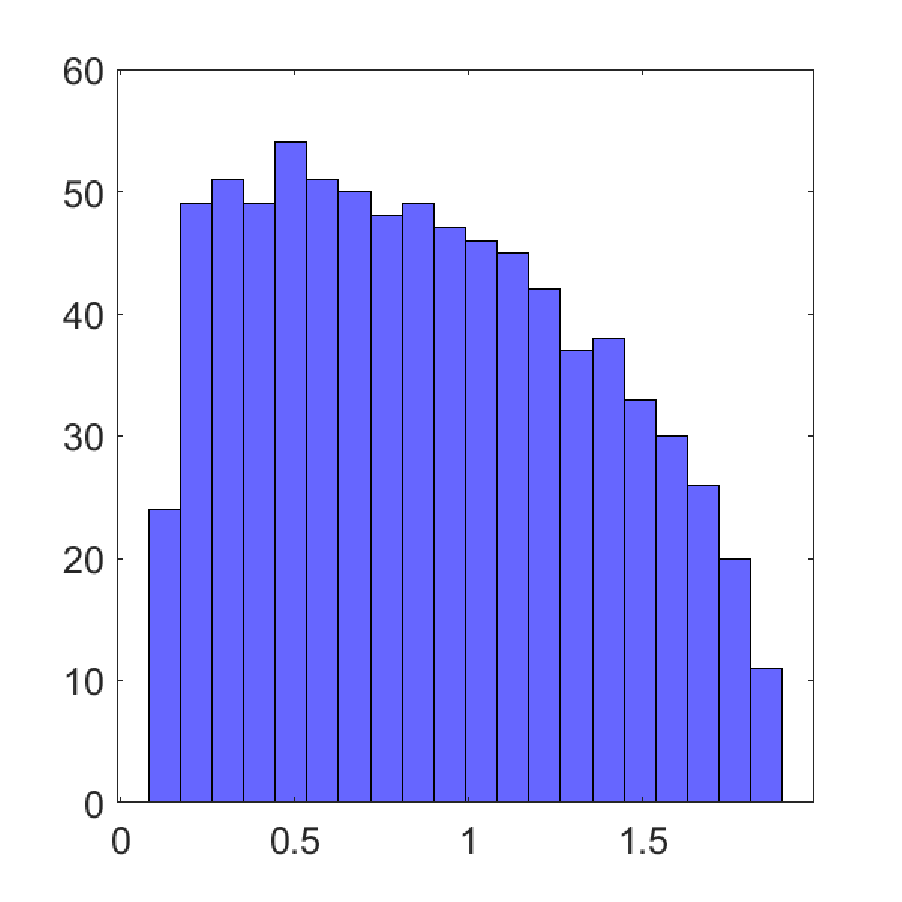}
          \caption{$\widehat{\bm X}_2$ via kurtosis-based FCA}
            \label{fig:xhat2_kurt_rec}
            \end{subfigure}
            \begin{subfigure}[t]{0.24\textwidth}
              \includegraphics[width = \textwidth,trim={0cm 0cm 0cm 0cm},clip]{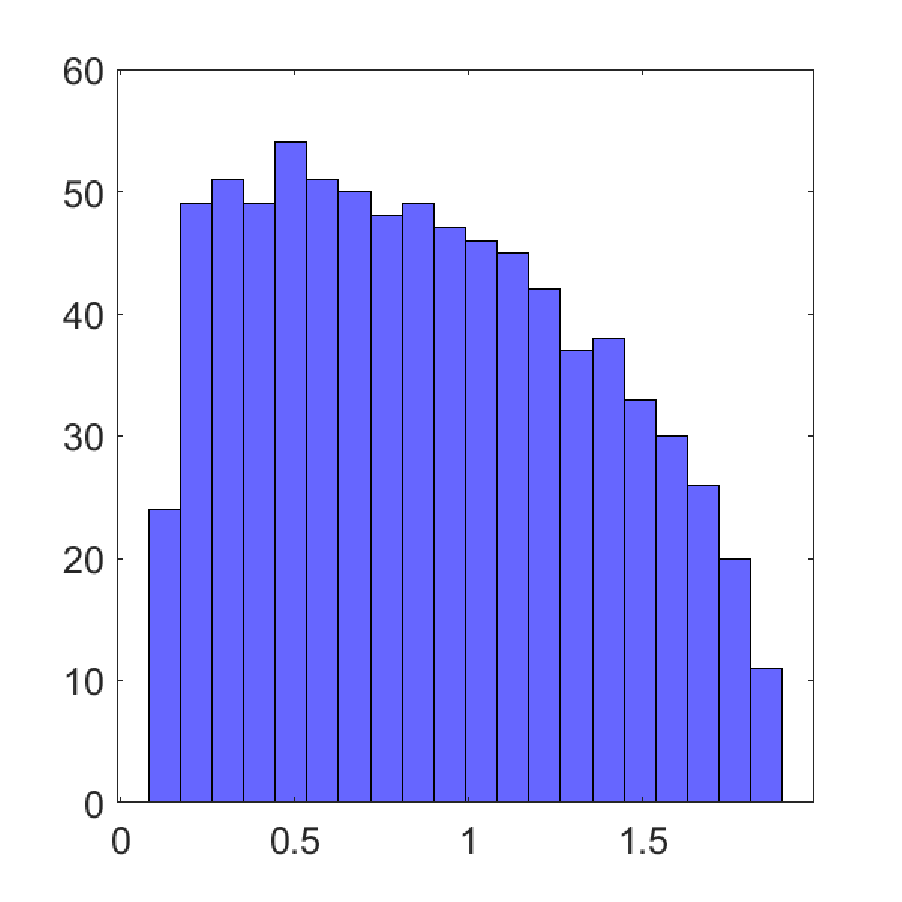}
          \caption{$\widehat{\bm X}_2$ via entropy-based FCA}
            \label{fig:xhat2_ent_rec}
            \end{subfigure}
            \caption{An experiment in the separation of rectangular random matrices.        
            In this simulation, $\bm A = [0.5,0.5;-0.5,0.6]$, $N = 800, M = 1000$ and $f(x) = (x-1)^4$. 
            The average errors over 200 trials are $ 1.55\times 10^{-3}$  for kurtosis-based FCF and $ 8.81\times 10^{-4}$ for entropy-based FCF. 
}
\label{fig:rec_rm_separation}
\end{figure}

\subsection{Unmixing mixed images using rectangular FCA}
\label{sec:locust}

We now consider the problem of unmixing mixed images. Grayscale images can be viewed as matrices so rectangular FCA can be used to separate the mixed images. We can also apply ICA to unmix the images via reshaping images to vectors and we shall compare the unmixing performance of FCA with that of ICA. Algorithm \ref{alg:icf_kurt} in Appendix \ref{sec:algica} describes the independent component factorization (ICF) mirroring the language we used for FCF.

We set $\bm{X}_1$ to be the grayscale image  of the locust in Figure \ref{fig::locust}. The  matrix $\bm{X}_2$ is a Gaussian random matrix of the same size with i.i.d. zero mean, unit variance entries as depicted in  Figure \ref{fig::Gau_noise}. We mix the matrices following \eqref{eq:general matrix model} and display the mixed images in Figures \ref{fig::mix1} and \ref{fig::mix2}.

Next, we apply (rectangular) free kurtosis based FCA to the mixed images $\bm Z = [\bm Z_1, \bm Z_2]^T$ and display the unmixed image that is closes to that of the locust in  Figure \ref{fig::fca_recover}. The unmixed image obtained by using (classical) kurtosis based ICA is displayed in Figure \ref{fig::ica_recover}. Both methods return unmixed images that are close to the original image of the locust. A closer inspection of Figures \ref{fig::fca_recover}  \ref{fig::ica_recover} reveals that FCA better unmixes the images than ICA as illustrated in Figures \ref{fig::fca_zoom_in} and  \ref{fig::ica_zoom_in}. Quantitatively speaking, when averaged over $200$ Monte-Carlo realizations of the noise, we find that the denoising error for kurtosis-based FCF is $5.35\times 10^{-3}$ while the error for kurtosis-based ICF is $2.42\times 10^{-1}$, thereby illustrating the superiority of FCF over ICF for this task. 

To gain additional insight on the improved unmixing performance of FCA relative to ICA for this example, we investigate the landscape of their respective objective functions. To that end we first note that the mixing matrix 
$$\bm{A} =  \dfrac{1}{2} \begin{bmatrix} \sqrt{2} &  \sqrt{2}\\ \sqrt{2} & -\sqrt{2}\end{bmatrix},$$
is orthogonal and so we can recast the spherical manifold optimization problem underlying FCA and ICA as a 1-D optimization problem in polar coordinates. In other words, we can parameterize the optimization problem in (\ref{eq:ica_opt}) and (\ref{eq:fca_opt}) in terms of $\bm{w} := \bm{w}_{\theta} = \begin{bmatrix} \cos(\theta) & \sin(\theta) \end{bmatrix}^T$. Similarly, optimization \eqref{eq:ent1} can be parameterized with $\bm{W} = [ \cos(\theta)  \sin(\theta); -\sin(\theta), \cos(\theta) ]^T$.
We compute and display the landscape of the objective functions corresponding to maximization of the classical kurtosis $\abs*{\cumu_4(\theta)}$, free kurtosis $\abs*{\clc_4(\theta)}$ and the free entropy $E(\theta)$ for $\theta \in [0,2 \pi]$ in Figures \ref{fig::icf_kurtosis},  \ref{fig::fca_kurtosis_polar} and \ref{fig::fca_entropy_polar}, respectively.

The dashed red line in these figures corresponds to the ground truth freely independent component direction associated with $\theta_{1} = \pi/4$ associated with the first column of the mixing matrix $\bm{A}$; the other direction (not displayed) is orthogonal and corresponds to the second column of $\bm{A}$ and is associated with $\theta_2 = 3\pi/4$.

Figures \ref{fig::fca_kurtosis_polar} and \ref{fig::fca_entropy_polar}, reveal that  $\abs*{\cumu(\theta)}$ and $E(\theta)$ are maximized at angles $\theta$ very close to $\theta_1 = \pi/4$. In contrast, Figure \ref{fig::icf_kurtosis} reveals that  $\abs*{\clc(\theta)}$ is maximized at an angle $\theta$ further away from $\theta_1$ than is the case for the FCA algorithms. This is why FCA better unmixes the images than ICA.

There is a more interesting story in these plots. Figure \ref{fig::fca_kurtosis_polar} shows that the $\abs*{\cumu(\theta_2)}$ for $\theta_2 = 3\pi/4$ is very close to zero, as expected because $\bm{X}_2$ is a Gaussian random matrix and in the large matrix limit the free rectangular kurtosis of its free counterpart is identically zero. The classical kurtosis of a Gaussian random variable is also zero. A closer inspection reveals that the classical kurtosis of the locust image is also close to zero (the scale of the polar plot initially obscures this fact!) while its free kurtosis is significantly greater than zero (or that of $\bm{X}_2$). 

The fact that the locust image  $\bm{X}_1$ and the Gaussian image $\bm{X}_2$  have a higher ``contrast'' in their free kurtosis  relative to their classic kurtosis is why FCA does better at unmixing them than ICA. Figure \ref{fig:diagmram} captures this perspective and suggests a direction for future research in more precisely defining how the ``contrasts'' between the scalar (or ICA) and matrix (or FCA) embeddings affects the realized unmixing performance.

  \begin{figure}
    \centering
            \begin{subfigure}[t]{0.32\textwidth}
            \includegraphics[width = \textwidth,trim={0cm 0cm 0cm 0cm},clip]{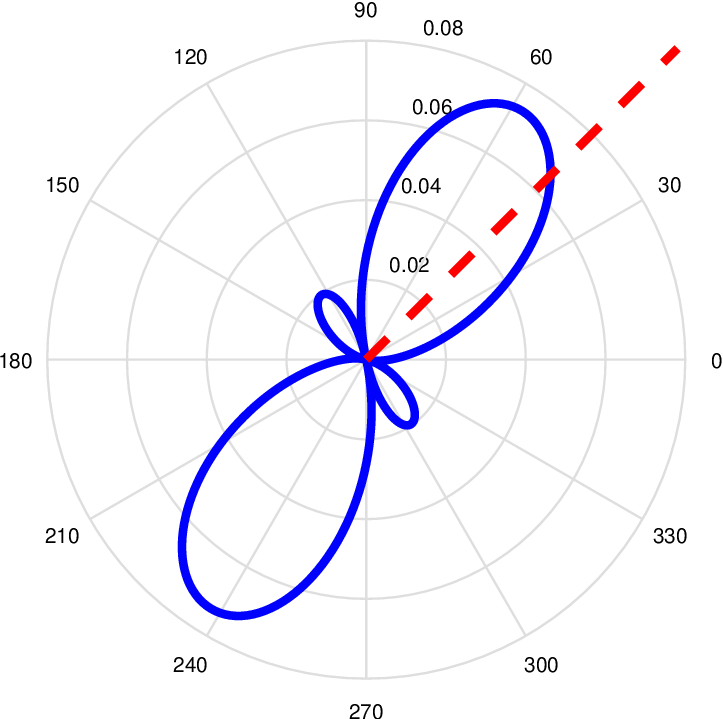}
            \caption{Polar graph of \(\abs*{\cumu(\theta)}\)}
            \label{fig::icf_kurtosis}
            \end{subfigure}
            \begin{subfigure}[t]{0.32\textwidth}
            \includegraphics[width = \textwidth,trim={0cm 0cm 0cm 0cm},clip]{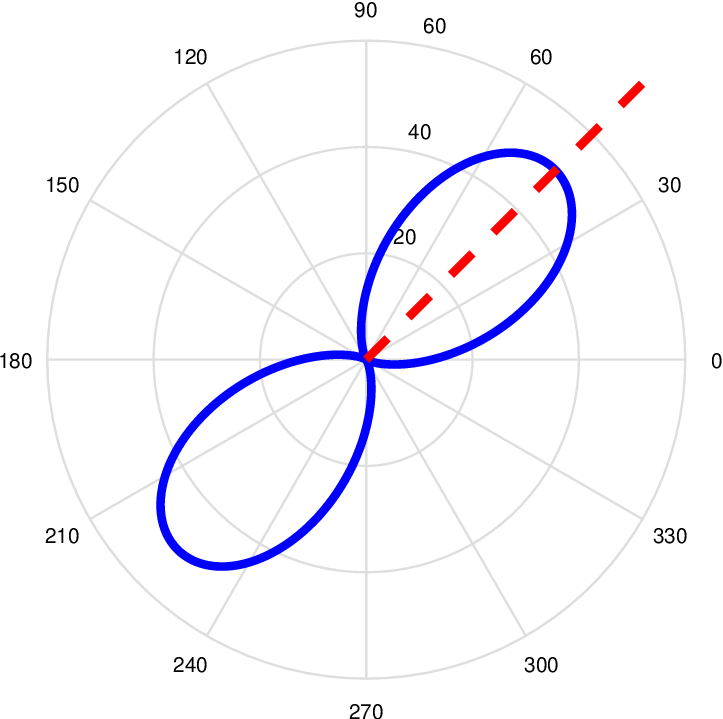}
            \caption{Polar graph of \(\abs*{\clc(\theta)}\)}
            \label{fig::fca_kurtosis_polar}
            \end{subfigure}
            \begin{subfigure}[t]{0.32\textwidth}
            \includegraphics[width = \textwidth,trim={0cm 0cm 0cm 0cm},clip]{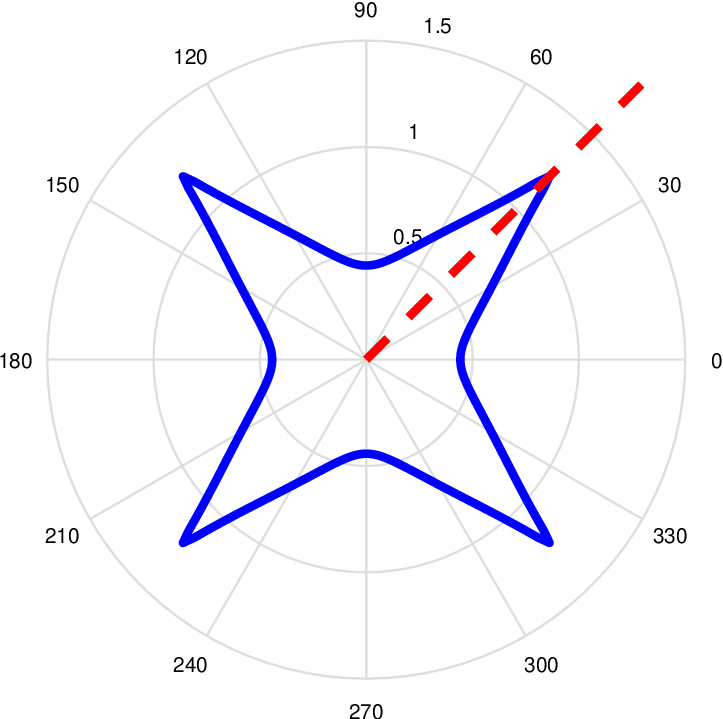}
            \caption{Polar graph of \(E(\theta)\)}
            \label{fig::fca_entropy_polar}
            \end{subfigure}
            \caption{(a) Polar graph of \(c(\theta)\) (b) Polar graph of \(\kappa(\theta)\) (c) Polar graph of \(E(\theta)\). The red dashed lines stand for the direction of $45^\circ$. Note that the directions of maximum of $\kappa(\theta)$ and $E(\theta)$ agree with the red line well, while the maximum of $c(\theta)$ is off. Because of the randomness of Gaussian noise, for different trials, the polar graph of ICA will vary. However, the  polar graphs of FCA are relatively stable.}
  \end{figure}

\subsection{Unmixing performance of free kurtosis vs free entropy FCA vs ICA}


We now compare the performance of FCF and ICF as a function of the dimensionality of the system, since the errors in FCF and ICF are both governed by the deviation from some limiting large sample quantities (or large matrix size).
Here, we adopt the same setup as in Section \ref{sec:rec_rm_separation} with whitened $\bm{X}_1$ and $\bm{X}_2$ matrices and $\bm{A} = [\sqrt{2}, \sqrt{2}; -\sqrt{2}, \sqrt{2}]/2 $. We increase \(N,M\) in a fixed ratio, and obtain an estimate of the unmixing matrix using kurtosis based FCA, entropy based FCA, kurtosis based ICA and entropy based ICA and compute the unmixing error over 200 Monte-Carlo realizations. 

Figures \ref{fig::conv_kurt} and  \ref{fig::cdf_kurt} show that free kurtosis based FCA and kurtosis based ICA realize similar unmixing performance. However, Figures \ref{fig::conv_entropy} and  \ref{fig::cdf_entropy}  show that 
free entropy based FCA has a lower error than entropy based ICA, while both have errors that decay at the same rate. 

In order to compare FCF and ICF in practical setting, we repeat the above separation experiments with $\bm{X}_1$ and $\bm{X}_2$ being images from Caltech-256 Dataset \cite{griffin2007caltech}. Since images there are of various sizes, for every pair of images $\bm{X}_1$ and $\bm{X}_2$, we first align them by their top left corners and crop the common parts. The resulting equal-sized images then get whitened and mixed. Finally, kurtosis, entropy based FCA and ICA are applied to compute the unmixing error. We run the experiment over 1000 random sampled image pairs from Caltech-256 Dataset. We find that for both kurtosis and entropy based methods, FCA is comparable to ICA when applied to pratical image separation (see Figure \ref{fig::k4_cluster} and \ref{fig::ent_cluster} ). Also, we note that the error of entropy based FCA is below 1.5 times error of entropy based ICA for 993 out of 1000 samples, while we don't have a theoretical explanation yet.

\begin{figure}
          \centering
            \begin{subfigure}{0.45\textwidth}
              \includegraphics[width = \textwidth,trim={0cm 0cm 0cm 0 },clip]{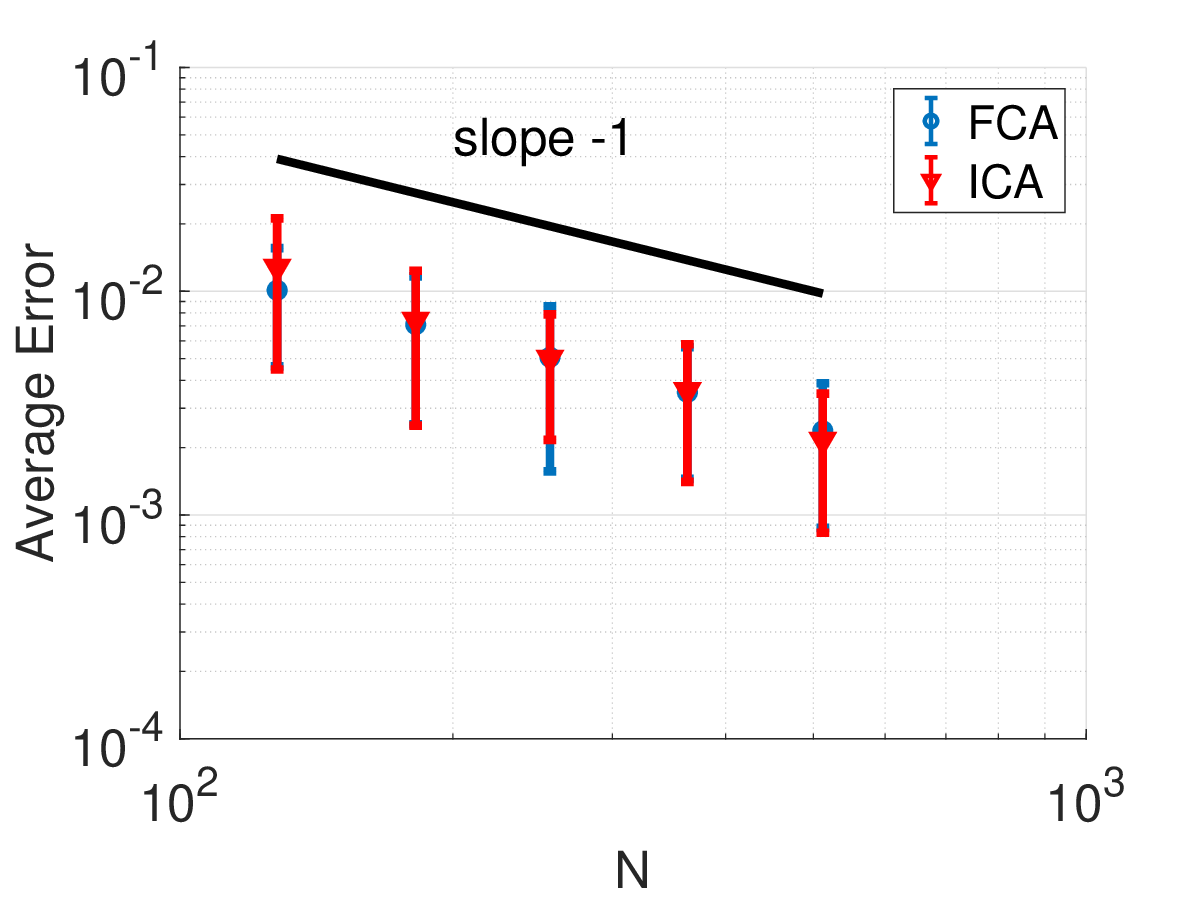}
                \caption{Kurtosis based Methods}
                \label{fig::conv_kurt}
            \end{subfigure}
            \begin{subfigure}{0.45\textwidth}
              \includegraphics[width = \textwidth,trim={0cm 0cm 0cm 0cm },clip]{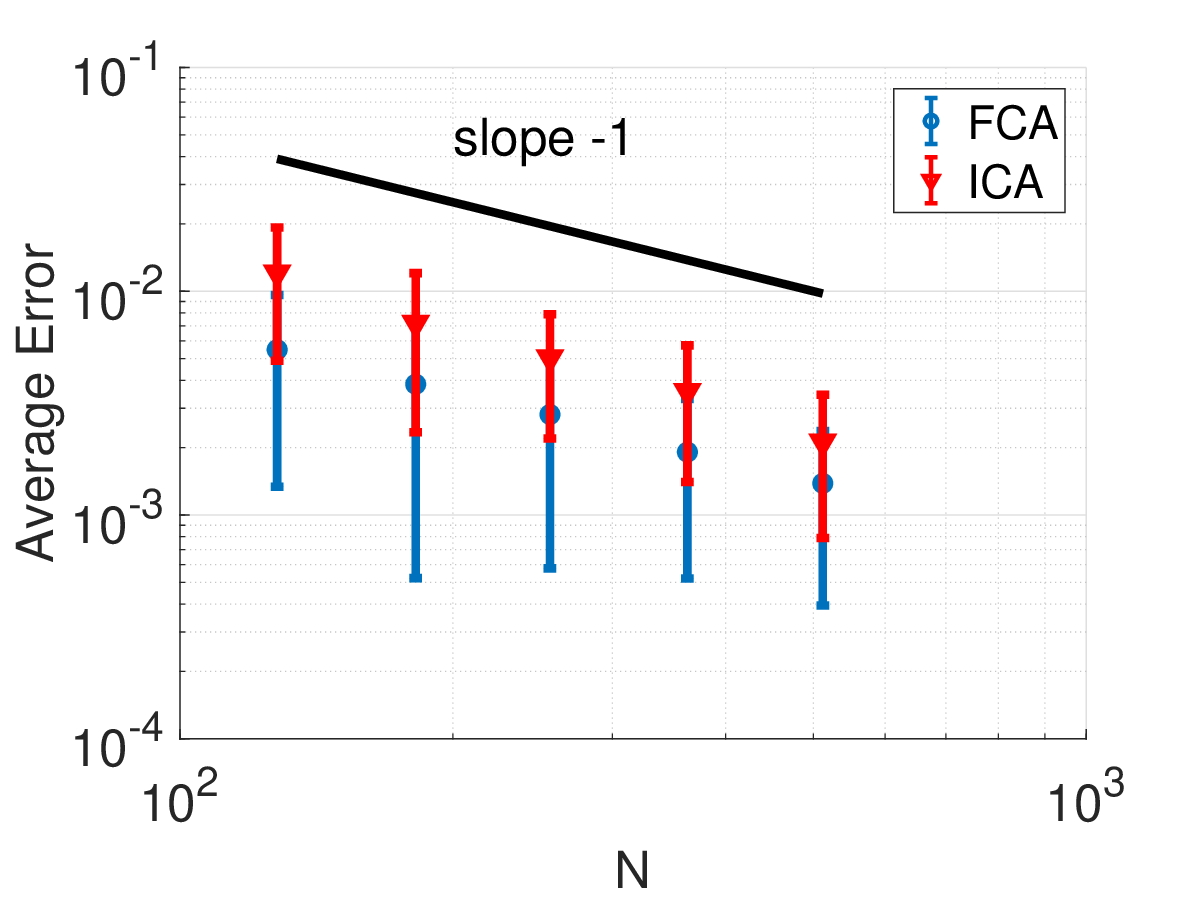}
                \caption{Entropy based Method}
                \label{fig::conv_entropy}
            \end{subfigure} \\
            \begin{subfigure}{0.45\textwidth}
              \includegraphics[width = \textwidth, trim={0 0 0 0 },clip]{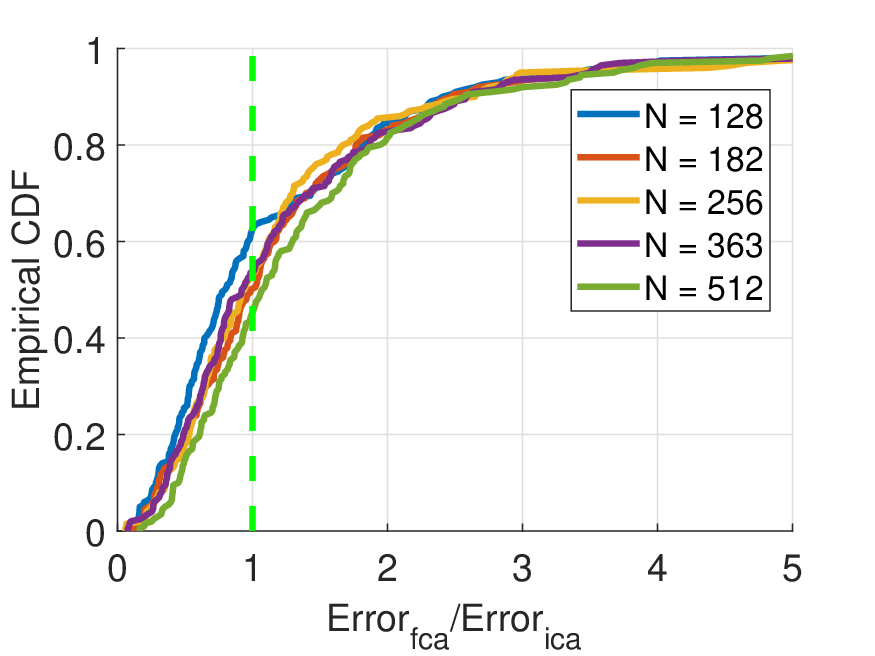}
                \caption{\(\text{Error}_{fca}/\text{Error}_{ica}\) (Kurtosis)}
                \label{fig::cdf_kurt}
            \end{subfigure}
            \begin{subfigure}{0.45\textwidth}
              \includegraphics[width = \textwidth, trim={0 0 0 0 },clip]{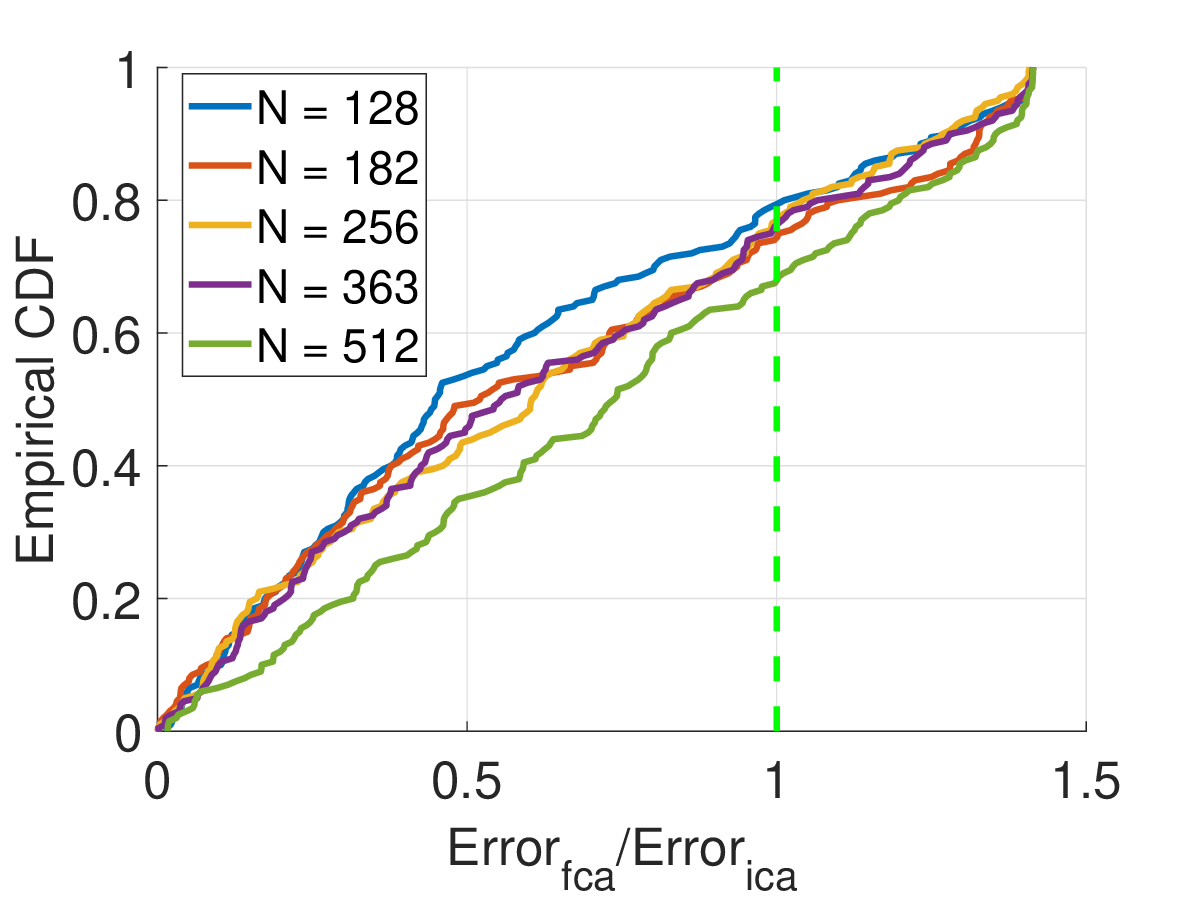}
                \caption{\(\text{Error}_{fca}/\text{Error}_{ica}\) (Entropy)}
                \label{fig::cdf_entropy}
            \end{subfigure}
            \caption{(a) Averaged (over 200 trials) errors of kurtosis based FCA and ICA for increasing dimension. (b) Averaged error of kurtosis based FCA and ICA for increasing dimension. (c) CDF of \(\text{Error}_{fca}/\text{Error}_{ica}\) for the kurtosis-based method. (d) CDF of \(\text{Error}_{fca}/\text{Error}_{ica}\) for the entropy-based method. All methods appear to have a convergence rate of \(N^{-1}\). In this simulation, we set \(N/M = 0.8\) and \(f(x) = (x-1)^4\).}
            \label{fig:fca comparison}
\end{figure}

\begin{figure}
          \centering
            \begin{subfigure}{0.45\textwidth}
              \includegraphics[width = \textwidth,trim={0cm 0cm 0cm 0cm},clip]{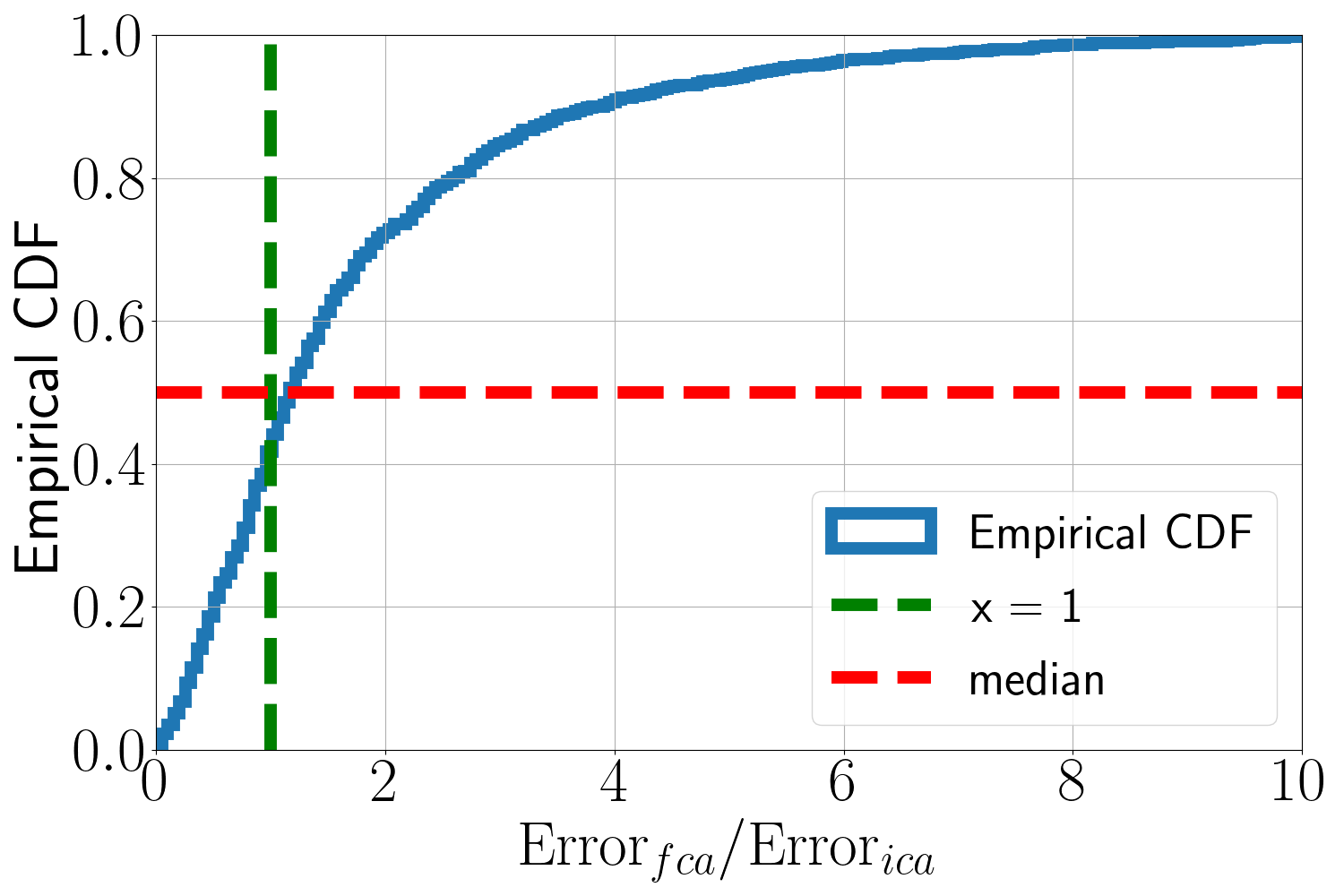}
                \caption{Kurtosis based Methods}
                \label{fig::k4_cluster}
            \end{subfigure}
            \begin{subfigure}{0.45\textwidth}
              \includegraphics[width = \textwidth,trim={0cm 0cm 0cm 0cm },clip]{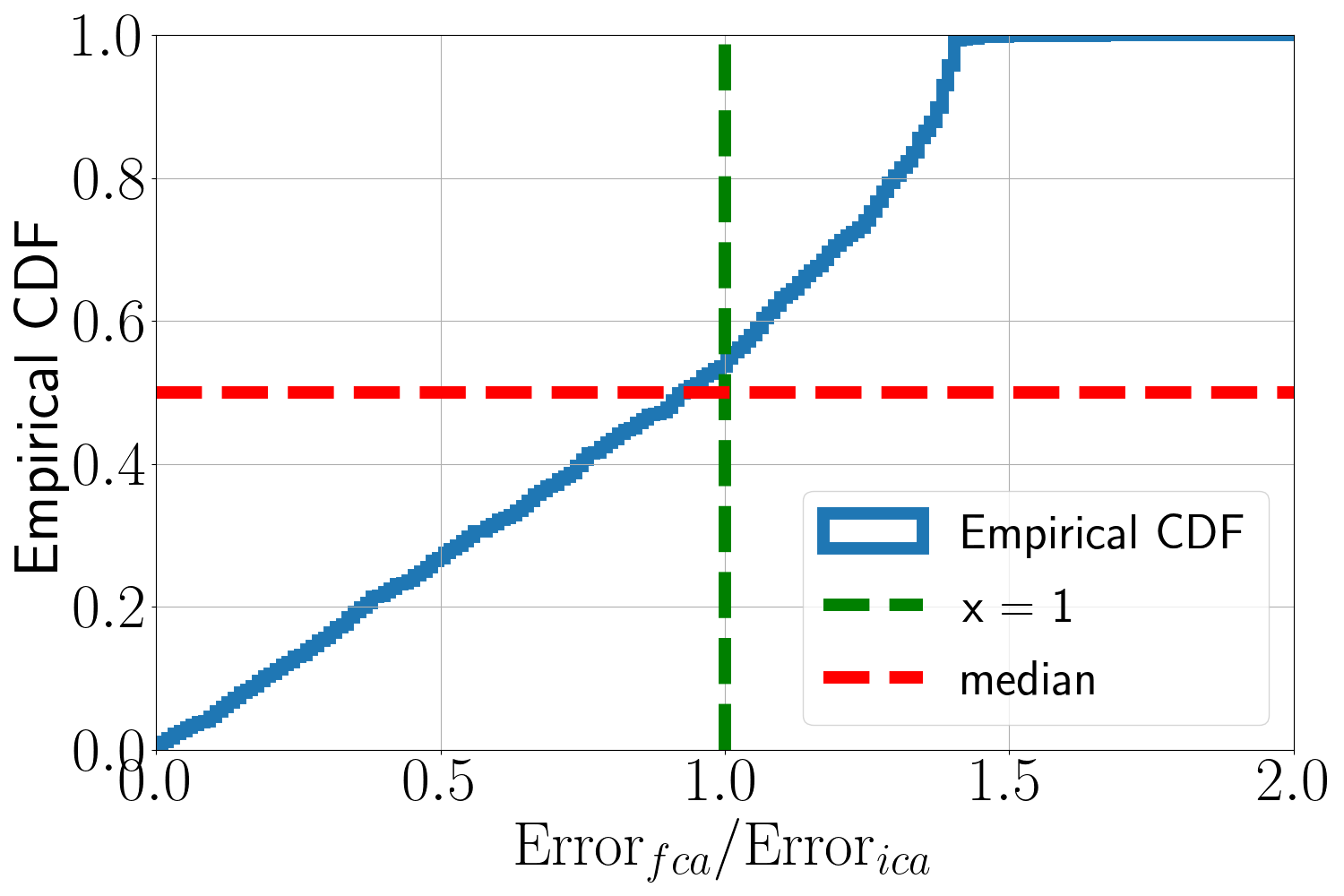}
                \caption{Entropy based Method}
                \label{fig::ent_cluster}
            \end{subfigure} 
            \caption{Comparison of FCA and ICA on image separation for 1000 random sampled image pairs fropm Caltech-256 dataset; (a) CDF of \(\text{Error}_{fca}/\text{Error}_{ica}\) for the kurtosis-based method. (b) CDF of \(\text{Error}_{fca}/\text{Error}_{ica}\) for the entropy-based method.}
            \label{fig:fca comparison cluster}
\end{figure}

\subsection{Unmixing mixed waveforms using rectangular FCA}
\label{sec:spectrogram_embed}
Let $\bm{x}_1$ and $\bm{x}_2$ denote two vectors representing the audio signals whose waveforms are displayed in  Figures \ref{fig:x1_sw} and \ref{fig:x2_sw}, respectively. Their mixture produces signals whose waveforms are displayed in  Figures \ref{fig:z1_sw} and \ref{fig:z2_sw}, respectively. This is the famous cocktail party problem \citep{haykin2005cocktail} and ICA is known to succeed in unmixing the mixed signals. Figures \ref{fig:xhat1_ica_sw} and  \ref{fig:xhat2_ica_sw} confirm that it does. 

In this setting the mixed waveforms are modeled as
\begin{equation}\label{eq:mixing model samples}
    \begin{bmatrix}
    \bm{z}_1^T \\ \vdots \\ \bm{z}_{s}^T
    \end{bmatrix} = \bm{A} \begin{bmatrix}
    \bm{x}_1^T \\ \vdots \\ \bm{x}_{s}^T 
    \end{bmatrix}.
\end{equation}
There is no matrix in sight in (\ref{eq:mixing model samples}), so we can seemingly not apply FCA directly. 

The key insight is that we are at liberty to design a linear matrix embedding operator $\mathcal{M}: \bm x \in \mathbb{R}^{n} \mapsto \bm{X} \in \mathbb{F}^{M \times N}$ for some integer $M$ and $N$. A simple example is by reshaping the $n = MN$ vector into an $M \times N$ matrix. Here linearity implies that for any scalars $\alpha$ and $\beta$ we have that 
$$ \mathcal{M}(\alpha \bm x + \beta \bm y) = \alpha \mathcal{M}(\bm x) + \beta \mathcal{M}(\bm{y}).$$
Then as a consequence of the linearity of the embedding operator we have that 
\begin{equation}\label{eq:reshaped mixing model}
    \begin{bmatrix}
    \mathcal{M}(\bm{z}_1)^T \\ \vdots \\ \mathcal{M}(\bm{z}_s)^T
    \end{bmatrix} = (\bm{A} \otimes \bm{I}_N) \begin{bmatrix}
    \mathcal{M}(\bm{x}_1)^T  \\ \vdots \\ \mathcal{M}(\bm{x}_s)^T  
    \end{bmatrix} 
\end{equation}
so that it fits (\ref{eq:general matrix model}) and we can apply FCA to estimate the mixing matrix and thus unmix the mixed waveforms.   

For the cocktail party problem we used a (complex-valued) STFT embedding, as described in Algorithm \ref{alg:spec_fcf}, and computed the mixed (complex-valued) STFT matrices $\bm{Z}_1$ and $\bm{Z}_2$ displayed in Figure \ref{fig:spectrogram} (Here we adapt the convention that each row corresponding to the spectrum for a particular time window). Since the mixing matrix is real-valued we modified the FCA algorithms slightly by whitening using only real part of the covariance matrix. 

Figures \ref{fig:xhat1_fca_sw} and \ref{fig:xhat2_fca_sw} show that FCA successfully unmixes the complex-valued STFT matrices of the latent waveforms. Figure \ref{fig:wave1} shows that FCA succeeds in unmixing the waveforms and that FCA better unmixes the waveforms than ICA. Figure \ref{fig:wave2} illustrates a setting where ICA does better. 

These experiments illustrate our general point that FCA can be used wherever ICA has been used and that they perform comparably well. The key step is embedding a  vector waveform as a matrix in a way that preserves the mixing model. We used the STFT embedding here -- other linear embeddings could be used as well. Determining the optimal embedding so we can reason about why FCA does better than ICA for the setup in Figure \ref{fig:wave1} but does not for the setup in Figure \ref{fig:wave2} is a natural next question. 

\begin{figure}
          \centering
            \begin{subfigure}[t]{0.30\textwidth}
              \includegraphics[width = \textwidth,trim={0cm 0cm 0cm 0cm},clip]{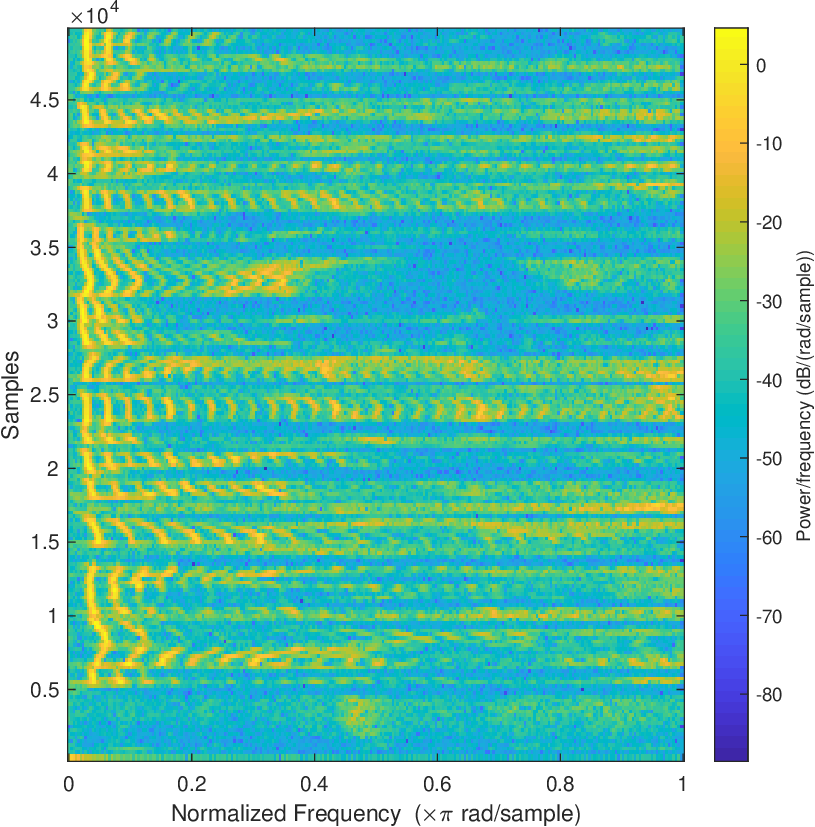}
          \caption{$\bm X_1$}
            \label{fig:x1_sw_spec}
            \end{subfigure}
            \begin{subfigure}[t]{0.30\textwidth}
              \includegraphics[width = \textwidth,trim={0cm 0cm 0cm 0cm},clip]{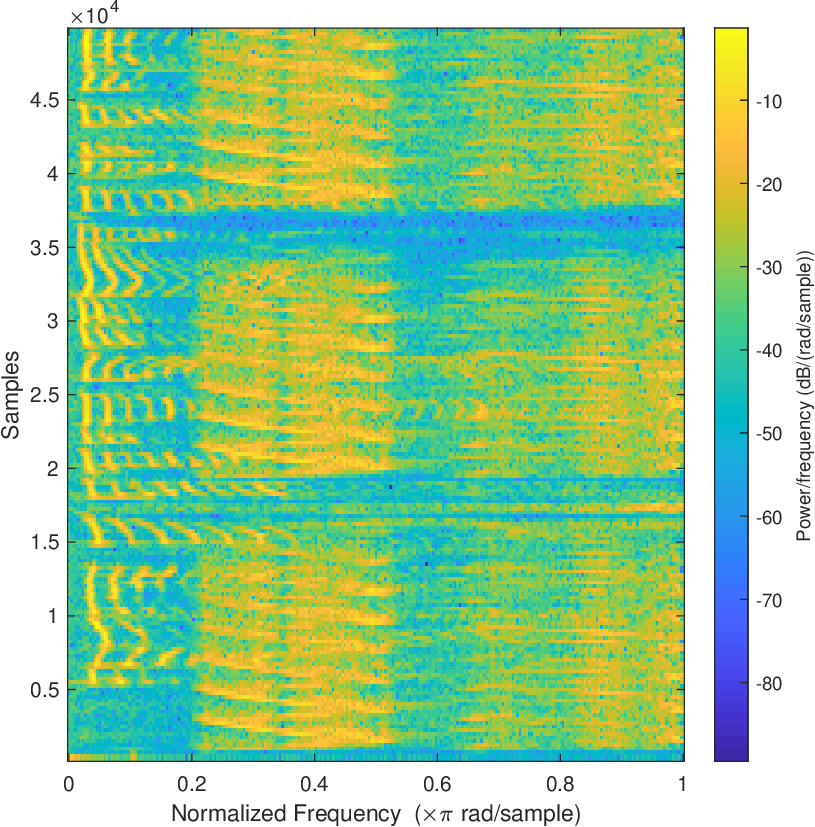}
          \caption{$\bm Z_1$}
            \label{fig:z1_sw_spec}
            \end{subfigure}  
            \begin{subfigure}[t]{0.30\textwidth}
              \includegraphics[width = \textwidth,trim={0cm 0cm 0cm 0cm},clip]{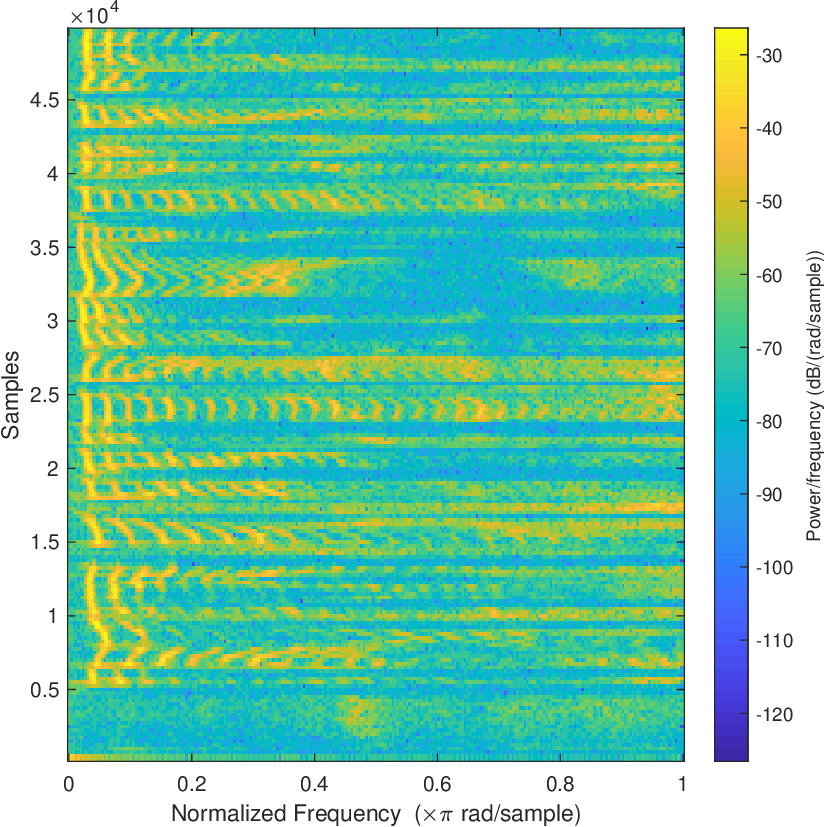}
          \caption{$\widehat{\bm X}_1$ via FCA with STFT embedding}
            \label{fig:xhat1_fca_sw_spec}
            \end{subfigure}\\
            \begin{subfigure}[t]{0.30\textwidth}
              \includegraphics[width = \textwidth,trim={0cm 0cm 0cm 0cm},clip]{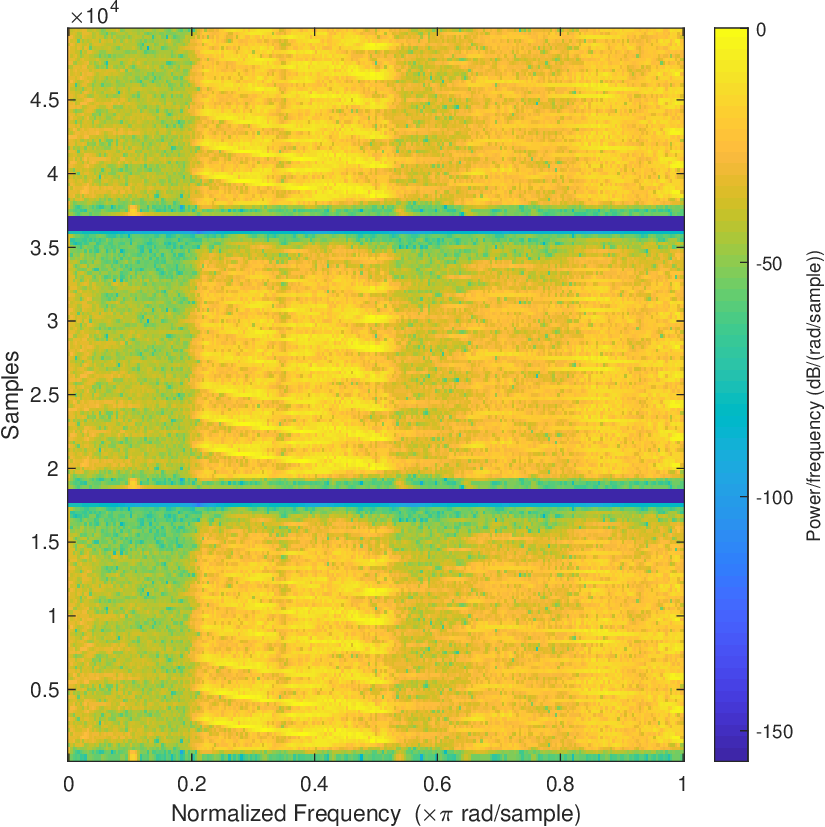}
          \caption{$\bm X_2$}
            \label{fig:x2_sw_spec}
            \end{subfigure} 
            \begin{subfigure}[t]{0.30\textwidth}
              \includegraphics[width = \textwidth,trim={0cm 0cm 0cm 0cm},clip]{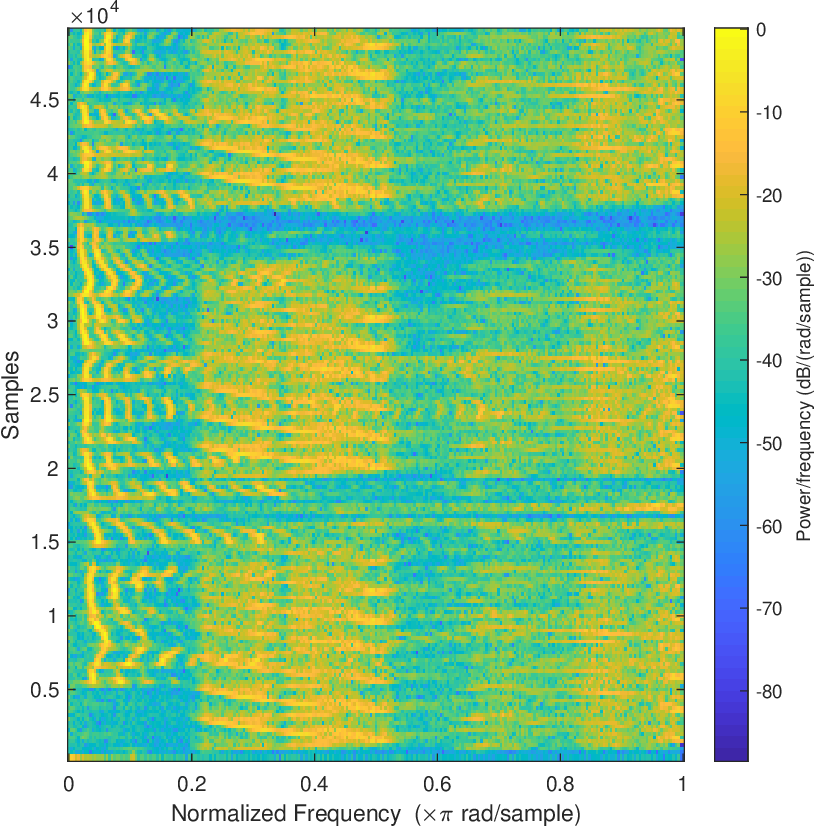}
          \caption{$\bm Z_2$}
            \label{fig:z2_sw_spec}
            \end{subfigure}
            \begin{subfigure}[t]{0.30\textwidth}
              \includegraphics[width = \textwidth,trim={0cm 0cm 0cm 0cm},clip]{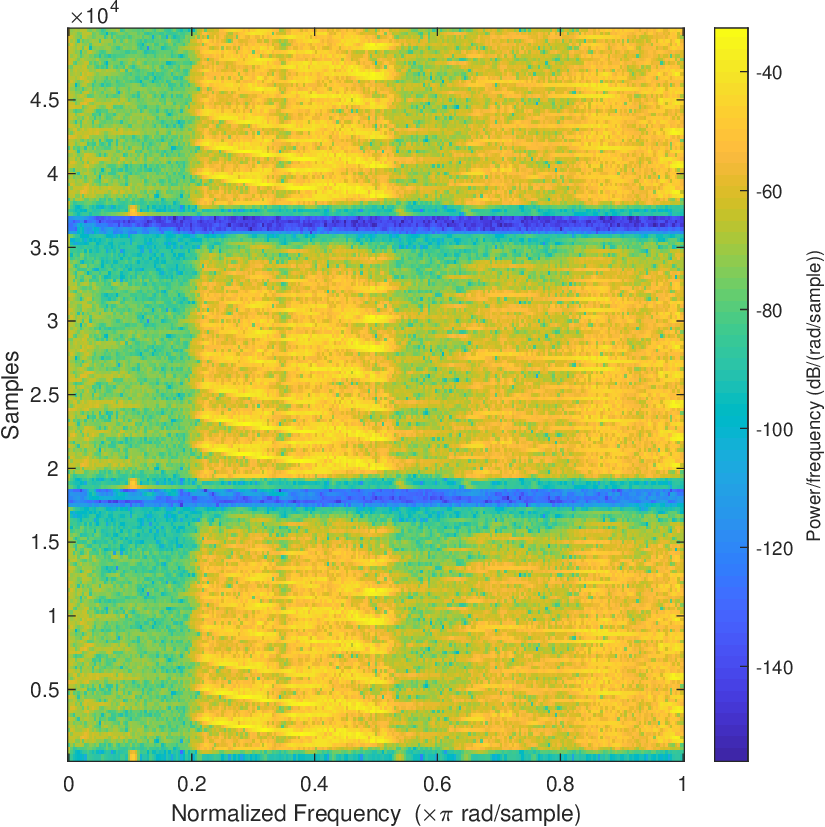}
          \caption{$\widehat{\bm X}_2$ via FCA with STFT embedding}
            \label{fig:xhat2_fca_sw_spec}
            \end{subfigure}
            \caption{The spectrogram (the magnitude of STFT matrices) of signals. Note that spectrogram is for demonstration and FCA is applied to STFT matrices directly. For the STFT, we adapt the Hamming window of $250$ samples, the number of overlapped samples is $125$, the number of DFT points is $256$.}
           \label{fig:spectrogram}
            \end{figure}
            
\subsection{Unmixing  rectangular matrices using self-adjoint FCA and more}

We can take this embedding idea even further by embedding mixed rectangular matrices modeled as  (\ref{eq:mixing model samples}) and embedding them as self-adjoint matrices as described in Algorithm  \ref{alg:embed_rec_fcf} and then using self-adjoint FCA to unmix them. Or, we may even embed rectangular matrices into another rectangular matrix with a different number of rows and columns as described in Algorithm \ref{alg:embed_rec_fcf}. Determining the right matricial embedding  adds another aspect to the question of optimal embedding selection as in Figure \ref{fig:diagmram}.


\section{Conclusions and Open Problems}
\label{sec:conclusions}
\begin{figure}[t]
  \centering
    \begin{subfigure}[t]{0.24\textwidth}
      \includegraphics[width = \textwidth,trim={0cm 0cm 0cm 0cm},clip]{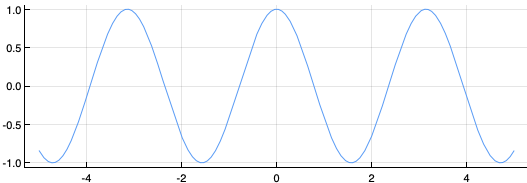}
  \caption{Cosine wave}
    \label{fig:wave2x1}
    \end{subfigure}
    \begin{subfigure}[t]{0.24\textwidth}
      \includegraphics[width = \textwidth,trim={0cm 0cm 0cm 0cm},clip]{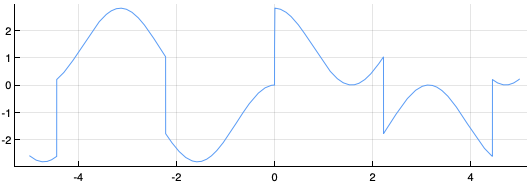}
  \caption{Mixed wave1}
    \label{fig:wave2z1}
    \end{subfigure}  
    \begin{subfigure}[t]{0.24\textwidth}
      \includegraphics[width = \textwidth,trim={0cm 0cm 0cm 0cm},clip]{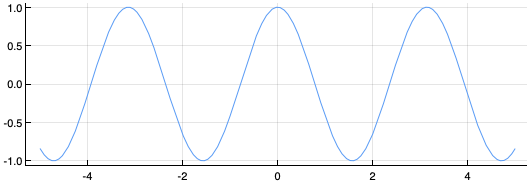}
  \caption{Wave1 via ICF}
    \label{fig:wave2ica1}
    \end{subfigure}
    \begin{subfigure}[t]{0.24\textwidth}
      \includegraphics[width = \textwidth,trim={0cm 0cm 0cm 0cm},clip]{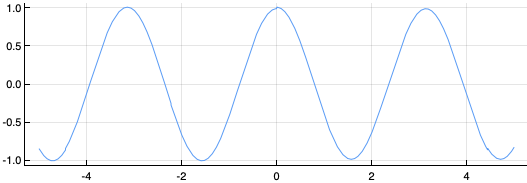}
  \caption{Wave1 via FCA}
    \label{fig:wave2fca1}
    \end{subfigure}
    \begin{subfigure}[t]{0.24\textwidth}
      \includegraphics[width = \textwidth,trim={0cm 0cm 0cm 0cm},clip]{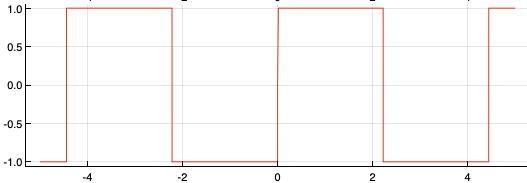}
  \caption{Square wave}
    \label{fig:wave2x2}
    \end{subfigure} 
    \begin{subfigure}[t]{0.24\textwidth}
      \includegraphics[width = \textwidth,trim={0cm 0cm 0cm 0cm},clip]{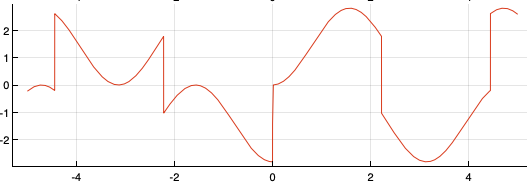}
  \caption{Mixed wave2}
    \label{fig:wave2z2}
    \end{subfigure}
    \begin{subfigure}[t]{0.24\textwidth}
      \includegraphics[width = \textwidth,trim={0cm 0cm 0cm 0cm},clip]{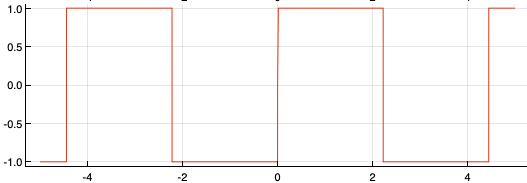}
  \caption{Wave2 via ICA}
    \label{fig:wave2ica2}
    \end{subfigure}
    \begin{subfigure}[t]{0.24\textwidth}
      \includegraphics[width = \textwidth,trim={0cm 0cm 0cm 0cm},clip]{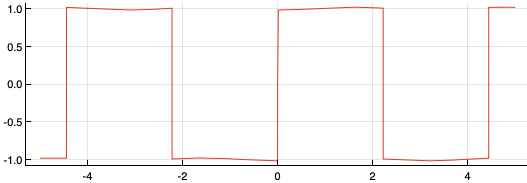}
  \caption{Wave2 via FCA}
    \label{fig:wave2fca2}
    \end{subfigure}
    \caption{
    An experiment in waveform separation using ICA and FCA. Note that subplots (c), (g) (unmixed waves via ICA) and (d), (h) (unmixed waves via FCA) both recover (a), (e). In this simulation, $\bm{A} =[\sqrt{2}, \sqrt{2}; -\sqrt{2}, \sqrt{2}]/2$ in (\ref{eq:general matrix model}). 
    The errors for ICA and FCA are $7.70\times 10^{-5}$ and $1.36\times 10^{-2}$ respectively.}
\label{fig:wave2}
\end{figure}

We have developed free component analysis as a non-commutative analog of independent component analysis. We proved that when certain identifiability conditions are met then mixtures of  self-adjoint and rectangular variants can be unmixed using self-adjoint and non-self-adjoint/rectangular FCA.  We developed an algorithm for umixing mixtures of  matrices based on FCA and demonstrated how FCA can be used to unmix images (viewed as matrices), speech signals and waveforms (when embedded as STFT matrices) and images where it initially fails (via FCA on a free subset of the mixed images). 

\subsection{Open Problems}

We now list some directions for future research. These include developing a non-linear extension of FCA analogous to non-linear ICA  \citep{eriksson2002blind, brakel2017learning, hyvarinen2017nonlinear, hyvarinen2016unsupervised, almeida2003misep, hyvarinen2018nonlinear}, a fast algorithm for  FCA analogous to fast ICA \citep{hyvarinen1999fast,chen2006efficient, oja2006fastica} and algorithms for sparse FCA analogous to sparse ICA \citep{comon2010handbook, bofill2001underdetermined}.

A more general line of inquiry is related to the so-called one-unit learning  work in ICA. In ICA, it is known that instead of maximizing the kurtosis, we can equivalently maximize a large  class of so-called contrast functions $G(\cdot)$ of the form \citep[Equation (2)]{hyvarinen1997one}
$$J_{G}(\bm{w}) = \mathbb{E}_x[G(\bm{w}^T x)] - \mathbb{E}{\gamma}[G(\gamma)],$$
where $G(\cdot)$ is non-quadratic well-behaving even function and $\gamma$ is a standardized Gaussian random variable. Developing the analog of this theory for the self-adjoint and rectangular FCA settings will allow for a finer study of the statistical efficiency of the FCA algorithms in the finite matrix setting akin to the work by \cite{arora2012provable} and facilitate the development of asymptotically consistent and statistically efficient estimators akin to the work by  \cite{chen2006efficient}. 

Our simulations showed that free entropy based FCA outperformed free kurtosis based FCA (see Figure \ref{fig::cdf_entropy}). Computing the free entropy is computationally more expensive than computing the free kurtosis. In ICA, the mutual information is approximated via a cumulant expansion \citep[Section 3.1, pp. 295]{comon_1994}. Developing a rapidly converging approximation to free entropy in terms of the free cumulants that converges faster than the approximation in \cite[Exercise 5, pp. 190]{mingo2017free} would lead a faster FCA that we expect to be statistically more efficient than free kurtosis based FCA.

In Section \ref{sec:spectrogram_embed}, we discuss the application of FCA to STFT matrices of signals. The usage of short time fourier transform matrix there is for a matricial representation of vector signals that fits FCA model. We assume the STFT matrices of sources signals are freely independent and we are still in the linear mixture regime where we have the number of STFT matrices from observed signals matches number of sources. We note that in the context of ICA, STFT, as a projection from single-channel into multi-channel, is used in single channel separation \cite{davies2007source}. 
Authors of \cite{casey2000separation, gao2003blind, barry2005single, mika2020single} apply ICA to rows of spectrogram (the magnitude of STFT matrices) assuming the statistical independence of spectral basis vectors spanning source signals. 
One can also assume the statistical independence of time variation of source signals and apply ICA to columns of STFT matrices as in \cite{gao2003blind}. 
It is natural and interesting to develop a single channel FCA. 
The main blocker at the moment is we are not aware of an analogy that projects a single matrix to multiple matrices in a meaningful way.

\subsection{Open Problem: Using FCA to construct new matrix models for freeness}

We can use the ICA  to decompose small patches of an image into linear independent combinations of ICA basis vectors that can be learned from the data via the ICA factorization \citep{hoyer2000independent, bell1997independent}. Figure \ref{fig:icapatch} displays the $36$ ICA bases patches thus obtained by reshaping into $6 \times 6$ matrices the $36$ ICA bases vectors corresponding to each of the columns of the $36 \times 36$ $\bm{W}_{\sf ica}$ matrix obtained by an kurtosis based ICA factorization of the $6 \times 6$ patches of the panda image in Figure \ref{fig:panda}. 

We can similarly use FCA to decompose the  patches of an image into ``as free as possible'' matrices. Figure \ref{fig:fcapatch} shows the $36$ free patch bases obtained by displaying the matricial elements of the $\bm{X}_{\sf fca}$ array of matrices of the panda image. Each sub-image in the panda is a linear combination of these free patches.

 \begin{figure}[ht]
    \centering
    \begin{subfigure}[t]{0.45\textwidth}
    \includegraphics[width = \textwidth,trim={0cm 0cm 0cm 0cm},clip, angle =180]{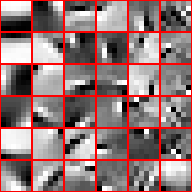}
    \caption{Patch ICA on panda patches.}
    \label{fig:icapatch}
    \end{subfigure}
    \begin{subfigure}[t]{0.45\textwidth}
    \includegraphics[width = \textwidth,trim={0cm 0cm 0cm 0cm},clip, angle =180]{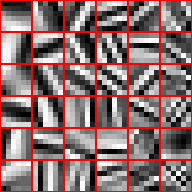}
    \caption{Patch ICA on hedgehog patches.}
    \label{fig:icapatchhedge}
    \end{subfigure}
    \begin{subfigure}[t]{0.45\textwidth}
   \includegraphics[width = \textwidth,trim={0cm 0cm 0cm 0cm},clip]{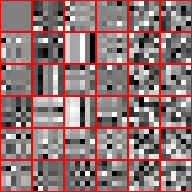}
   \caption{Patch FCA on panda patches.}
     \label{fig:fcapatch}
    \end{subfigure}
    \begin{subfigure}[t]{0.45\textwidth}
    \includegraphics[width = \textwidth,trim={0cm 0cm 0cm 0cm},clip]{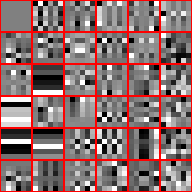}
    \caption{Patch FCA on hedgehog patches.}
    \label{fig:fcapatchhedge}
    \end{subfigure}
    \caption{Patch bases obtained via an ICA (top row) or FCA (bottom row) factorization of $6 \times 6$ patches of the panda  and the hedgehog images from Figures \ref{fig:panda} and \ref{fig:hedgehog} respectively. }
    \label{fig:patch}
\end{figure}

The patch FCA versus patch ICA bases vectors for  the hedgehog image in Figure \ref{fig:hedgehog} are shown in Figures \ref{fig:fcapatchhedge} and \ref{fig:icapatchhedge}.  Comparing the ICA bases vectors in Figure \ref{fig:icapatch} to the FCA patch  bases reveals that the ICA bases contain diagonal elements whereas the FCA bases are more checkerboard like and are even reminiscent of the 2D- DCT matrix. The ICA patch bases seem to depend on the image much more strongly than the FCA patch bases. Both the FCA and ICA patch bases are more structured than we might have expected.  

Since FCA worked in unmixing the panda and hedgehog images and since each of these images is composed of a linear combination of FCA extracted free patches, this suggests a way of constructing  not-so-random matrix models from random linear combinations of not-so-random (sub) matrices that are asymptotically free. This line of inquiry would complement the recent work by 
\citep{anderson2014asymptotically,cebron2016universal,male2011traffic} in developing not-as-random matrix models that are asymptotically free.  

FCA can serve as a valuable computational tool for reasoning and formulating mathematically plausible conjectures about matricial freeness in not-so-random matrices.

\subsection{Open Problem: Improving FCA by ``more free'' sub-matrix selection}

FCA (and ICA) do not always succeed in unmixing images. See, for example Figure \ref{fig:jennifer_brad} where applying FCA to the mixed images does not produce a good estimate of the mixing matrix. In Figure \ref{fig:jennifer_brad_sub}, we show how we can better estimate the mixing matrix by applying FCA to sub-matrices instead. In this example,  we can reason that FCA on the whole matrix fails because the in-alignment faces make the matrices ``less free'' whereas the sub-matrices are ``more free''. 

We can formalize this idea further by examining how random or not-random the left and right singular vectors of the matrices are. Asymptotically free matrices have left and right singular vectors that are isotropically random relative to each other. Hence, if $V_1$ and  $V_2$ are $N \times N$ right singular vectors of two matrices and if $V_1$ and $V_2$ are independent and isotropically random, then  we expect the entries of $V_1^TV_2$ to be delocalized and having the values of order   $N^{-1/2}$. We can employ a similar argument for the left singular vectors. 

We can use this as a heuristic for quantifying how close-to-free two matrices we are trying to unmix are. 

For the panda and hedgehog images in Figure \ref{fig:panhogmix}, we can see from Figures \ref{fig:uprodpan} and \ref{fig:vprodpan} that the right and left singular vectors respectively are more uniform and so we might FCA to succeed as it indeed did.

In contrast, for the matrices in the Figure \ref{fig:jennifer_brad}, the right and left singular vectors of the matrices in Figures \ref{fig:uprodjenn} and \ref{fig:vprodjenn} respectively are not that uniform and so we might expect FCA to fail as it did.

The sub-matrices on which we applied FCA in Figure \ref{fig:jennifer_brad_sub} are ``more free'' than the matrices in Figure \ref{fig:jennifer_brad} and so FCA worked better in the former case than in the latter. ICA similarly fails as FCA when applied to the whole matrices and similarly succeeds when applied to the sub-matrices. 

New algorithmic methods for identifying ``more (freely) independent'' sub-matrices to improve the unmixing performance of FCA (and ICA) would be invaluable in applications where practitioners have applied FCA (or ICA) and given up because it seemingly did  not succeed. Such methods would help make FCA, and ICA, (even) great(er) (again)! 

\begin{figure}[t]
  \centering
    \begin{subfigure}[t]{0.3\textwidth}
      \includegraphics[width = \textwidth,trim={0cm 0cm 0cm 0cm},clip]{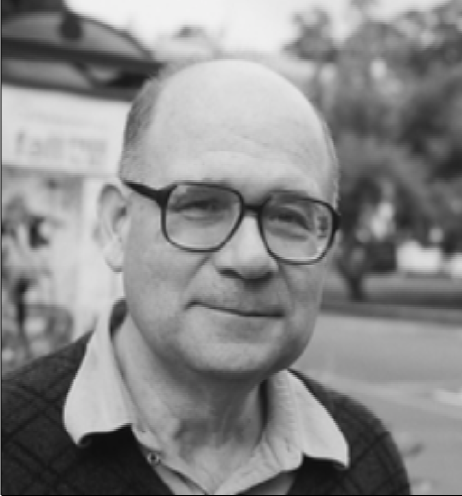}
  \caption{$X_1$}
    \label{fig:jbx1}
    \end{subfigure}
    \begin{subfigure}[t]{0.3\textwidth}
      \includegraphics[width = \textwidth,trim={0cm 0cm 0cm 0cm},clip]{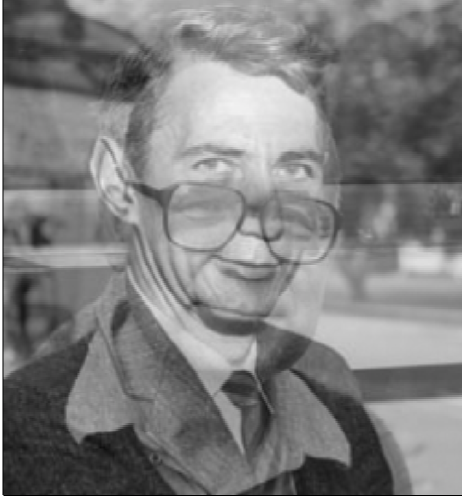}
  \caption{$Z_1$}
    \label{fig:jbz1}
    \end{subfigure}  
    \begin{subfigure}[t]{0.3\textwidth}
      \includegraphics[width = \textwidth,trim={0cm 0cm 0cm 0cm},clip]{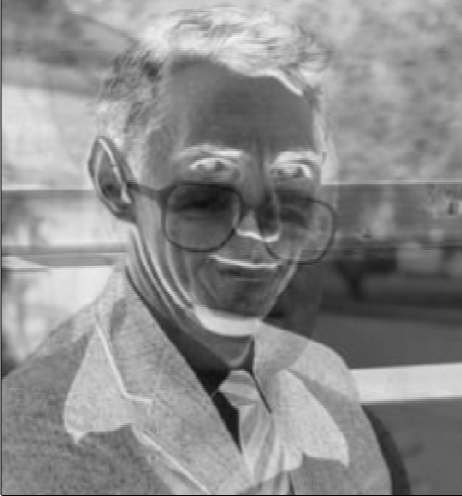}
  \caption{$\widehat X_1$ via FCA}
    \label{fig:jbfca1}
    \end{subfigure}
    \begin{subfigure}[t]{0.3\textwidth}
      \includegraphics[width = \textwidth,trim={0cm 0cm 0cm 0cm},clip]{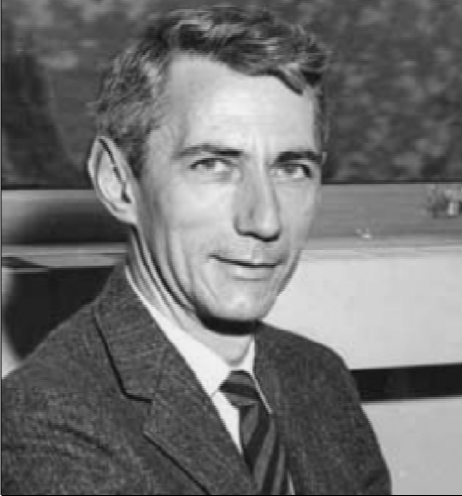}
  \caption{$X_2$}
    \label{fig:jbx2}
    \end{subfigure} 
    \begin{subfigure}[t]{0.3\textwidth}
      \includegraphics[width = \textwidth,trim={0cm 0cm 0cm 0cm},clip]{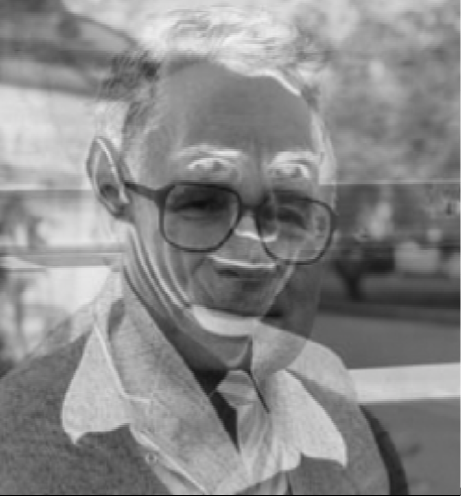}
  \caption{$Z_2$}
    \label{fig:jbz2}
    \end{subfigure}
    \begin{subfigure}[t]{0.3\textwidth}
      \includegraphics[width = \textwidth,trim={0cm 0cm 0cm 0cm},clip]{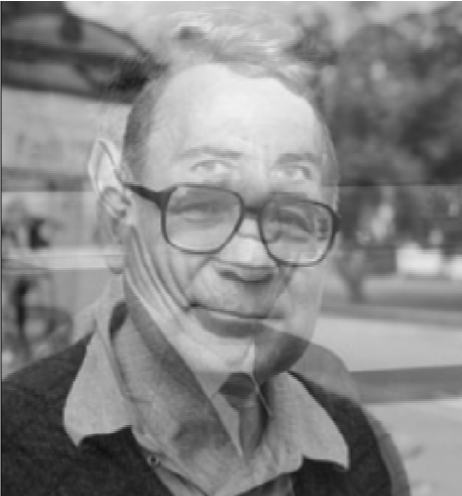}
  \caption{$\widehat X_2$ via FCA}
    \label{fig:jbfca2}
    \end{subfigure}
    \caption{An Application of FCA to images not so free.  The mixing matrix is $\bm{A} =[\sqrt{2}, \sqrt{2}; -\sqrt{2}, \sqrt{2}]/2$}
\label{fig:jennifer_brad}
\end{figure}

\begin{figure}[t]
  \centering
    \begin{subfigure}[t]{0.3\textwidth}
      \includegraphics[width = \textwidth,trim={0cm 0cm 0cm 0cm},clip]{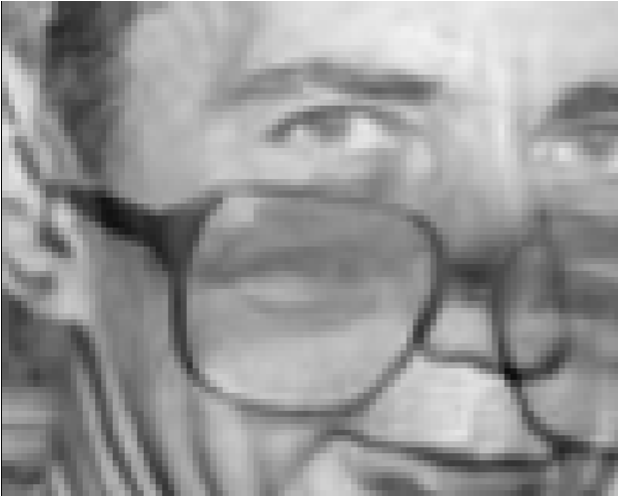}
  \caption{Sub matrix of $Z_1$}
    \label{fig:jbmixsub1}
    \end{subfigure}
    \begin{subfigure}[t]{0.3\textwidth}
      \includegraphics[width = \textwidth,trim={0cm 0cm 0cm 0cm},clip]{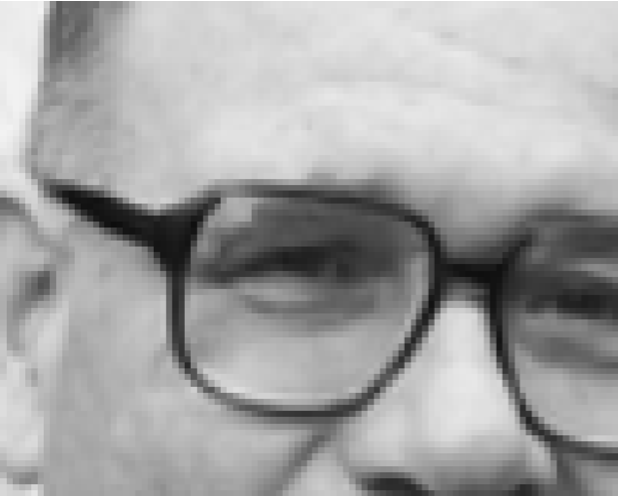}
  \caption{$\widehat X_1$ by FCF}
    \label{fig:jbfcasub1}
    \end{subfigure}  
    \begin{subfigure}[t]{0.3\textwidth}
      \includegraphics[width = \textwidth,trim={0cm 0cm 0cm 0cm},clip]{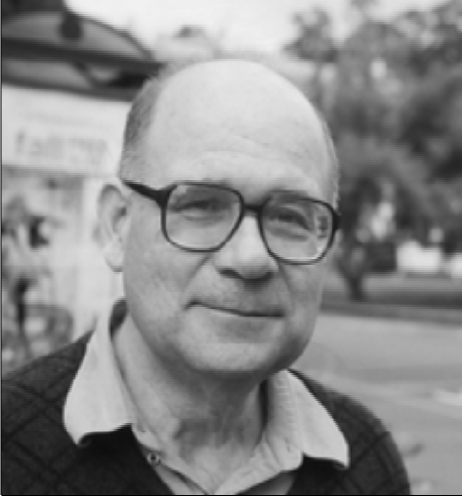}
  \caption{Full unmixed image 1}
    \label{fig:jbfcafull1}
    \end{subfigure}
    \begin{subfigure}[t]{0.3\textwidth}
      \includegraphics[width = \textwidth,trim={0cm 0cm 0cm 0cm},clip]{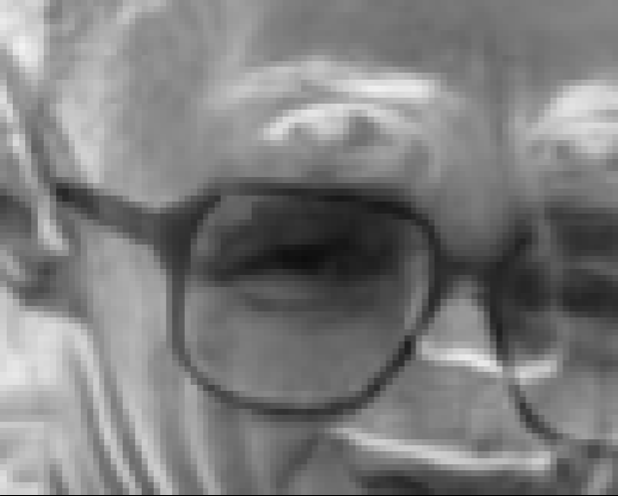}
  \caption{Sub matrix of $Z_2$}
    \label{fig:jbmixsub2}
    \end{subfigure} 
    \begin{subfigure}[t]{0.3\textwidth}
      \includegraphics[width = \textwidth,trim={0cm 0cm 0cm 0cm},clip]{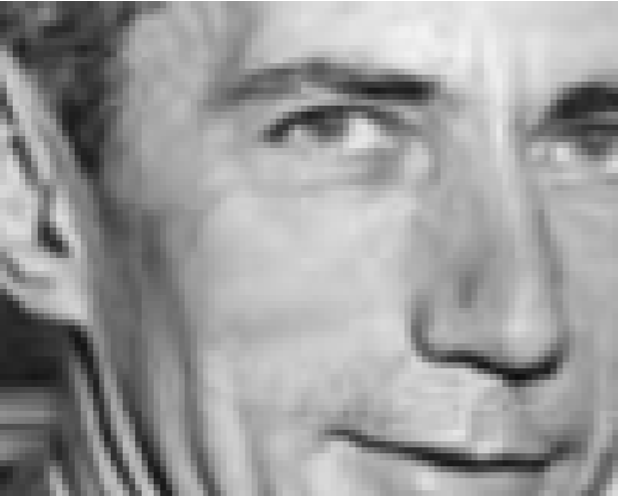}
  \caption{$\widehat X_2$ by FCF}
    \label{fig:jbfcasub2}
    \end{subfigure}
    \begin{subfigure}[t]{0.3\textwidth}
      \includegraphics[width = \textwidth,trim={0cm 0cm 0cm 0cm},clip]{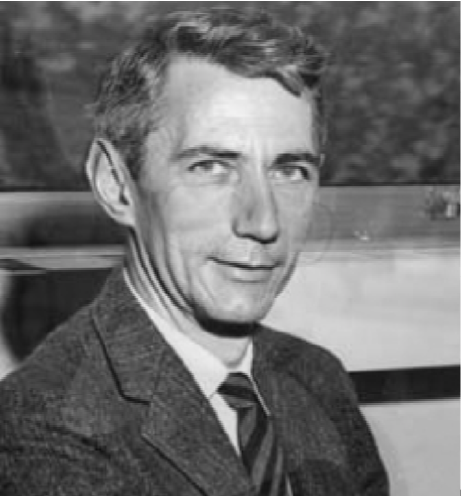}
  \caption{Full unmixed image 2}
    \label{fig:jbfcafull2}
    \end{subfigure}
    \caption{Application of FCA to sub images gives better results. The mixing matrix is $\bm{A} =[\sqrt{2}, \sqrt{2}; -\sqrt{2}, \sqrt{2}]/2$}
\label{fig:jennifer_brad_sub}
\end{figure}

\begin{figure}[t]
  \centering
    \begin{subfigure}[t]{0.4\textwidth}
      \includegraphics[width = \textwidth,trim={0cm 0cm 0cm 0cm},clip]{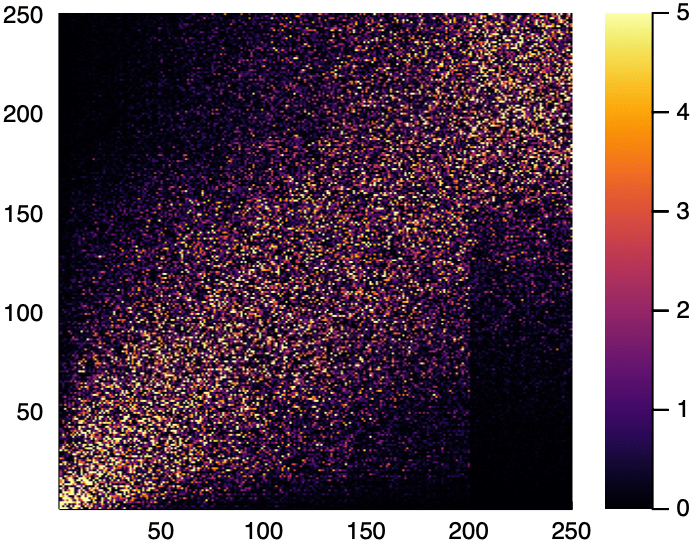}
  \caption{}
    \label{fig:uprodjenn}
    \end{subfigure}
    \begin{subfigure}[t]{0.4\textwidth}
      \includegraphics[width = \textwidth,trim={0cm 0cm 0cm 0cm},clip]{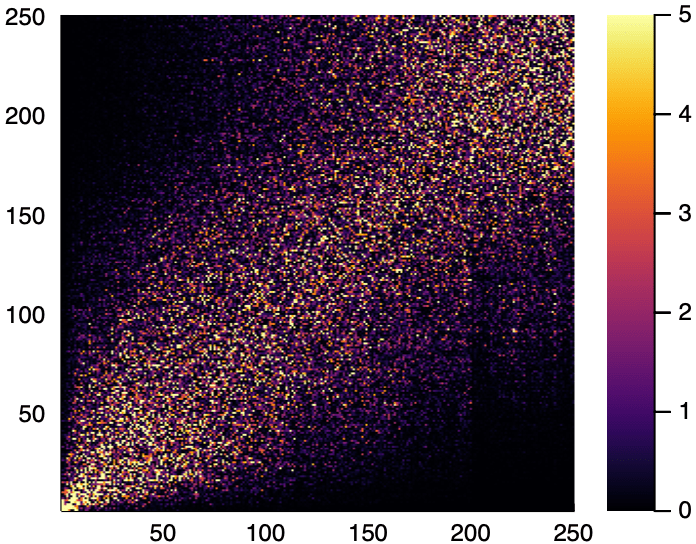}
  \caption{}
    \label{fig:vprodjenn}
    \end{subfigure}  
    \begin{subfigure}[t]{0.4\textwidth}
      \includegraphics[width = \textwidth,trim={0cm 0cm 0cm 0cm},clip]{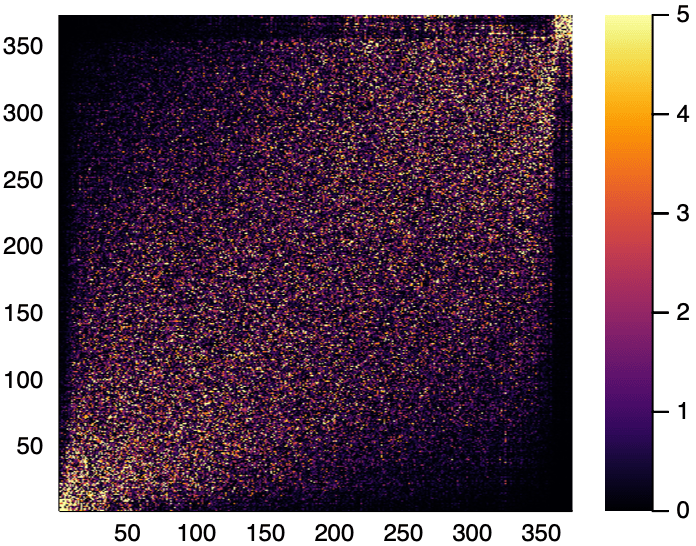}
  \caption{}
    \label{fig:uprodpan}
    \end{subfigure} 
    \begin{subfigure}[t]{0.4\textwidth}
      \includegraphics[width = \textwidth,trim={0cm 0cm 0cm 0cm},clip]{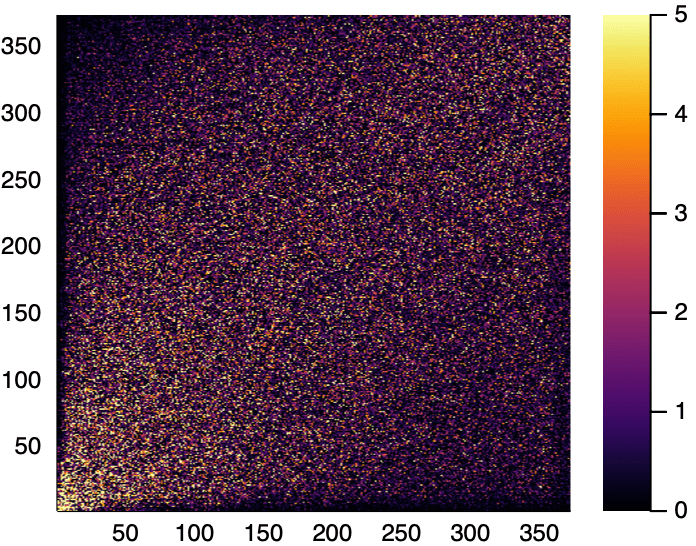}
  \caption{}
    \label{fig:vprodpan}
    \end{subfigure}
    \caption{Normalized square of inner products between: (a) left singular vectors of Figure \ref{fig:jbx1} and \ref{fig:jbx2}, (b) right singular vectors of Figure \ref{fig:jbx2} and \ref{fig:jbx1}, (c) left singular vectors of Figure \ref{fig:panda} and \ref{fig:hedgehog}, (d) right singular vectors of Figure \ref{fig:panda} and \ref{fig:hedgehog}. We observe that inner product between the  right singular vectors of Figure \ref{fig:jbx2} and \ref{fig:jbx1} (corresponding to (b)) are clearly not uniform.}
\label{fig:singvec_prod}
\end{figure}

\subsection{Open Problem: Convergence guarantee}
\label{sec:op_conv}
As discussed in Remark \ref{rmk:nonconvex}, the optimization problems \eqref{eq:opt_kurt} and \eqref{eq:kurt3} are non-concave. 
Also, in practice, we are subject to noise, finite sample and finite dimension (recall that the random matrices are asymptotically free). 
Therefore, while FCF seems to work empirically in the numerical simulations, a further theoretical guarantee of the convergence to the global maximum need to be studied. 
Along this line, Arora, Ge, et al give a provable convergence guarantees for ICA \cite{arora2012provable}. 
The setup considered there is analogous to Theorem \ref{thm:kurt1} (recover one column orthogonal matrix at a time). 
It is natural to ask whether one can establish similar result for FCA. 
Furthermore, it is interesting to develop optimization algorithm with provable convergence guarantees for setup considered by Theorem \ref{thm:kurt3} and \ref{thm:ent}.  

\clearpage
\section*{Acknowledgements}

We thank Peter Bickel for inspiring us to revisit FCA via a serendipitous meeting at the Santa Fe Institute in December 2015. That meeting, and his remarks on ICA and all the ways in which it is natural,  provided the spark for us spending the rest of that workshop and the following month thinking about all the ways that FCA was natural for random matrices and images. We implemented our first FCA algorithm soon thereafter and leaned into the theory after getting, and being overjoyed by, the image separation results in Figure \ref{fig::fca_recover}.

We thank Arvind Prasadan for his detailed comments and suggestions on earlier versions of this manuscript. We are grateful to Alfred Hero for his suggestion to try the denoising simulation in Figure \ref{fig::locust} which brought into sharp focus for us for the first time that FCA could do (much) better than ICA (this was a simulation we had been avoiding till because we feared the opposite!). This work has benefited from Roland Speicher's many insightful comments and suggestions and from Octavio Arizmendi Echegaray's remarks that made us better understand the underlying free probabilistic structures that made some of the FCA identifiability related questions fundamentally different than their ICA counterparts.

This research was supported by ONR grant N00014-15-1-2141, DARPA Young Faculty Award D14AP00086 and  ARO MURI W911NF-11-1-039.

A \texttt{Julia}  implementation of the FCA algorithm as well as code to reproduce the  simulations and figures in this paper, is available at Github \citep{fca}.

\newpage

\appendix
\addcontentsline{toc}{section}{Appendices}
\clearpage

\section{What is freeness of random variables?}
\label{sec:math_prelim}

The goal of this section is to introduce the freeness of non-commutative random variable. 
We first discuss independence (freeness) in the context of the scalar probability, free probability for self-adjoint (non-commutative) random variables and free probability for rectangular (non-commutative) random variables respectively. 
We focus on the the behavior of (free) cumulants and (free) entropy of independent (free) random variables, which are the basis ICA (FCA). 
The connection between in independent random matrices and free random variables is given at the end.

For a detailed introduction of free probability, readers are referred to \citep{roland_2006, free_entropy_HP, mingo2017free}
\subsection{Prologue: What is independence of commuting random variables?}  
\label{subsec:scalarindependence}
Here, we briefly review statistical independence in scalar probability. We state the behavior of cumulants and entropy of independent random variables, which are the basis of ICA. 
In the end, we discuss the unique role the Gaussian random variables play in ICA.

\subsubsection{Mixed moments point of view}

Let $I$ denotes an index set, and $(\alge_i)_{i\in I}$ denote random variables.
    They are independent if for any $n \in \N$ and $m_1, \cdots, m_n \geq 0$, 
    $$
      \E[\alge_{i(1)}^{m_1} \cdots \alge_{i(n)}^{m_n}] = \E[\alge_{i(1)}^{m_1}] \cdots \E[\alge_{i(n)}^{m_n}].
    $$
    if $i(j) \in I$, $j = 1,\cdots n$ are all distinct.
    An alternative definition is that for any polynomials $P_1,\cdots P_n$ of one variables,
    \begin{equation}  \label{eq:def scalar indpendence}
      \E[P_1(\alge_{i(1)})\cdots P_n(\alge_{i(n)})] = 0
    \end{equation}
    if $\E[P_j(\alge_{i(j)})] = 0$ for all $j = 1, \cdots, n$ and $i(j) \in I$, $j = 1,\cdots n$ are all distinct.

\subsubsection{Cumulants -- kurtosis and higher order -- independent additivity}

The (joint) cumulants of $n$ random variables $a_1, \cdots, a_k$ is defined by
    \begin{equation}    \label{eq:def scalar cumu}
        \icacumu_n(a_1, \cdots, a_n) = \sum_{\pi} (\abs*{\pi} - 1)! (-1)^{\abs*{\pi} - 1} \prod_{B \in \pi}\E \left[\prod_{i \in B} a_i\right],
    \end{equation}
    where $\pi$ runs through all partitions of $\{1, \cdots, m\}$, $B$ runs through all blocks of partition $\pi$.
    Equivalently, $\{\icacumu_m\}_{m \geq 1}$ is defined through
    \begin{equation} \label{eq:def scalar cumu 2}
      E(x_1 \cdots x_n) = \sum_{\pi} \prod_{B \in \pi} \icacumu_{\abs*{B}} (a_i:i\in B)
    \end{equation}
    
    The reason that ICA adapts an optimization problem involving cumulants is the following property:
    if $(x_i)_{i \in I}$ are independent, then for any $n\in \N$
    \begin{equation}
    \label{eq:cumu_ind_scalar}
      \icacumu_n(x_{i(1)}, \cdots, x_{i(n)}) = 0
    \end{equation}
    whenever there exists $1\leq \ell, k\leq n$ with $i(\ell) \neq i(k)$. 
    That is, any cumulants involving two (or more) independent random variables is zero.
    Adapt the notation
    $$
      \icacumu_n(x) := c_n(x, \cdots, x).    $$
    A quick consequence of \eqref{eq:cumu_ind_scalar} is that for independent $x_1$ and $x_2$.
    \begin{equation}    \label{eq:scalar cumu ind additivity}
      \icacumu_n(x_1 + x_2) = \icacumu_n(x_1) + \icacumu_n(x_2).
    \end{equation}

\subsubsection{Entropy -- independent additivity}
For random variables $x_1,\cdots, x_n$ with joint distribution $f(x_1,\cdots,x_n)$, the (joint) entropy is defined by \citep{cover2012elements}
    \begin{equation}    \label{eq:scalar ent}
      h(x_1,\cdots,x_n) = -\int f(\alpha_1,\dots,\alpha_n) \log f(\alpha_1,\cdots,\alpha_n) \mathrm{d} \alpha_1\cdots \mathrm{d} \alpha_n.
    \end{equation}
    
    The joint entropy of a set of variables is less than or equal to the sum of the individual entropies of the variables in the set,
    \begin{equation}
    \label{eq:ent_ind_scalar}
      h(x_1,\cdots,x_n) \leq h(x_1) + \cdots + h(x_n).
    \end{equation}
    In particular, the equality in \eqref{eq:ent_ind_scalar} holds if and only if $x_1,\cdots, x_n$ are independent.
    Therefore, entropy is regarded as a measure of independence and thus can be used in ICA.
    
    We also want to recall another handful property of entropy. 
    For random vectors $x, y$ satisfying linear relation $y = \bm A x$, we have that 
    \begin{equation}
    \label{eq:ent_linear_transformation}
    h(y_1, \cdots, y_n) = h(x_1, \cdots, x_n) + \log \abs*{\mathrm{det}\bm A}. 
    \end{equation}
    In particular, the entropy is invariant under orthogonal linear transformation.  

\subsubsection{Why Gaussians cannot be unmixed: Gaussians have zero higher order cumulants}

In ICA, the optimization problem people used finds the independent direction by maximizing the kurtosis (fourth cumulant). 
    However, all cumulants of order larger than $2$ for Gaussian random variables vanish.
    Thus ICA fails to unmix Gaussian random variables.
ICA based on the entropy also fails to unmix Gaussian random variables, as nontrivial mixtures of independent Gaussian random variables can still be independent Gaussian.
On the other hand, it was shown that this is the only case where ICA does not work \citep{comon_1994}. 
A result of this kind is called an identifiability condition.

\subsection{Freeness of self-adjoint random variables}  \label{subsec:prelim self adjoint}
We first introduce the definition of probability space for non-commutative random variables.
The starting point is the an unital algebra of non-commutative variables.
\begin{definition}
Let $\alg$ be a vector space over $\C$ equipped with product $\cdot : \alg \times \alg \mapsto \alg$. Denote the vector space addition by $+$, we call $\alg$ an algebra if for all $a, b, c\in \alg$ and $\alpha \in \C$,
\begin{enumerate}
  \item[(a)] $a(bc) = (ab)c$,
  \item[(b)] $a(b + c) = ab + ac$,
  \item[(c)] $\alpha (ab) = (\alpha a)b = a(\alpha b)$.
\end{enumerate}
We call $\alg$ a unital algebra if there is a unital element $1_\alg$ such that, for all $a \in \alg$
\begin{equation}
  a = a 1_\alg = 1_\alg a.
\end{equation}
An algebra $\alg$ is called a $*$-algebra if it is also endowed with an antilinear $*$-operation $\alg \ni a \mapsto a^* \in \alg$, such that $(\alpha a)^* = \bar \alpha a^*$, $(a^*)^* = a$ and $(ab)^* = b^*a^*$ for all $\alpha \in \C$, $a, b \in \alg$. 
\end{definition}
Note that $ab = ba$ does not necessarily hold for general $a,b \in \alg$, i.e., they are non-commutative. 

\begin{definition}
\label{def:fps}
A (non-commutative) $*$-probability space $(\alg, \fnl)$ consists of a unital $*$-algebra and a linear functional $\fnl: \alg \rightarrow \C$, which serves as the `expectation'. 
We also require that $\fnl$ satisfies 
\begin{itemize}
    \item[(a)] (positive) $\fnl(aa^*) \geq 0$ for all $a \in \alg$. 
    \item[(b)] (tracial) $\fnl(ab) = \fnl(ba)$ for all $a, b \in \alg$.
    \item[(c)]  $\fnl(1_\alg) = 1$. 
\end{itemize}
The elements $a \in \alg$ are called non-commutative random variables. (We may omit the word non-commutative if there is no ambiguity.)
Given a series of random variables $\alge_1, \cdots, \alge_k \in \alg$, for any choice of $n \in \N$, $i(1),\cdots,i(n) \in [1..k]$ and $\epsilon_1, \cdots, \epsilon_n \in \{1, *\}$, $\fnl(\alge_{i(1)}^{\epsilon_1}\cdots \alge^{\epsilon_n}_{i(n)})$ is a mixed moment of $\{\alge_i\}_{i = 1}^k$.
The collection of all moments is called the joint distribution of $\alge_1,\cdots, \alge_k$.
\end{definition}

The moments of general random variables can be complex-valued; self-adjoint random variables, which are defined below, necessarily have real-valued moments and will be the object of our study.

\begin{definition}
Let $(\alg, \fnl)$ be a non-commutative probability space, a element $a \in \alg$ is self-adjoint if $a = a^*$. 
In particular, the moments of self-adjoint elements are real (see Remark 1.2 in \cite{roland_2006}). 
\end{definition}

The counterpart of independence in free probability is freely independence or simply free.
We now consider the freeness of self-adjoint random variables from various perspectives as in Section \ref{subsec:scalarindependence}.

\subsubsection{Mixed moments point of view}

The following official definition of freeness should be compared with \eqref{eq:def scalar indpendence}.
\begin{definition}  \label{def:free self-adjoint}
Let $(\alg, \fnl)$ be a non-commutative probability space and fix a positive integer $n \geq 1$. 

For each $i \in I$, let $\alg_i \subset \alg$ be a unital subalgebra. The subalgebras $(\alg_i)_{i \in I}$ are called freely independent (or simply free), if for all $k \geq 1$
\begin{equation*}
\fnl(\alge_1\cdots \alge_k) = 0
\end{equation*}
whenever \(\fnl(\alge_j) = 0\) for all $j = 1, \cdots, k,$ and neighboring elements are from diffierent subalgebras, i.e. \(\alge_j \in \alg_{i(j)}\), \(i(1) \neq i(2), i(2) \neq i(3),\cdots, i(k-1)\neq i(k)\).

In particular, a series of elements \((\alge_i)_{i \in I}\) are called free if the subalgebras generated by $\alge_i$ and $\alge_i^*$ are free. 
\end{definition}

\subsubsection{Free cumulants -- free additivity}

The analog of cumulants for non-commutative random variables is called free cumulants, which was proposed by Roland Speicher \cite{Roland_1994, roland_2006}. 

The notion of non-crossing partition lies underneath the free probability and free cumulants. 

\begin{definition}[Non-crossing Partition, Definition 9.1 of \cite{roland_2006}]
\label{def:noncom_partition}
Consider set $S = [1..n]$. 
\begin{enumerate}
    \item[(a)] We call $\pi = \{V_1, \cdots, V_r\}$ a partition of the set $S$ if and only if $V_i$ ($1\leq i \leq r$) are pairwise disjoint, non-void subsets of $S$ such that $\cup_{i =1}^r V_{i} = S$. We call $V_1, \cdots, V_r$ the block of $\pi$. Given two elements $a, b \in S$, we write $a \sim_\pi b$ if for $a$ and $b$ belong to the same block of $\pi$. 
    \item[(b)] A partition $\pi$ of the set $S$ is called non-crossing if there does not exist any $a_1 < b_1 < a_2 < b_2 $ in $S$ such that $a_1 \sim_\pi a_2 \nsim b_1 \sim_{\pi} b_2$.
    \item[(c)] The set of all non-crossing parations of $S$ is denoted by $NC(n)$. 
\end{enumerate}
\end{definition}

\begin{definition}
\label{def:freecum}
Given a $*$-probability space $(\alg, \fnl)$, the free cumulants refer to a family of multilinear functionals $\{\cumu_m: \alg^m \mapsto \C\}_{m \geq 1}$. Here, the multilinearity means that $\cumu_m$ is linear in one variable when others hold constant, i.e., for any $\alpha, \beta \in \C$ and $a, b \in \alg$,
\begin{equation}
\label{eq:multilinear}
  \cumu_m(\cdots, \alpha a + \beta b, \cdots) = \alpha\cumu_m(\cdots, a, \cdots) + \beta\cumu_m(\cdots, b, \cdots).
\end{equation}
\end{definition}

Explicitly, for $a_1, \cdots, a_n \in \alg$, their mixed free cumulant is defined through (cf. \eqref{eq:def scalar cumu 2})
    \begin{equation}
      \fnl(a_1\cdots a_n) = \sum_{\pi \in NC(n)} \prod_{B \in \pi} \cumu_{\abs*{B}}\left(a_i:i\in B \right).
    \end{equation}
    Equivalently (cf. \eqref{eq:def scalar cumu}), 
    \begin{equation}
      \cumu_n(a_1, \cdots, a_n) = \sum_{\pi \in NC(n)} \mu(\pi, \bm{1}_n) \sum_{B \in \pi} \varphi\left(\prod_{i \in B} a_i\right), 
    \end{equation}
    where $\mu$ is the M\"{o}bius function on $NC(n)$.
\begin{example}
We have that
$$
  \cumu_1(a_1) = \fnl(a_1),
$$
$$
  \cumu_2(a_1, a_2) = \fnl(a_1a_2) - \fnl(a_1)\fnl(a_2),
$$
\begin{align*}
  \cumu_3(a_1,a_2,a_3) & =  \fnl(a_1a_2a_3)  - \fnl(a_1)\fnl(a_2a_3) -  \fnl(a_2)\fnl(a_1a_3) \\
  & -  \fnl(a_3)\fnl(a_1a_2) + 2 \fnl(a_1)\fnl(a_2)\fnl(a_3).
\end{align*}
\end{example}

Recall that in the scalar probability, mixed cumulants of independent random variables vanish (see \eqref{eq:cumu_ind_scalar}). 
The same holds for the free cumulants in the free probability.
\begin{theorem}[Theorem 11.16 of \cite{roland_2006}] 
\label{thm:vanish_mix_cumu}
      Let \((\alg,\fnl)\) be a \\ 
  non-commutative probability space with associated free cumulants \((\cumu_\ell)_{\ell \in \N}\). Consider random variables $(\alge_i)_{i\in I}$. Assume that they are freely independent. Then for all \(n \geq 2\), and $i(1), \cdots, i(n) \in I$, we have \(\cumu_n(a_{i(1)},\cdots,a_{i(n)}) = 0\) whenever there exist \(1\leq l,k \leq n\) with \(i(l)\neq i(k)\).
\end{theorem}

With the above theorem, one can easily show the free additivity of free cumulants. 

\begin{proposition}
\label{prop:freeadditivity}
Consider a non-commutative probability space $(\alg, \fnl)$. 
For a self-adjoint random variable $a \in \alg$, set
\begin{equation} \label{eq:cumu_ident_var}
\cumu_m(a):= \cumu_m(a,a,\cdots,a).
\end{equation}
\begin{enumerate}
\item[(a)] For any $m \geq 1$ and $\alpha \in \C$, we have that
\begin{equation}
\label{eq:freecumu_powerlaw}
  \cumu_m(\alpha a) = \alpha^m \cumu_m(a).
\end{equation}
This immediately follows from the multilinearity of free cumulants (see \eqref{eq:multilinear}).

\item[(b)](Free additivity, Proposition 12.3 in \cite{roland_2006})  For any $m \geq 1$, if $a, b \in \alg$ are freely independent, then
\begin{equation}
\label{eq:freeadditivity}
  \cumu_m(a + b) = \cumu_m(a) + \cumu_m(b).
\end{equation}
The above equation should be compared with \eqref{eq:scalar cumu ind additivity}.
\end{enumerate}
\end{proposition}


\subsubsection{Free entropy -- free additivity}
\label{subsubsec:free entropy}
  For non-commutative random variables, the free entropy is introduced by Voiculescu \citep{voiculescu1993analogues,  Voiculescu_1994,  Voiculescu_1997}. 
  Here we provide a brief introduction. 
  Readers are referred to Section 6 of \cite{free_entropy_HP} for further details.
  
  We first examine the Boltzmann-Gibbs formula of classical entropy.
  The idea is that the entropy of a ``macrostate" is proportional to the logarithm of its probability, which is determined by the count of associated ``microstates''. 
  Mathematically, the association is defined through an appropriate distance, and the probability of a ``macrostate'' is given by the volume of all close ``microstates''. 
  This motivates the following formulation of scalar entropy. 
  
  Let $a$ be a random variable supported in a finite interval $[-R, R]$, then its entropy is a limit of log volumes:
  \begin{equation}
      h(a) = \lim_{\substack{r\rightarrow \infty \\ \epsilon \rightarrow 0}} \lim_{N \rightarrow \infty} \frac1N  \log \lambda_N \left(\{ x \in [-R, R]^N : \left\lvert m_k(\delta_N(x)) - m_k(a)\right\rvert \leq \epsilon, k \leq r\}\right),
  \end{equation}
  where $\lambda_N$ is the $N$-deimensional Lebesgue measure, $m_k$ denotes the $k$th moment and $\delta_N(x)$ is the atomic measure $(\delta(x_1) + \delta(x_2) + \cdots + \delta(x_N)) / N$ serving as ``microstates''. 
  Here the volume is Lebesgue measure of $x \in \R^n$ whose corresponding atomic measure approximates $a$ up to $r$th moments.
  One then takes a normalized limit improving the approximation to get entropy.
  
  The moments are estimated using the functional $\fnl(\cdot)$ in free probability.
  Due to the non-commutative nature of matrices and the fact that free independence occurs asymptotically among large matrices (see Section \ref{sec:connection}), one can adapt self-adjoint matrices for ``microstates''. 
  We then arrive at the following definition of free entropy. 
  
  \begin{definition}
  Let $M_N(\C)^{sa}$ denote all $N \times N$ self-adjoint matrices and $\tr(\cdot):= \frac1N \Tr(\cdot)$ denote normalized trace.  
  Given a $*$-probability space $(\alg, \fnl)$ and a self-adjoint element $a \in \alg$. For $n,r \in \N$, $\epsilon > 0$ and $R > 0$, we define the set
  \begin{equation*}
      \Gamma(a; R, N, r, \epsilon) = \{A \in M_N(\C)^{sa}: \norm{A} \leq R, \abs*{\tr(A^k) - \fnl(a^k)} \leq \epsilon, k \leq r\}.
  \end{equation*}
  Recall that there is a natural linear bijection between $M_N(\C)^{sa}$ and $\R^{N^2}$, and let $\Lambda_N$ denote the induced measure on $M_N(\C)^{sa}$ from the Lebesgue measure of $\R^{N^2}$,
  the free entropy of $a$ is then defined by:
  \begin{equation}
      \chi(a) = \sup_{R > 0} \lim_{\substack{r\rightarrow \infty \\ \epsilon \rightarrow 0}} \limsup_{N \rightarrow \infty}\left[\frac1{N^2} \log\Lambda_N\left( \Gamma(a; R, N, r, \epsilon) \right)  + \frac12 \log N\right].
  \end{equation}
  One can extend above definition to multivariate case. For self-adjoint elements $a_1, \cdots, a_\nc \in \alg$, define the set
  \begin{align*}
      \Gamma(a_1,\cdots, a_\nc; & R, N, r, \epsilon) =  \{(A_1,\cdots, A_\nc) \in (M_N(\C)^{sa})^{\nc}: \norm{A_i} \leq R, \\
      & \abs*{\tr(A_{i_1} \cdots A_{i_k}) - \fnl(a_{i_1} \cdots a_{i_k})} \leq \epsilon \text{ for all $1 \leq i_1, \cdots, i_k \leq \nc$}, k \leq r\},
  \end{align*}
  the joint free entropy is then given by
    \begin{equation}
    \begin{aligned}
      & \chi(a_1,  \cdots, a_\nc) = \\
       & \sup_{R > 0} \lim_{\substack{r\rightarrow \infty \\ \epsilon \rightarrow 0}} \limsup_{N \rightarrow \infty}\left[\frac1{N^2} \log\Lambda_N^{\otimes \nc}\left( \Gamma(a_1,\cdots, a_\nc; R, N, r, \epsilon) \right)  + \frac \nc 2 \log N\right].
    \end{aligned}
  \end{equation}
  \end{definition}
  
  The free entropy shares the similar properties with the scalar entropy. 
   
\begin{proposition}
\label{prop:entlintran}
      Let $\bm \alge = (\alge_1,\cdots,\alge_\nc)^T$ where $\alge_i$ are self-adjoint non-commutative random variables.  Let $\orth(\nc)$ denote the set of $\nc \times \nc$ orthogonal matrices. Then for any $\mixmq = (q_{ij})_{i,j=1}^\nc \in \orth(\nc)$, 
        \begin{equation}
        \label{eq:ent_orth}
          \ent\left((\mixmq \bm \alge)_1,\cdots,(\mixmq \bm \alge)_\nc\right) = \ent(\alge_1,\cdots,\alge_\nc).
        \end{equation}
        That is, the free entropy is invariant under the orthogonal transformation (cf. \eqref{eq:ent_linear_transformation}).
    \end{proposition}
    \begin{proof}
    This proposition is a special case of a general result. For any matrix $\bm A \in \R^{n \times n}$, we actually have that (see Corollary 6.3.2 in \cite{free_entropy_HP}),
    \begin{equation}
    \label{eq:ent_orth_1}
      \ent\left((\bm A \bm \alge)_1,\cdots,(\bm A \bm \alge)_\nc\right) = \ent(\alge_1,\cdots,\alge_\nc) + \log \abs*{\det \bm A}.
    \end{equation}
    Now, for $\mixmq \in \orth(\nc)$, $\mixmq^T\mixmq = \bm I$, thus
    \begin{equation}
      (\det \mixmq)^2 = \det \mixmq^T \det \mixmq = \det (\mixmq^T \mixmq) = \det \bm I = 1.
    \end{equation}
    That is, $\abs*{\det \mixmq} = 1$ and thus $\log\abs*{\det \mixmq} = 0$. Now, set $\bm A = \mixmq$ in \eqref{eq:ent_orth_1},  we obtain \eqref{eq:ent_orth}. 
    \end{proof}

 The following proposition is the analogue of \eqref{eq:ent_ind_scalar} for free entropy.
 
    \begin{proposition}
     \label{prop:entfreeind}
      Let \(\alge_1,\cdots,\alge_\nc\) be self-adjoint non-commutative random variables, then
        \begin{equation}
        \label{eq:entfreeind}
          \ent(\alge_1,\cdots,\alge_\nc) \leq \sum_{i = 1}^\nc \ent(\alge_i).
        \end{equation}
    Further assume that \(\ent(\alge_i) > -\infty \) for $i = 1,\cdots,n$, then the above equality holds if and only if \(\alge_1,\cdots,\alge_\nc\) are freely independent. 
\begin{proof}
The proof for the inequality can be found in Proposition 6.1.1 in \cite{free_entropy_HP}. The equivalence between the equality and freely independence is Theorem 6.4.1 in  \cite{free_entropy_HP}.
\end{proof}
    \end{proposition}

\subsubsection{Analogue of Gaussian random variables in free probability: the free semi-circular element}

\label{subsubsec:semicirclr}
The analogous element to a Gaussian random variable in a $*$-probability space is a semicircular element. 
Recall that the Gaussian random variable is characterized by vanishing cumulants of order higher than 2, the semicircular elements can be defined in a similar manner.
\begin{definition}
\label{def:semi}
Given a $*$-probability space $(\alg, \fnl)$, we call a random variable $a \in \alg$ a semicircular element if 
\begin{equation}
\cumu_m(a) \equiv 0, \qquad \text{for $m \geq 3$},
\end{equation}
and $\cumu_2(a) > 0$ (such that $a$ is not constant).
\end{definition}

\subsection{Freeness of non-self-adjoint random variables}   \label{subsec:prelim rec}
We briefly introduce the mathematical preliminaries for a rectangular probability space. 
We omit some technicalities, which are beyond the scope of this paper.
For a thorough introduction, readers are referred to \cite{florent_free_probability,florent_free_entropy}.

Consider a $*$-probablity space $(\alg, \fnl)$ with $p_1,p_2$ of non-zero self-adjoint projections which are pairwise orthogonal (i.e. $\forall i \neq j, p_ip_j = 0$), and such that $p_1 + p_2 = 1_\alg$. Then any element $a \in \alg$ can be represented in the following block form
\begin{equation}  \label{eq:matrix decomposition free prob}
    a = \begin{bmatrix} a_{11} & a_{12} \\ a_{21} & a_{22}\end{bmatrix},
\end{equation}
where $\forall i,j = 1,2, a_{ij} = p_i a p_j$ and we define $\alg_{ij} := p_i \alg p_j$.
Note that $\alg_{ii}$ is a subalgebra, and we equip it with the functional $\fnl_k = \frac{1}{\rho_k} \fnl \vert_{\alg_{kk}}$, where $\rho_k := \fnl(p_k)$.
That is,
\begin{equation}
    \varphi_1(a) = \frac{1}{\rho_1} \varphi(a^{11}), \text{where $a^{11} = \begin{bmatrix} a_{11} & 0 \\ 0 & 0\end{bmatrix}$},
\end{equation}
and similar for $\varphi_2(x)$.
The functionals $\fnl_i$, $i = 1, 2$ are tracial in the sense that $\fnl_k(p_k) = 1$ and for all $i,j$, $x \in \alg_{ij}$, $y \in \alg_{ji}$,
\begin{equation}
    \rho_i \fnl_i(xy) = \rho_j \fnl_j(yx).
\end{equation}

\begin{definition}
Such a family $(\alg, p_1, p_2, \fnl_1, \fnl_2)$ is called a $(\rho_1, \rho_2)$-rectangular probability space. 
We call $a \in \alg_{12} = p_1 \alg p_2$ rectangular random variable.
\end{definition}

\begin{remark}
If $a$ is a rectangular element, then in the matrix decomposition \eqref{eq:matrix decomposition free prob}, only $a_{12}$ is non-zero.
Later in Section \ref{subsubsec:recmat}, we will model rectangular matrices by embedding them into $a_{12}$ of rectangular random variables.
\end{remark}

For such a rectangular probability space, the linear span of $p_1, p_2$ is denoted by $\cond$. 
Then $\cond$ is subalgebra of finite dimension.
Define the $\E_\cond(a) = \sum_{i = 1}^2 \fnl(a_{ii})p_i$. 
It can be checked that $\E_\cond(1_\alg) = 1_\alg$ and $\forall (d, a, d') \in \cond \times \alg \times \cond$, $\E_\cond(dad') = d\E_\cond(a)d'$.
The map $\E_\cond(\cdot)$ is regarded as the conditional expectation from $\alg$ to $\cond$. 

We now consider the freeness in rectangular probability space.

\subsubsection{Mixed moments point of view}

The following definition of freeness should be compared with \eqref{eq:def scalar indpendence} and Definition \ref{def:free self-adjoint}.
\begin{definition}
Given a rectangular probability space and subalgebra $\cond$ with the corresponding conditional expectation $\E_\cond$. 
A family $(\alg_i)_{i\in I}$ of subalgebras containing $\cond$ is said to be free with amalgamation over $\cond$ (we simply use the word free when there is no ambiguity) if for all $k \geq 1$
\begin{equation}
    \E_\cond(x_1 \cdots x_k) = 0
\end{equation}
whenever $\E_\cond(x_j) = 0$ for all $j = 1, \cdots, k$, and neighboring elements are from different subalgebras, i.e., $\alge_j \in \alg_{i(j)}$, $i(1) \neq i(2), i(2) \neq i(3),\cdots, i(k-1)\neq i(k)$. In particular, a family of rectangular random variables \(\{\alge_i\}_{i\in I}\) are called free if the subalgebras generated by $\cond$, $x_i$, and $x_i^*$ are free. 
\end{definition}

\subsubsection{Rectangular free cumulants -- free additivity}
\label{sec:rect_free_cumulant}

The free cumulants are also defined for rectangular probability space \cite{florent_free_probability, florent_free_entropy}.
\begin{definition}[Analogue of cumulant in rectangular probability space]
\label{def:freecum_rec}
Given a \\
$(\rho_1,\rho_2)$-probability space $(\alg, p_1, p_2, \fnl_1, \fnl_2)$, for any $n \geq 1$, we denote $n$-th tensor product over $\cond$ of $\alg$ by $\alg^{\otimes_{\cond^m}}$. We recall a family of linear functions $\{\cumu_{m}: \alg^{\otimes_{\cond^m}} \mapsto \C\}_{m \geq 1}$ introduced in \cite{florent_free_probability} (which are denoted as $c^{(1)}$ in \cite{florent_free_probability}, see Section 3.1 there).
By linearity, we mean that for $m \geq 1$ and any $a, b \in \alg$ and $a, b \in \C$, 
\begin{equation}
\label{eq:linearity_rec}
    \cumu_m(\cdots \otimes (\alpha a + \beta b) \otimes \cdots) = \alpha \cumu_m(\cdots \otimes a\otimes \cdots) + \beta \cumu_m(\cdots \otimes b\otimes \cdots).
\end{equation}
For convenience, we call $\{\cumu_{m}\}_{m \geq 1}$ rectangular free kurtosis (or kurtosis when there is no ambiguity). 
For each $m \geq 1$ and any rectangular random variable $a$, we put
\begin{equation}    \label{eq:reccumu single}
        \cumu_{2m}(a) := \cumu_{2m}(a \otimes a^* \otimes \cdots \otimes a \otimes a^*).
\end{equation}
We consider consider the even order as odd order cumulants vanishes for all rectangular elements.
\end{definition}

\begin{remark}
In \cite{florent_free_probability, florent_free_entropy}, the free cumulants refer to a family of linear functions between $\alg^{\otimes_{\cond^n}}$ and $\cond$. 
The rectangular cumulants throughout the paper are their coefficient functions of $p_1$.
\end{remark}

The following vanishing lemma holds for the rectangular cumulants defined as in above.
\begin{theorem}[Vanishing of mixed cumulants, Theorem 2.1 of \cite{florent_free_entropy}]
\label{thm:vanish_mix_cumu_rec}
     A family $(x_i)_{i \in I}$ of elements in $\alg$ is free with amalgamation over $\cond$ if and only if for all $n \geq 2$, and $i(1), \cdots, i(n) \in I$, we have $\cumu_n(x_{i(1)}\otimes \cdots \otimes x_{i(n)}) = 0$ whenever there exists $1\leq l, k\leq n$ with $i(l) \neq i(k)$.
\end{theorem}

Consequently, the analogue of Proposition \ref{prop:freeadditivity} also hold for rectangular cases with rectangular free kurtosis defined in \eqref{eq:reccumu single}. 
The analog of \eqref{eq:freeadditivity} for rectangluar free kurtosis follows from equation (10) in \cite{florent_free_probability}.
The analog of \eqref{eq:freecumu_powerlaw} is a direct result of \eqref{eq:linearity_rec}.

\subsubsection{Rectangular free entropy -- free additivity}
\label{subsubsec:free entropy rec}

 The free entropy $\ent$ for rectangular free probability space is introduced in \cite{florent_free_entropy}. 
 The idea is similar to the self-adjoint case. 
 One adapts rectangular matrices as ``microstates'' and use conditional expectation $\E_\cond(\cdot)$ to evaluate moments.
 Readers are referred to Section 5.1 of \cite{florent_free_entropy} for a precise definition. 
 
  
The analogues of Proposition \ref{prop:entlintran} and \ref{prop:entfreeind} also hold for rectangular free entropy. 
The orthogonal invariance of rectangular free entropy is a direct result of Corollary 5.11 of \cite{florent_free_entropy}.
On the other hand, Proposition 5.3, Theorem 5.7 and Corollary 5.16 of \cite{florent_free_entropy} together prove the analogue of Proposition \ref{prop:entfreeind} for rectangular case.

\subsubsection{Analogue of Gaussian random variables in rectangular free probability: the free Poisson element}

\label{subsubsec:poisson}
\begin{definition}
\label{def:poisson}
    Given an rectangular probability space $(\alg,\fnl)$. 
    An rectangular random variable $a \in \alg_{12}$ is a free Poisson element if
    \begin{equation}
        \cumu_{2m}(a) \equiv 0, \qquad \text{for $m \geq 2$}.
    \end{equation}
\end{definition}
    
\subsection{When are random matrices (asymptotic) free?}
\label{sec:connection}
Here, we describe the free probability in the context of random matrices and the explicit formulas of free kurtosis and entropy as functions of the input matrices.
\subsubsection{Symmetric random matrix}
\label{sec:connection sym}
Given a $N > 0$, we consider the algebra consists of all the real $N \times N$ matrices over scalar random variables $L^{2}(\Sigma, P)$:
\begin{equation}
    \alg = M_N(L^2(\Sigma, P))
\end{equation}
and for any $\bm X \in \alg$, the functional $\fnl$ on it is
\begin{equation}
    \fnl(\bm X) = \frac1N \E[ \Tr(\bm X)].
\end{equation}
Denote the matrix transpose with complex conjugate by $*$. 
Then $(\alg, \fnl)$ is a $*$-probability space.

We recall the notion of convergence in distribution and the definition of asymptotic freely independence \cite{roland_2006}.
\begin{definition}[Asymptotic freely independence]
\label{def:asy_ind}
Let, $(\alg_N , \fnl_N)~(N \in \mathbbm{N})$ and $(\alg, \fnl)$ be non-commutative probability spaces. Let $I$ be an index set and consider for each $i \in I$ random variables $a_i(N) \in \alg_N$ and $a_i \in \alg$. We say that $(a_i(N))_{i \in I}$ converges in distribution towards $(a_i)_{i \in I}$ if we have each joint moment of $(a_i(N))_{i \in I}$ converges towards the corredsponding joint moment of $(a_i)_{i \in I}$, i.e., for all $n \in N$ and all $i(i), \cdots, i(n) \in I$
\begin{equation}
    \lim_{N \rightarrow \infty} \fnl_N(a_{i(1)}(N)\cdots a_{i(n)}(N)) = \fnl(a_{i(1)}\cdots a_{i(n)}).
\end{equation}
Furthermore, we say $(a_i(N))_{i \in I}$ are asymptotic free if it converges in distribution to a limit $(a_i)_{i\in I}$, which is free in $(\alg,\fnl)$. 
\end{definition}

A pair of symmetric (Hermitian) random matrices with isotropically random eigenvectors that are independent of the eigenvalues (and each other) are asymptotically free \citep{roland_2006}.

Given the $*$-probability space $(\alg, \fnl(\cdot))$ defined as above, recall the free kurtosis defined in \eqref{eq:free_kurtosis_expfor}.
Thus for a self-adjoint random matrix $\bm X \in \alg$ with $\fnl(\bm X) = 0$, the free kurtosis is explicitly given by
\begin{equation}
    \cumu_4(\bm X) = \frac{1}{N} \E[\Tr(\bm X^4)] - 2\left(\frac 1N \E[\Tr(\bm X^2)]\right)^2.
\end{equation}
Also, denote the eigenvalues density function of $\bm X$ by $\mu(x)$, free entropy is defined by \citep{free_entropy_HP}
\begin{equation}
    \ent(\bm X) = \int \int \log \abs*{x - y} \mathrm{d} \mu(x) \mathrm{d} \mu(y).
\end{equation}

For a large class of random matrices $\bm X$, the free kurtosis and entropy concentrate around a deterministic value when $N$ is large. 
For example, if $\bm X$ is a Wigner matrix or Wishart matrix, then $\mathrm{Var}[\cumu_4(\bm X)] \rightarrow 0$ and $\mathrm{Var}[\ent(\bm X)] \rightarrow 0$ as $N \rightarrow \infty$. 
Thus single sample gives us an accurate empirical estimate.
Given a realization $x$ of a random matrix $\bm X$ with $\E \left[\Tr(\bm X)\right] = 0$, the empirical free kurtosis is
\begin{equation}
\label{eq:emp_kurt_sym}
     \widehat \cumu_4(x) = \frac{1}{N} \Tr(x^4) - 2\left(\frac 1N \Tr(x^2)]\right)^2.
\end{equation}
Also, the empirical free entropy is given by
\begin{equation}
\label{eq:emp_ent_sym}
    \widehat \ent(x) = \frac{1}{N(N - 1)}\sum_{i \neq j} \log \abs*{\eg_i - \eg_j},
\end{equation}
where $\eg_i$ denotes the eigenvalue of $x$.

\subsubsection{Rectangular random matrix} \label{subsubsec:recmat}
Consider a rectangular random matrix of size $N \times M$, and assume that $N \leq M$. 
In \cite{florent_free_probability}, the author embedded a $N \times M$ matrix into the top right block of a $(N + M) \times (N + M)$ "extension matrix". 
The algebra of all $(N + M) \times (N + M)$ random matrices together with this block structure is defined as an rectangular probability space $(\mathbb M_{N + M}(L^2(\Sigma, \P)), \mathrm{diag}(I_N, 0_M), \mathrm{diag}(0_N, I_M), \frac1N \Tr, \frac1M \Tr)$ \cite{florent_free_probability}.

We recall the following definition of asymptotic freely independence in rectangular probability space \cite{florent_free_entropy}. 
\begin{definition}
[Asymptotic free independence]
\label{def:asy_ind_rec}
Let, for each $N \in \N$, $(\alg_N, p_1(N), p_2(N), \fnl_{1,N}, \fnl_{2,N})$ be a $(\rho_{1,N}, \rho_{2,N})$-rectangular probability space such that
\begin{equation*}
    (\rho_{1,N}, \rho_{2,N}) \rightarrow (\rho_1, \rho_2), \qquad N \rightarrow \infty.
\end{equation*}
Let $I$ be an index set and consider for each $i \in I$ random variables $a_i(N) \in \alg_N$. We say that $(a_i(N))_{i\in I}$ converges in $\cond$-distribution towards $(a_i)_{i\in I}$ for some random variables $a_i \in \alg$ in some $(\rho_1, \rho_2)$-probability space $(\alg,p_1, p_2,\fnl_1, \fnl_2)$ if the $\cond$-distribution converge pointwise.

Furthermore, we say $(a_i(N))_{i \in I}$ are asymptotically free $(N \rightarrow \infty)$, if the limits $(a_i)_{i\in I}$ are free in $(\alg,p_1, p_2,\fnl_1, \fnl_2)$.
\end{definition}
Independent bi-unitary invariant rectangular random matrices with converging singular law are asymptotically freely independent \citep{florent_free_probability, florent_free_entropy}. 

Following \eqref{eq:def_kurt_rec}, the free kurtosis for a single $N \times M$ random matrices $\bm X$ is given by 
\begin{equation}
    \cumu_4(\bm X) = \frac{1}{N}\E[\Tr((\bm X \bm X^H)^2)] - (1 + \frac NM) \left(\frac1N \E[\Tr((\bm X \bm X^H))]\right)^2. 
\end{equation}
Denoting the probability density function of eigenvalues of $\bm X \bm X^H$ by $\mu(x)$, setting $\alpha = \frac{N}{N + M}$ and $\beta = \frac{M}{N + M}$, the free entropy is given by \citep{florent_free_entropy}
\begin{equation}
    \ent (\bm X) = \alpha^2\int \int \log\abs*{x - y}\mathrm{d}\mu(x) \mathrm{d}\mu(y) + (\beta - \alpha)\alpha \int \log x \mathrm{d}\mu(x).
\end{equation}
Again, empirical statistics over a single sample of large dimension give an accurate estimate of limit value. 
Given a realization $x$ of a rectangular random matrix $\bm X$, the empirical free kurtosis is given by
\begin{equation}
\label{eq:emp_kurt_rec}
    \widehat \cumu_4(x) = \frac{1}{N}\Tr((
    x x^H)^2) - (1 + \frac{N}{M})  \left(\frac{1}{N}\Tr(x x^H)\right)^2.
\end{equation}
The empirical free entropy is given by
\begin{equation}
\label{eq:emp_ent_rec}
    \widehat \ent (x) = \frac{\alpha^2}{N(N - 1)} \sum_{i\neq j} \log \abs*{\eg_i - \eg_j} + \frac{(\beta - \alpha)\alpha}{N} \sum_{i = 1}^N \log \eg_i,
\end{equation}
where $\eg_i$ denote the eigenvalue of $xx^H$.



\section{Proof of Proposition \ref{prop:positive cov} and \ref{prop:real positiv cov}}
We proof Proposition \ref{prop:positive cov} and \ref{prop:real positiv cov} for the covariance matrix for rectangular case. The self-adjoint case can be proved with straightforward modification.
\subsection{Proof of Proposition \ref{prop:positive cov}}
By Remark 1.2 of \cite{roland_2006}, for any random variable $a$, $\fnl(a^*) = \overline{\fnl(a)}$. 
Thus,
\begin{equation}
\begin{aligned}
    \overline{[\cov_{\bm z \bm z}]}_{ij} & = \overline{\fnl_1(\tilde z_i \tilde z_j^*)} \\
    & = \fnl_1((\tilde z_i \tilde z_j^*)^*)\\
    & = \fnl_1(\tilde z_j\tilde z_i^*)= [\bm C_{\bm z \bm z}]_{ji}.
\end{aligned}
\end{equation}
Therefore, $\cov_{\bm z \bm z}$ is Hermitian.

We turn to show that $[\cov_{\bm z \bm z}]$ is positive semi-definite. Actually, as $\fnl$ is a linear functional, for any column vector $\bm \alpha = [\alpha_1, \cdots, \alpha_\nc]$,
\begin{equation}
    \bm \alpha \cov_{\bm z \bm z} \bm \alpha^H = \fnl((\sum_{i = 1}^\nc \alpha_i \tilde z_i) (\sum_{i = 1}^\nc \alpha_i \tilde z_i)^*) \geq 0
\end{equation}
where we used that $\fnl(\cdot)$ is positive.
This completes the proof.

\subsection{Proof of Proposition \ref{prop:real positiv cov}}
Since $\bm z = \bm A \bm x$ and $\bm C_{xx} = \bm I$, 
$$
    \bm C_{\bm z\bm z} = \bm A  \bm C_{\bm x\bm x} \bm A^H = \bm A\bm A^H.
$$
Note that we assume that $\bm A$ is real and non-singular, $\bm C_{\bm z\bm z}$ is real and positive-definite.

\section{Proofs of the main results}

\subsection{Proof of  Theorem \ref{thm:kurt1}}

The proof of the Theorem \ref{thm:kurt1} relies on the free additivity of free cumulants, for which readers are referred to Proposition \ref{prop:freeadditivity} (and its rectangular analogue in Section \ref{sec:rect_free_cumulant}). 

\subsubsection{Proof of Theorem \ref{thm:kurt1} (a)}
Set $\bm g = \mixmq^T \argkurt$, then $\argkurt = \mixmq \bm g$. As $\bm x$ and $\bm y$ are related via \eqref{eq:sym fca model}, we have that 
\begin{equation}
\label{eq:proof_kurt1_0}
  \argkurt^T \bm y = \argkurt^T \mixmq \bm x = (\mixmq^T \argkurt)^T \bm x =\bm g^T\bm x.
\end{equation}
Adapt the notation $\bm g = (g_1,\cdots,g_\nc)^T$. Note that $\alge_i$ are freely independent, then using \eqref{eq:freeadditivity}, we have that
\begin{equation}
  \cumu_4(\bm g^T\bm x) = \cumu_4\left(\sum_{i = 1}^\nc g_i\alge_i\right) = \sum_{i = 1}^\nc \cumu_4(g_i \alge_i).
\end{equation}
By \eqref{eq:freecumu_powerlaw}, $\cumu_4(g_i\alge_i) = g_i^4 \cumu_4(\alge_i)$ for $i = 1, \cdots, \nc$, thus the above equation becomes
\begin{equation}
\label{eq:proof_kurt1_1}
  \cumu_4(\bm g^T\bm x) =  \sum_{i = 1}^\nc g_i^4 \cumu_4(x_i).
\end{equation}
%
Combining \eqref{eq:proof_kurt1_0} and \eqref{eq:proof_kurt1_1}, we get
\begin{equation}
\label{eq:proof_kurt1_3}
  \left \lvert\cumu_4(\argkurt^T \bm y) \right \rvert= \left\lvert \sum_{i = 1}^\nc  g_i^4\cumu_4(\alge_i) \right\rvert.
\end{equation}

When $\argkurt$ runs over all unit vectors, $g = \mixmq^T \argkurt$ also runs over all unit vectors. 
Therefore, if $\argkurt^{(1)}$ is a maximizer of \eqref{eq:opt_kurt}, then $\argkurt^{(1)} = \mixmq g^{(1)}$ where $g^{(1)}$ is a maximizer of
\begin{equation}
\label{eq:proof_kurt1_2}
  \mathop{\max}_{\bm g \in \R^\nc,~\norm{u} = 1}  \left\lvert \sum_{i = 1}^\nc  g_i^4\cumu_4(\alge_i) \right\rvert.
\end{equation}
Thus in order to prove (a), it is equivalent to show that $\bm g^{(1)}$ is maximizer of \eqref{eq:proof_kurt1_2} if and only if  $\bm g^{(1)} \in \{(\pm 1, 0,\cdots,0)^T\}$.

For any unit vector $u$, since $\abs*{g_i} \leq 1$, we have that
\begin{equation}
\label{eq:proof_kurt_5}
\sum_{i = 1}^\nc g_{i}^4 \leq \sum_{i = 1}^\nc g_i^2 = 1.
\end{equation}
Note that the equality holds if and only if there is a index $i$ such that $g_i \in \{\pm 1\}$ (thus $g_j = 0$ for all $j \neq i$). 
Then using \eqref{eq:kurt1_main} and \eqref{eq:proof_kurt_5},
\begin{equation}
\label{eq:proof_kurt1_4}
\begin{aligned}
  \left\lvert \sum_{i = 1}^\nc  g_i^4\cumu_4(\alge_i) \right\rvert & \leq \sum_{i = 1}^\nc g_i^4 \abs*{\cumu_4(\alge_i)}  \\
& \leq \sum_{i = 1}^\nc g_i^4 \abs*{\cumu_4(\alge_1)}  \\
& \leq  \abs*{\cumu_4(\alge_1)}.
\end{aligned}
\end{equation}
On the other hand, for $\bm g = (\pm1, 0,\cdots, 0)^T$, it can be checked that all equalities in \eqref{eq:proof_kurt1_4} hold.
Thus
\begin{equation}
\label{eq:proof_kurt1_6}
\mathop{\max}_{\bm g \in \R^\nc,~\norm{\bm g} = 1}  \left\lvert \sum_{i = 1}^\nc  g_i^4\cumu_4(\alge_i) \right\rvert = \abs*{\cumu_4(\alge_1)}
\end{equation}
and $\bm g^{(1)}$ is a maximizer of \eqref{eq:proof_kurt1_2} if $\bm g^{(1)} \in \{(\pm 1, 0,\cdots,0)^T\}$.

For the other direction, if $\bm g^{(1)}$ is maximizer of \eqref{eq:proof_kurt1_2}, then the second equality in \eqref{eq:proof_kurt1_4} holds for $\bm g = \bm g^{(1)}$.
That is,
\begin{equation}
\label{eq:proof_kurt1_7}
  0 = \sum_{i = 1}^\nc (g^{(1)}_i)^4 \left(\abs*{\cumu_4(\alge_i)} - \abs*{\cumu_4(\alge_1)}\right).
\end{equation}
Due to \eqref{eq:strictdeckurt}, $\abs*{\cumu_4(\alge_i)} - \abs*{\cumu_4(\alge_1)} < 0$ for $i = 2,\cdots,\nc$. Thus \eqref{eq:proof_kurt1_7} implies $g^{(1)}_i = 0$ for $i = 2,\cdots,\nc$.
Since $\bm g^{(1)}$ is a unit vector, $\bm g^{(1)} \in \{(\pm 1, 0, \cdots, 0)^T\}$.
This completes the proof.

\subsubsection{Proof of Theorem \ref{thm:kurt1} (b)}

In the proof of (a), the arguments up to \eqref{eq:proof_kurt1_7} only rely on properties of free kurtosis $\cumu(\cdot)$ and condition \eqref{eq:kurt1_main}.
Thus \eqref{eq:proof_kurt1_2}, \eqref{eq:proof_kurt1_4}, \eqref{eq:proof_kurt1_6} and \eqref{eq:proof_kurt1_7} also apply in the setting of (b). 
Thus in order to prove (b), it is equivalent to show that $u^{(1)}$ is a maximizer of \eqref{eq:proof_kurt1_2} if and only if
\begin{enumerate}
\item[(i)] $g^{(1)}_i = 0$ for $i = r + 1,\cdots,\nc$,
\item[(ii)]  there is an index $i$ such that $g_i^{(1)} \in \{\pm 1\}$.
\end{enumerate}
The backward direction can be checking directly using $\abs*{\cumu_4(x_1)} = \cdots = \abs*{\cumu_4(x_r)}$.

We now prove the forward direction.
If $\bm g^{(1)}$ maximizes  \eqref{eq:proof_kurt1_2}, then it satisfies \eqref{eq:proof_kurt1_7}. 
By \eqref{eq:firstrequal}, $\abs*{\cumu_4(\alge_i)} - \abs*{\cumu_4(\alge_1)} = 0$ for $i = 1,\cdots,r$ and $\abs*{\cumu_4(\alge_i)} - \abs*{\cumu_4(\alge_1)} < 0$ for $i = r + 1,\cdots,\nc$. (i) then follows.
On the other hand, as $\abs*{\cumu_4(\alge_1)} = \cdots = \abs*{\cumu_4(\alge_r)}$, enforcing the third equality in \eqref{eq:proof_kurt1_4} implies
\begin{equation}
\label{eq:proof_kurt1_8}
  \sum_{i = 1}^r (g^{(1)}_i)^4 = 1.
\end{equation}
By the observation below \eqref{eq:proof_kurt_5}, this indicates indicates (ii).
This completes the proof.

\subsection{Proof of  Theorem \ref{thm:kurt2}}
Set $\bm g = \mixmq^T \argkurt$, we use the notation $\bm g = [g_1, \cdots, g_\nc]^T$. 
As $\bm w^{(i)} \in \{\pm \mixmq_i\}$ for $i = 1, \cdots, k - 1$,
\begin{equation}
  \norm{\argkurt} = 1, \argkurt \perp \argkurt^{(1)}, \cdots, \argkurt^{(k - 1)} \iff \norm{\bm g} = 1, g_1 = \cdots = g_{k - 1} = 0.
\end{equation}
Using \eqref{eq:proof_kurt1_3}, if  $\argkurt^{(k)}$ is a maximizer of \eqref{eq:kurt2_a}, then $\argkurt^{(k)} = \mixmq \bm g^{(k)}$ where $\bm g^{(k)}$ is a maximizer of 
\begin{equation}
\label{eq:proof_kurt2_1}
  \mathop{\mathop{\max}_{\bm g \in \R^\nc, ~\norm{\bm g} = 1}}_{g_1 = \cdots = g_{k - 1} = 0} \left\lvert \sum_{i = 1}^n  g_i^4\cumu_4(\alge_i) \right\rvert.
\end{equation}

Thus in order to prove (a), it is equivalent to show that $\bm g^{(k)} = (g_1^{(k)}, \cdots, g_\nc^{(k)})^T$ is maximizer of \eqref{eq:proof_kurt2_1} if and only if  $g^{(k)}_k \in \{\pm 1\}$ (thus $g^{(k)}_j = 0$ for $j \neq k$).

As we are maximizing over unit vector $\bm g$ such that $g_1 = \cdots = g_{k - 1} = 0$, again using \eqref{eq:kurt1_main} and \eqref{eq:proof_kurt_5}
\begin{equation}
\label{eq:proof_kurt2_2}
\begin{aligned} \left\lvert\sum_{i = 1}^{\nc}  g_i^4\cumu_4(\alge_{i}) \right\rvert  &= \left\lvert\sum_{i = k}^{\nc}  g_i^4\cumu_4(\alge_{i}) \right\rvert\\
 & \leq \sum_{i = k}^\nc g_i^4 \abs*{\cumu_4(x_i)}  \\
& \leq \sum_{i = k}^\nc g_i^4 \abs*{\cumu_4(x_k)} \\
& \leq  \abs*{\cumu_4(x_k)}.
\end{aligned}
\end{equation}
For $\bm g$ with $g_k \in \{ \pm 1\}$, it can be checked that all equalities in \eqref{eq:proof_kurt2_2} hold. Thus
\begin{equation}
  \mathop{\mathop{\max}_{g \in \R^\nc, ~\norm{g} = 1}}_{g_1 = \cdots = g_{k - 1} = 0} \abs*{\sum_{i = 1}^\nc  g_i^4\cumu_4(x_i) } = \abs*{\cumu_4(x_k)},
\end{equation} 
and $\bm g^{(k)}$ is a maximizer if $g^{(k)}_k \in \{ \pm 1\}$.

For the other direction, if $\bm g^{(k)}$ is a maximizer of \eqref{eq:proof_kurt2_1}, all equalities in  \eqref{eq:proof_kurt2_2} hold with $g = g^{(k)}$. 
In particular, the third equality in  \eqref{eq:proof_kurt2_2} implies
\begin{equation}
\label{eq:proof_kurt2_3}
    0 = \sum_{i = k}^\nc \left(g^{(k)}_i\right)^4 \left(\abs*{\cumu_4(x_i)} - \abs*{\cumu_4(x_k)}\right).
\end{equation}
Due to \eqref{eq:strictdeckurt}, $\abs*{\cumu_4(x_i)} - \abs*{\cumu_4(x_k)} < 0$ for $i = k + 1,\cdots,n$. 
Thus \eqref{eq:proof_kurt2_3} implies that $g^{(k)}_i = 0$ for $i = k + 1, \cdots, \nc$.
Since $\bm g^{(k)}$ is a unit vector, $g^{(k)}_k \in \{\pm 1\}$.
This completes the proof.

\subsection{Proof of  Theorem \ref{thm:kurt3}}

We prove Theorem \ref{thm:kurt3} by showing the following:

\begin{itemize}
    \item[(a)] $\mixmq$ is a maximizer of \eqref{eq:kurt3}.
    \item[(b)] For any permutation matrix $\permm$ and signature matrix $\signm$, $\mixmq \signm \permm$ is a maximizer of \eqref{eq:kurt3}.
    \item[(c)] Any maximizer $\Argkurt$ of \eqref{eq:kurt3} satisfies $\Argkurt = \mixmq \signm \permm$ for some permutation matrix $\permm$ and signature matrix $\signm$.
\end{itemize}

\subsubsection{Proof of (a)}
We prove (a) by showing that
\begin{equation}
\label{eq:thm:kurt3_2}
  \mathop{\max}_{\Argkurt \in \orth(\nc)} \sum_{i = 1}^\nc \left \lvert \cumu_4\left((\Argkurt^T \bm y)_i\right) \right \rvert =  \sum_{i = 1}^\nc \abs*{\cumu_4(\alge_i)}
\end{equation}
and $\Argkurt = \mixmq$ reaches the maximum. 
Set $\bm G = \mixmq^T\Argkurt \in \orth(\nc)$. As $\bm x$ and $\bm y$ are related via \eqref{eq:sym fca model},
\begin{equation}
\begin{aligned}
  \Argkurt^T \bm y & = \Argkurt^T \mixmq \bm x\\
  & = (\mixmq^T\Argkurt)^T \bm x \\
  & = \bm G^T \bm x.
\end{aligned}
\end{equation}
Adapt the notation $\bm G = (g_{ij})_{i,j= 1}^\nc$. Then for all $i = 1,\cdots,n$, $(\Argkurt^T \bm y)_i = (\bm G^T \bm x)_i = \sum_{j = 1}^\nc g_{ji} x_j$. 
Together with \eqref{eq:freeadditivity} and \eqref{eq:freecumu_powerlaw}, for any $i = 1, \cdots, \nc$, we have that
\begin{equation}
\begin{aligned}
  \cumu_4((\Argkurt^T \bm y)_i) = &  \cumu_4\left(\sum_{j = 1}^\nc g_{ji} x_j\right) \\ 
  = & \sum_{j = 1}^\nc \cumu_4\left(g_{ji}\alge_j\right) \\
  = & \sum_{j = 1}^\nc g_{ji}^4 \cumu_4(\alge_j).
\end{aligned}
\end{equation} 
Apply triangular inequality to above equation, we get
\begin{equation}
\label{eq:thm:kurt3_0}
  \left \lvert \cumu_4((\Argkurt^T \bm y)_i)\right \rvert \leq \sum_{j = 1}^\nc g_{ji}^4 \left \lvert\cumu_4\left( \alge_j\right)\right \rvert.
\end{equation}
Note that $(g_{j1},\cdots,g_{jn})^T$ is a unit vector, by \eqref{eq:proof_kurt_5}, $\sum_{j = 1}^\nc g_{ji}^4 \leq 1$. 
Then summing \eqref{eq:thm:kurt3_0} over $i = 1, \cdots, n$, we obtain that
\begin{equation}
\label{eq:thm:kurt3_1}
\begin{aligned}
   \sum_{i = 1}^\nc \left \lvert \cumu_4((\Argkurt^T y)_i)\right \rvert \leq & \sum_{i = 1}^\nc \sum_{j = 1}^\nc g_{ij}^4 \left \lvert\cumu_4\left( \alge_j\right)\right \rvert \\
    = & \sum_{j = 1}^\nc \left(\sum_{i = 1}^\nc g_{ji}^4\right)\abs*{\cumu_4(\alge_j)} \\
    \leq & \sum_{j = 1}^\nc \abs*{\cumu_4(\alge_j)}.
\end{aligned}
\end{equation}
Actually, for $\Argkurt = \mixmq$, $\mixmq^T \bm y = \mixmq^T \mixmq \bm x = \bm x$, thus
\begin{equation}
\label{eq:thm:kurt3_3}
\sum_{i = 1}^\nc \left \lvert \cumu_4((\mixmq^T \bm y)_i)\right \rvert  = \sum_{i = 1}^\nc \abs*{\cumu_4(\alge_i)}.
\end{equation}
Equations \eqref{eq:thm:kurt3_3} and \eqref{eq:thm:kurt3_1} together imply \eqref{eq:thm:kurt3_2}. Then by \eqref{eq:thm:kurt3_3},  $\mixmq$ is a maximizer of \eqref{eq:kurt3}.

\subsubsection{Proof of (b)}

We first introduce several notations. For a permutation matrix $\permm = (p_{ji})_{i,j = 1}^\nc$, there is a associate permutation $\sigma$ such that $p_{\sigma(i)i} = 1$ and $p_{ji} = 0$ for all $i = 1, \cdots, \nc$ and $j \neq \sigma(i)$. For a signature matrix $\signm$, we denote its $i$-th diagonal elements by $S_i$. 

Now for any $\permm$ and $\signm$, under the light of \eqref{eq:thm:kurt3_2}, it is desired to show that \\
$\sum_{i = 1}^\nc \left \lvert \cumu_4\left(((\mixmq\permm\signm)^T \bm y)_i\right) \right \rvert =  \sum_{i = 1}^\nc \abs*{\cumu_4(\alge_i)}$.
As $\bm x$ and $\bm y$ satisfy \eqref{eq:sym fca model}, we have
\begin{equation}
\label{eq:thm:kurt3_4}
\begin{aligned}
  (\mixmq\permm\signm)^T \bm y = & \signm^T \permm^T \mixmq^T \bm y \\
     = & \signm^T \permm^T \bm x \\
     = & (S_1\alge_{\sigma(1)}, \cdots,S_\nc x_{\sigma(\nc)})^T.
\end{aligned}
\end{equation}
As $S_i \in \{\pm 1\}$, by \eqref{eq:freecumu_powerlaw}
\begin{equation}
\label{eq:thm:kurt3_6}
  \cumu_4(S_ix_{\sigma(i)}) = S_i^4 \cumu_4(x_{\sigma(i)}) = \cumu_4(x_{\sigma(i)}).
\end{equation}
Combining \eqref{eq:thm:kurt3_4} and \eqref{eq:thm:kurt3_6} together, we obtain that
\begin{equation}
\begin{aligned}
\sum_{i = 1}^\nc \left \lvert \cumu_4\left(((\mixmq\permm\signm)^T \bm y)_i\right) \right \rvert = & \sum_{i = 1}^\nc \left \lvert \cumu_4(S_i\alge_{\sigma(i)}) \right \rvert  \\
= & \sum_{i = 1}^\nc \abs*{\cumu_4(\alge_{\sigma(i)})} \\
= & \sum_{i = 1}^\nc \abs*{\cumu_4(\alge_{i})}.
\end{aligned}
\end{equation}
This completes the proof of (b).

\subsubsection{Proof of (c)}

By $(b)$, any matrix $\widehat \Argkurt$ of the form $\widehat \Argkurt = \mixmq \permm \signm$ is a maximizer.
For the other direction, we want to show that any maximizer $\widehat \Argkurt$ can be written in the this form.

Actually, if $\widehat \Argkurt$ is a maximizer, we consider $ (\widehat g_{ij})_{i,j = 1}^\nc = \widehat{\bm{G}}= \mixmq^T \widehat \Argkurt$. 
The thrid equality of \eqref{eq:thm:kurt3_1} holds with $g_{ij} = \widehat g_{ij}$. 
That is, 
\begin{equation}
\label{eq:thm:kurt3_5}
  \sum_{j = 1}^\nc \left(\sum_{i = 1}^\nc \widehat g_{ji}^4\right)\abs*{\cumu_{4}(\alge_j)} = \sum_{j = 1}^\nc \abs*{\cumu_{4}(\alge_j)}. 
\end{equation}
Since we assume the components of $\alge$ has non-zero free kurtosis (see \eqref{eq:kurt3_assumption}) and $\sum_{i = 1}^\nc \widehat g_{ji}^4 \leq 1$ for $j = 1,\cdots, \nc$, \eqref{eq:thm:kurt3_5} is equivalent to
\begin{equation}
  \sum_{i = 1}^\nc \widehat g_{ji}^4 = 1, \qquad \text{for $j = 1,\cdots, s$}.
\end{equation}
By the observation below \eqref{eq:proof_kurt_5}, for each $j$, there is a $i$ such that $\widehat g_{ji} \in \{\pm 1\}$ while $\widehat g_{jk} = 0$ for $k \neq i$. 
That is, each column of $\widehat{\bm G}$ has exactly one non-zero entry.
By Proposition \ref{prop:orth_ps_one}, $\widehat{\bm G} \in \orth_{sp}$ and thus $\widehat{\bm G} = \permm \signm$ for some permutation matrix $\permm$ and signature matrix $\signm$. 
Rcall that $\widehat \Argkurt = \mixmq \widehat{\bm G}$, we arrive at $\widehat \Argkurt = \mixmq \permm \signm$. 
This completes the proof.

\subsection{Proof of  Theorem \ref{thm:ent}}

The proof of Theorem \ref{thm:ent} relies on the orthogonal invariance and subadditivity of free entropy, for which readers are referred to Proposition \ref{prop:entlintran} and \ref{prop:entfreeind} (and their rectangular analogues in Section \ref{subsubsec:free entropy rec}).

As in the proof of Theorem \ref{thm:kurt3}, we will show the following:

\begin{itemize}
    \item[(a)] $\mixmq$ is a maximizer of \eqref{eq:ent1}. 
    \item[(b)] For any permutation matrix $\permm$ and signature matrix $\signm$, $\mixmq \signm \permm$ is a maximizer of  \eqref{eq:ent1}.
    \item[(c)] Any maximizer $\Argkurt$ of \eqref{eq:ent1} satisfies $\Argkurt = \mixmq \signm \permm$ for some permutation matrix $\permm$ and signature matrix $\signm$.
\end{itemize}

\subsubsection{Proof of (a)}
 Set $\bm Z = \bm Q^T\Argkurt$. As $\bm \alge$ and $\bm y$ are related via \eqref{eq:sym fca model}, $\bm W^T \bm y = (\bm Q \bm Z)^T \bm Q \bm \alge = \bm Z^T \bm \alge$. Then by \eqref{eq:entfreeind},
\begin{equation}
\label{eq:ent_proof0}
  \sum_{i = 1}^\nc \ent\left((\bm W^T \bm y)_i\right) = \sum_{i = 1}^\nc \ent\left((\bm Z^T \bm \alge)_i\right) \geq \ent\left((\bm Z^T \alge)_1,\cdots,(\bm Z^T \bm \alge)_\nc\right).
\end{equation}
On the other hand, note that $Z$ is an orthogonal matrix, then by \eqref{eq:ent_orth}, 
\begin{equation}
\label{eq:ent_proof5}
   \ent\left((\bm Z^T \bm \alge)_1,\cdots,(\bm Z^T \bm \alge)_\nc\right) = \ent\left(\alge_1,\cdots,\alge_\nc\right)
\end{equation}
Combining \eqref{eq:ent_proof0} and \eqref{eq:ent_proof5} together, we obtain that, for any $\Argkurt \in \orth(\nc)$,
\begin{equation}
\label{eq:ent_proof-1}
\sum_{i = 1}^\nc \ent\left((\bm W^T \bm y)_i\right)  \geq  \ent\left(\alge_1,\cdots,\alge_\nc\right)
\end{equation}

Now consider $\Argkurt = \mixmq$. 
As $\mixmq^T \bm y = \mixmq^T \mixmq \bm \alge = \bm \alge$, we have that 
\begin{equation}
\label{eq:ent_proof1}
\sum_{i = 1}^\nc \ent\left((\mixmq^T \bm y)_i\right) = \sum_{i = 1}^\nc \ent\left(\alge_i\right).
\end{equation}
On the other hand, as $\alge_i$ are freely independent, then by Proposition \ref{prop:entfreeind},
\begin{equation}
  \sum_{i = 1}^\nc \ent(\alge_i) = \ent(\alge_1,\cdots,\alge_\nc).
\end{equation}
Then \eqref{eq:ent_proof1} becomes
\begin{equation}
\label{eq:ent_proof-2}
  \sum_{i = 1}^\nc \ent\left((\mixmq^T \bm y)_i\right) = \ent\left(\alge_1,\cdots,\alge_\nc\right).
\end{equation}
Equations \eqref{eq:ent_proof-2} and \eqref{eq:ent_proof-1} together indicate
\begin{equation}
\label{eq:ent_proof3}
   \min_{\Argent \in \orth(\nc)} \sum_{i = 1}^\nc \ent\left((\bm W^T \bm y)_i\right) = \ent\left(\alge_1,\cdots,\alge_\nc\right)
\end{equation}
and $\mixmq$ is a maximizer of \eqref{eq:ent1}.

\subsubsection{Proof of (b)}

Adapt the notations introduced in the proof of Theorem \ref{thm:kurt3} (b). For any permutation matrix $\permm$ associated with permutation $\sigma$ and signature matrix $\signm = \diag(S_1,\cdots,S_\nc)$, we have that (see \eqref{eq:thm:kurt3_4})
\begin{equation}
  (\mixmq\permm\signm)^Ty = (S_1\alge_{\sigma(1)}, \cdots ,S_\nc \alge_{\sigma(n)})^T.
\end{equation}
Thus
\begin{equation}
\label{eq:ent_proof2}
  \sum_{i = 1}^\nc \ent\left(((\mixmq \permm \signm)^T y)_i\right) = \sum_{i = 1}^\nc \ent(S_i\alge_{\sigma(i)}). 
\end{equation}
As $S_i \in \{\pm 1\}$ can be regard as $1$-by-$1$  orthogonal matrices, then the $1$-dimensional verision of \eqref{eq:ent_orth_1} yields
\begin{equation}
  \ent(S_i \alge_{\sigma(i)}) = \ent(\alge_{\sigma(i)}), \qquad \text{for $i = 1, \cdots, n$}.
\end{equation}
Then \eqref{eq:ent_proof2} becomes
\begin{equation}
\sum_{i = 1}^\nc \ent\left(((\mixmq \permm \signm)^T y)_i\right)   = \sum_{i = 1}^\nc \ent(\alge_{i}). 
\end{equation}
Under the light of \eqref{eq:ent_proof3}, $\mixmq \permm \signm$ is a maximizer of \eqref{eq:ent1}.

\subsubsection{Proof of (c)}

By $(b)$, any matrix $\widehat \Argkurt$ of the form $\widehat \Argkurt = \mixmq \permm \signm$ is a maximizer.
For the other direction, it is enough to show that, any maximizer $\widehat \Argkurt$ of \eqref{eq:ent1} can be written in the form $\widehat \Argkurt = \mixmq \permm \signm$ for some permutation matrix $\permm$ and signature matrix $\signm$. 
Actually, if $\widehat \Argkurt$ maximize \eqref{eq:ent1}, then by \eqref{eq:ent_proof3},
\begin{equation}
\label{eq:ent_proof4}
\sum_{i = 1}^\nc \ent\left((\widehat \Argkurt ^T y)_i\right) = \ent\left(\alge_1,\cdots,\alge_\nc\right)
\end{equation}
Since $\widehat \Argkurt^T\mixmq$ is a orthogonal matrix, then by \eqref{eq:ent_orth} and \eqref{eq:sym fca model},
\begin{equation}
\begin{aligned}
  \ent\left(\alge_1,\cdots,\alge_\nc\right) & = \ent\left((\widehat \Argkurt^T \mixmq \alge)_1,\cdots,(\widehat \Argkurt^T \mixmq \alge)_\nc\right) \\
  &= \ent\left((\widehat \Argkurt^T y)_1,\cdots,(\widehat \Argkurt^T y)_\nc\right) 
\end{aligned}
\end{equation}
Then \eqref{eq:ent_proof4} becomes
\begin{equation}
  \sum_{i = 1}^\nc \ent\left((\widehat \Argkurt^T y)_i\right) = \ent\left((\widehat \Argkurt^T y)_1,\cdots,(\widehat \Argkurt^T y)_\nc\right) 
\end{equation}
By Proposition \ref{prop:entfreeind}, the above equation indicates that $\widehat \Argkurt^Ty$ has freely independent components. As we assume that $\alge$ has at most one semi-circular element, Theorem \ref{thm:identifiability} implies that $\widehat \Argkurt = \mixmq \permm \signm$ for some permutation matrix $\permm$ and signature matrix $\signm$. This completes the proof.

\subsection{Proof of  Theorem \ref{thm:identifiability}}

\begin{definition}
We denote all matrix of size $\nc \times \nc$ which are product of a permutation matrix and a signature matrix by
\begin{equation}
  \orth_\ps = \orth_\ps(\nc) := \{\permm\signm ~\vert ~ \text{$\permm$ is a permutation matrix, $\signm$ is a signature matrix}\}.
\end{equation}
Let $\orth := \orth(\nc)$ denotes the sets of orthogonal matrix of size $\nc \times \nc$. Note that any permutation matrix $\permm$ and signature matrix $\signm$ belong to $\orth$. 
Furthermore, it can be checked that $\orth_\ps$ is a subgroup of $\orth$.
\end{definition}

We first prove two propositions of $\orth_{\ps}$.
An orthogonal matrix must contain at least one nonzero entry in each column (and each row). 
On the other hand, the matrix belonging to $\orth_\ps$ has exactly one nonzero entry in each column (and each row). 
The following proposition states that this characterizes the matrices contained in $\orth_\ps$.
\begin{proposition}
\label{prop:orth_ps_one} Fix a positive integer $\nc \geq 1$,
$\mixmq \in \orth(\nc)$ has exactly one non-zero entry in each column if and only if $\mixmq \in \orth_\ps(\nc)$.
\end{proposition}
\begin{proof}
If $\mixmq \in \orth_{\ps}$, then $\mixmq = \permm \signm$ for some permutation matrix $\permm$ and signature matrix $\signm$.  Thus it follows that $\mixmq$ has exactly one non-zero entry in each column.

For the other direction, consider an arbitrary $\mixmq \in \orth(\nc)$ with exactly one non-zero entry in each column. 
Note that $\mixmq$ has totally $n$ non-zero entries. 
As $\mixmq$ is non-singular, it also has exactly one non-zero entry in each row.
As a result, there exists a permutation matrix $\permm$ such that $\permm^T \mixmq $ is a diagonal matrix. 

On the other hand, note that $(\permm^T\mixmq)^T(\permm^T\mixmq) = \mixmq^T\mixmq = I$, $\permm^T\mixmq$ is a diagonal orthogonal matrix. 
Thus the diagonal entries of $ \permm^T\mixmq$ are either $+1$ or $-1$. 
Then there exists a signature matrix $\signm$ such that $ \permm^T \mixmq = \signm$. 
That is equivalent to $\mixmq = \permm \signm \in \orth_{\ps}$. 
This completes the proof.
\end{proof}

By above proposition, for any $\mixmq \in \orth \backslash \orth_{\ps}$, there must be a column with more than one non-zero entry.
For the later purpose, we prove a stronger result.

\begin{proposition}
\label{prop:orth_ps}
Given any $\nc \geq 2$, consider matrix $\mixmq = (q_{ij})_{i,j= 1}^s \in \orth(\nc) \backslash \orth_\ps(\nc)$. 
Then there is a $2\times 2$ submatrix of $\mixmq$ with all $4$ entries non-zero. 
Explicitly, there exist $i,j,k,\ell \in \{1,\cdots,n\}$ ($i \neq j$, $k \neq \ell$) such that all $q_{ik}$, $q_{i\ell}$, $q_{jk}$, and $q_{j\ell}$ are non-zero.
\end{proposition}
\begin{proof}


We first make the following observation. Two orthogonal vectors either share $0$ or more than $2$ positions for non-zero entries.
Actually, consider any $u = (u_1,\cdots,u_\nc)^T$ and $v = (v_1,\cdots,v_\nc)^T$ such that $u$ and $v$ are orthogonal.
Assume that there is exactly one index $k$ such that both $u_k$ and $v_k$ are non-zero, then 
\begin{equation}
  u^T v = \sum_{i = 1}^\nc u_iv_i = u_kv_k \neq 0.
\end{equation}
This contradicts the fact that $u^T v = 0$.

Now we are ready to prove the proposition.
Denote $i$-th columns of $\mixmq$ by $\mixmq_i$, for $i = 1,\cdots,\nc$. 
Note that the $\{\mixmq_i\}_{i = 1}^\nc$ form an orthonormal basis.
As $\mixmq \in \orth(\nc) \backslash \orth_\ps(\nc)$, there must be a column containing more than two non-zero entries.
Without lose of generality, assume it is $\mixmq_1$.
If all $\mixmq_2,\cdots,\mixmq_{\nc}$ share $0$ positions of non-zero entry with $\mixmq_1$, then $\{\mixmq_i\}_{i= 2}^\nc$ span a linear space of dimension less than $n - 2$. 
This contradicts with the fact that $\{\mixmq_i\}_{i= 2}^\nc$ span a linear space of dimension $\nc - 1$.
Thus there must exist a $j \in \{2,\cdots,\nc\}$ such that $\mixmq_1$ and $\mixmq_j$ share at least one positions for non-zero entry. 
By the observation we made in the last paragraph, $\mixmq_1$ and $\mixmq_j$ then share at least two positions of non-zero entry. 
This completes the proof.
\end{proof}

\begin{corollary}
\label{coro:entrynonzero}
Fix a positive integer $n \geq 2$ and a $\mixmq \in \orth(\nc) \backslash \orth_\ps(\nc)$. There exists indexes $i,j,k,\ell \in [1,..,\nc]$ ($i \neq j$ and $k \neq \ell$), such that for any $m \geq 3$, 
\begin{equation}
\label{eq:entrynonzero}
 q_{ik}^{m - 1}q_{jk} \neq 0, \quad \text{and} \quad q_{i\ell}^{m - 1}q_{j\ell} \neq 0. 
\end{equation}
In particular, if $\nc = 2$, then for any $m \geq 3$,
\begin{equation}
\label{eq:entrynonzero2}
 q_{11}^{m - 1}q_{21} \neq 0, \quad \text{and} \quad q_{12}^{m - 1}q_{22} \neq 0. 
\end{equation}
\end{corollary}

Theorem \ref{thm:identifiability} can be obtained as a corollary of the following lemma.
\begin{lemma}
\label{free_darmois}
Fix a $\nc \geq 2$, let \(\bm \alge = (\alge_1,\alge_2, \cdots, \alge_\nc)^T\) and \(\bm y = (y_1,y_2, \cdots, y_\nc)^T\) be two random vectors such that \(\bm y = \mixmq \bm \alge\), where \(\mixmq \in \orth(\nc)\). Assume $(\alge_i)_{i = 1}^\nc$ are freely independent. Now if $(y_i)_{i = 1}^\nc$ are freely independent, 
then at least one of the following happens:
\begin{enumerate}
    \item[(a)] $\mixmq \in \orth_\ps(\nc)$.
    \item[(b)] At least two components of $\bm \alge$ are semicircular (or Poisson in the non-self-adjoint setting). 
\end{enumerate}
\end{lemma}
We first show that Theorem \ref{thm:identifiability} follows from Lemma \ref{free_darmois}.
\begin{proof}[Proof of Theorem \ref{thm:identifiability}]
As $\bm \alge$ and $\bm y$ satisfy \eqref{eq:sym fca model},  $\bm \alge = (\mixmq^T \Argkurt) \Argkurt^T \bm y$.
Now, by assumption, $\bm \alge$ and $\Argkurt^T \bm y$ have free components.  
Then according to Lemma \ref{free_darmois}, there are two possibilities: (a) $\mixmq^T \Argkurt  \in \orth_\ps$ or (b) $\bm \alge$ has at least two semicircular components.
As (b) has been excluded, (a) happens. 
That is, there exist a permutation matrix $\permm$ and a signature matrix $\signm$ such that $\mixmq^T \Argkurt  = \permm \signm$, i.e., $\Argkurt = \mixmq  \permm \signm$.
\end{proof}

\begin{proof}[Proof of Lemma \ref{free_darmois}]
We first consider self-adjoint setting.
If $\mixmq \in \orth_\ps(\nc)$, then the components of $\bm y$ are exactly the components of $\bm \alge$ with different order and possible sign change. It is not surprising that  $y_i$ are freely independent. 
In the following, we assume that $\mixmq \in \orth(\nc) \backslash \orth_\ps(\nc)$, and $\bm \alge, \bm y$ has free components, the goal is to show that $\bm \alge$ has at least two semicircular elements.

We start with the case where $n = 2$.
Then it is desired to show $\alge_1$ and $\alge_2$ are both semicircular elements.
Recall the Definition \ref{def:semi} for the semicircular element, it is enough to show $\cumu_m(\alge_i) \equiv 0$ for all $m\geq 3$ and $i = 1,2$.

Fix $m \geq 3$, we consider the mixed cumulants of $y_1,y_2$ of the specific form \\
$\cumu_m(y_1,\cdots,y_1, y_2, y_p)$ for $p = 1,2$. 
As $y_1,y_2$ are free-independent, these cumulants satisfies the condition of Theorem \ref{thm:vanish_mix_cumu} by noting that $i(1) = 1 \neq i(m - 1) = 2$. 
Thus these mixed cumulants vanishes, i.e.,
\begin{equation}
\label{eq:ident4}
\cumu_m(y_1,\cdots,y_1, y_2, y_p) = 0\quad \mathrm{for}~p = 1, 2.
\end{equation}

On the other hand, as $(y_i)_{i = 1}^2$ are linear combinations of $(\alge_i)_{i = 1}^2$, using multi-linearity of $\cumu_m(\cdot)$ (see \eqref{eq:multilinear}), we will express $\cumu_m(y_1,\cdots,y_1, y_2, y_p)$ as linear combinations of $\cumu_m(\alge_i)$ (recall the notation  \eqref{eq:cumu_ident_var}). 
Adapt the notation $\mixmq = (q_{ij})_{i,j= 1}^2$, then $y_i = \sum_{j = 1}^2 q_{ij}\alge_j$. We first derive the expression for $\cumu_m(y_1,\cdots,y_1, y_2, y_1)$ (i.e., $p = 1$),
\begin{equation}
\label{eq:ident1}
\cumu_m(y_1,\cdots,y_1, y_2, y_1) = \cumu_m\left(\sum_{j = 1}^2 q_{1j}\alge_{j}, \cdots, \sum_{j = 1}^2 q_{1j}\alge_{j}, \sum_{j = 1}^2 q_{2j}\alge_{j}, \sum_{j = 1}^2 q_{1j}\alge_{j}\right).
\end{equation}
Apply \eqref{eq:multilinear} to the right hand side of \eqref{eq:ident1} to expand the first variable,
\begin{equation}
\begin{aligned}
\cumu_m & (y_1,\cdots,y_1, y_2, y_1)  = \\ 
& \sum_{j_1 = 1}^2 q_{1j_1}  \cumu_m\left(\alge_{j_1}, \sum_{j = 1}^2 q_{1j}\alge_{j},\cdots, \sum_{j = 1}^2 q_{1j}\alge_{j}, \sum_{j = 1}^2 q_{2j}\alge_{j}, \sum_{j = 1}^2 q_{1j}\alge_{j}\right).
\end{aligned}
\end{equation}
Again apply \eqref{eq:multilinear} for the second variable, we obtain that
\begin{equation}
\begin{aligned}
\cumu_m( & y_1,\cdots,y_1, y_2, y_1) = \\
& \sum_{j_1 = 1}^2\sum_{j_2 = 1}^2 q_{1j_1} q_{1j_2} \cumu_m\left(\alge_{j_1}, \alge_{j_2},\cdots, \sum_{j = 1}^2 q_{1j}\alge_{j}, \sum_{j = 1}^2 q_{2j}\alge_{j}, \sum_{j = 1}^2 q_{1j}\alge_{j}\right).
\end{aligned}
\end{equation}
Repeat applying \eqref{eq:multilinear} for the rest $n-2$ variables, we arrive at
\begin{equation}
\label{eq:ident2}
\cumu_m(y_1,\cdots,y_1, y_2, y_1) = \sum_{j_1 = 1}^2\cdots \sum_{j_n = 1}^2\left( \prod_{\ell = 1}^{\nc - 2}q_{1 j_{\ell}} \right)q_{2j_{\nc - 1}}q_{1j_\nc} \cumu_m(\alge_{j_1}, \cdots, \alge_{j_\nc}).
\end{equation}
There are in total $2^\nc$ terms in above summation. 
Note that $\alge_1$ and $\alge_2$ are free independent. Then by Theorem \ref{thm:vanish_mix_cumu}, most of these cumulants vanish. 
For example, $\cumu_{\nc}(\alge_1,\alge_2,\cdots \alge_2) = 0$ where $j_1 = 1 \neq j_2 = 2$. Consequently, there are only two terms corresponding to the choices of indexes $j_1 = j_2 = \cdots = j_\nc = 1$ and $j_1 = j_2 = \cdots = j_\nc = 2$ survive. 
Thus using the notation \eqref{eq:cumu_ident_var}, \eqref{eq:ident2} can be written as
\begin{equation}
\label{eq:ident3}
  \cumu_m(y_1,\cdots,y_1, y_2, y_1) = q_{11}^{m-2}q_{21}q_{11}\cumu_m(\alge_1) + q_{12}^{m-2}q_{22}q_{12}\cumu_m(\alge_2).
\end{equation}
Combining \eqref{eq:ident3} with \eqref{eq:ident4}, we obtain that
\begin{equation}
\label{eq:ident5}
  q_{11}^{m-2}q_{21}q_{11}\cumu_m(\alge_1) + q_{12}^{m-2}q_{22}q_{12}\cumu_m(\alge_2) = 0.
\end{equation}
Repeat \eqref{eq:ident4} to \eqref{eq:ident5} for $\cumu_m(y_1,\cdots,y_1, y_2, y_2)$ (i.e., $p = 2$), we find that
\begin{equation}
\label{eq:ident6}
  q_{11}^{m-2}q_{21}q_{21}\cumu_m(\alge_1) + q_{12}^{m-2}q_{22}q_{22}\cumu_m(\alge_2) = 0.
\end{equation}
Writing \eqref{eq:ident5} and \eqref{eq:ident6} in the matrix form, we obtain that
\begin{equation}
    \label{eq:ident7}
    \begin{aligned}
      \begin{pmatrix}
        q_{11}^{m-2}q_{21}q_{11} & q_{12}^{m-2}q_{22}q_{12}\\q_{12}^{m-2}q_{21}q_{21}&q_{12}^{m-2}q_{22}q_{22}
        \end{pmatrix} & 
        \begin{pmatrix}
          \cumu_m(\alge_1) \\
            \cumu_m(\alge_2) 
        \end{pmatrix} = \\ 
        \begin{pmatrix}q_{11} & q_{12} \\ q_{21} & q_{22} \end{pmatrix} & \begin{pmatrix} 
        q_{11}^{m-2}q_{21} & 0 \\ 0 & q_{12}^{m-2}q_{22}
\end{pmatrix}
\begin{pmatrix}
          \cumu_m(\alge_1) \\
            \cumu_m(\alge_2) 
        \end{pmatrix} = \vec{0}
\end{aligned}
    \end{equation}
We actually get a linear equation system for $\cumu_m(\alge_1)$ and $\cumu_m(\alge_2)$. 
Note that $\mixmq = (q_{ij})_{i = 1}^2$ is an orthogonal matrix and thus is invertible. 
Thus \eqref{eq:ident7}  is equivalent to
\begin{equation}
  \begin{pmatrix} 
        q_{11}^{m-2}q_{21} & 0 \\ 0 & q_{12}^{m-2}q_{22}
\end{pmatrix}
\begin{pmatrix}
          \cumu_m(\alge_1) \\
            \cumu_m(\alge_2) 
        \end{pmatrix} = \vec{0}.
\end{equation}

Now, as $ \mixmq \in \orth(2) \backslash \orth_\ps(2)$, then by \eqref{eq:entrynonzero2},
above linear equation system has a unique solution, $\cumu_m(\alge_i) = 0$, $i = 1,2$. 
Note that this holds for all $m \geq 3$.
Then by Definition \ref{def:semi}, we conclude that $\alge_i$ for $i = 1, 2$ are semicircular elements. This concludes the proof for $n = 2$.
    
For general $n \geq 2$, as \(\mixmq \in \orth \backslash \orth_\ps \), by Corollary \ref{coro:entrynonzero}, there exist $i,j,k,\ell$ ($i\neq j$ and $k \neq \ell$) such that \eqref{eq:entrynonzero} holds. We will show that $\alge_k,\alge_\ell$ are semicircular elements. For fixed $m \geq 3$, we consider the vanishing mixed cumulants 
    \begin{equation}
    \cumu_m(y_i,\cdots,y_i, y_j, y_p) = 0\quad \mathrm{for}~p = 1, \cdots, \nc.
    \end{equation}
Use relation $y_i = \sum_{j = 1}^\nc q_{ij}\alge_j$ and multilinearity of $\cumu_m$, we can repeat \eqref{eq:ident4} to \eqref{eq:ident5} for each $\cumu_m(y_i,\cdots,y_i, y_j, y_p)$ and get
\begin{equation}
  q_{i1}^{m-1}q_{j1}q_{p1} \cumu_m(\alge_1) + \cdots + q_{i\nc}^{m-1}q_{j\nc}q_{p\nc} \cumu_m(\alge_\nc) = 0, \quad \text{for $p = 1,\cdots,\nc$.}
\end{equation}
Write above equations in the matrix form:
\begin{equation}
    \label{eq:ident8}
  \begin{pmatrix}q_{11} & \cdots & q_{1\nc} \\ \vdots & \ddots & \vdots \\q_{\nc 1} & \cdots & q_{\nc \nc} \end{pmatrix}\begin{pmatrix} 
        q_{i1}^{m-2}q_{j1} & &  \\ & \ddots & \\ & & q_{i\nc}^{m-2}q_{j\nc}
\end{pmatrix}
\begin{pmatrix}
          \cumu_m(\alge_1) \\
            \vdots\\
            \cumu_m(\alge_\nc) 
        \end{pmatrix} = \vec{0}.
\end{equation}
Again, $\mixmq = (q_{ij})_{i = 1}^\nc$ is invertible and $q_{ik}^{m-1}q_{jk} \neq 0$ (see \eqref{eq:entrynonzero}), thus $\cumu_m(\alge_k) = 0$. 
For the same reason, $\cumu_m(\alge_\ell) = 0$.  As these hold for all $m \geq 3$, $\alge_k,\alge_\ell$ are semicircular elements.

 For non-self-adjoint setting, the proof is exactly the same as above with with Theorem \ref{thm:vanish_mix_cumu} replaced by Theorem \ref{thm:vanish_mix_cumu_rec} and Definition \ref{def:semi} replace by Definition \ref{def:poisson}.
\end{proof}



\section{Proof of Theorem \ref{thm:grad_F}}
\begin{lemma}
\label{lemma:devr_cumu}
Given $\bm{Y} = [\bm{Y}_1,\cdots,\bm{Y}_\nc]^T \in \C^{N\nc\times N}$ with $\bm{Y}_i \in \C^{N \times N}$ Hermitian matrices and a vector $\bm w = [w_1,\cdots,w_s] \in \R^s$, for
$$
     \bm X = \widetilde{\bm{w}}^T \bm Y, \quad \text{with $\widetilde{\bm{w}} = \bm w \otimes \bm{I}_N$},
$$
we recall the empirical free kurtosis
$$
    \widehat{\cumu}_4(\bm X) =  \frac{1}{N} \Tr(\bm X^4) - 2\left[\frac1N\Tr(\bm X^2)\right]^2.
$$
Then we have that
\begin{equation}
\label{eq:thm_devr_cumu}
    \frac{\partial \widehat{\cumu}_4(\bm X)}{\partial w_k} =  \frac4N \Tr(\bm Y_i\bm X^3) - \frac{8}{N^2} \Tr(\bm X^2) \Tr(\bm Y_i\bm X).
\end{equation}
\end{lemma}
\begin{proof}
As $\Tr(\cdot)$ is a linear function of entries of input matrix,
\begin{equation}
\label{eq:proof_derv_cumu_0}
    \frac{\partial \widehat{\cumu}_4(\bm X)}{\partial w_k} = \frac1N \Tr\left(\frac{\partial \bm X^4}{w_k}\right) - \frac{4}{N^2} \Tr(\bm X^2) \Tr\left(\frac{\partial \bm X^2}{\partial w_k}\right).
\end{equation}

Note that
$$
    \bm X = \widetilde{\bm {w}}^T \bm Y = w_1 \bm Y_1 + \cdots + w_\nc \bm Y_\nc,
$$
thus, for any $k = 1, \cdots, \nc,$
\begin{equation}
\label{eq:proof_derv_cumu_4}
    \frac{\partial \bm X}{\partial w_k} = \bm Y_k.
\end{equation}
Therefore,
\begin{equation}
\label{eq:proof_derv_cumu_1}
    \frac{\partial \bm X^4}{\partial w_k} =\bm Y_k\bm X^3 + \bm X \bm Y_k\bm X^2 + \bm X^2 \bm Y_k \bm X + \bm X^3 \bm Y_k.
\end{equation}
Using $\Tr(AB) = \Tr(BA)$, we find that
$$\Tr(\bm Y_k \bm X^3) = \Tr(\bm X \bm Y_k\bm X^2) = \Tr(\bm X^2 \bm Y_k \bm X) = \Tr(\bm X^3 \bm Y_k)$$
and thus
\begin{equation}
\label{eq:proof_derv_cumu_2}
    \Tr\left(\frac{\partial \bm X^4}{w_k}\right) = 4\Tr(\bm Y_k \bm X^3).
\end{equation}

Repeat \eqref{eq:proof_derv_cumu_1} to \eqref{eq:proof_derv_cumu_2} for $\Tr\left(\frac{\partial \bm X^2}{\partial w_k}\right)$, we get that
\begin{equation}
\label{eq:proof_derv_cumu_3}
    \Tr\left(\frac{\partial \bm X^2}{\partial w_k}\right) = 2\Tr(\bm Y_k \bm X).
\end{equation}
Plug \eqref{eq:proof_derv_cumu_2} and \eqref{eq:proof_derv_cumu_3} into \eqref{eq:proof_derv_cumu_0}, we obtain \eqref{eq:thm_devr_cumu}.

\end{proof}

\begin{lemma}
\label{lemma:devr_ent}
Given $\bm Y = [\bm Y_1,\cdots,\bm Y_\nc]^T \in \C^{N\nc\times N}$ with $\bm Y_i \in \C^{N \times N}$ are Hermitian matrices and a vector $w = [w_1,\cdots,w_s] \in \R^s$. For 
$$
     \bm X = \widetilde{\bm{w}}^T \bm Y, \quad \text{with $\widetilde{\bm{w}} = w \otimes I_N$},
$$
with eigenvalues $\eg_i$ and corresponding eigenvectors $v_i$, we recall the empirical free entropy
$$
    \widehat{\ent}(\bm X) =  \frac{1}{N(N - 1)}\sum_{i \neq j} \log \abs*{\eg_i - \eg_j}.
$$
Then we have that
\begin{equation}
\label{eq:thm_devr_ent}
    \frac{\partial \widehat{\ent}(\bm X)}{\partial w_k} =  \frac{1}{N(N - 1)} \sum_{i \neq j} \frac{\partial_{w_k} \eg_i - \partial_{w_k} \eg_j}{\eg_i - \eg_j}
\end{equation}
with $\partial_{w_k}\eg_i = v_i^T \bm Y_k v_i$.
\end{lemma}
\begin{proof}
Equation \eqref{eq:thm_devr_ent} is obtained by directly taking derivative.  
The fact that $\partial_{w_k}\eg_i = v_i^T \bm Y_k v_i$ follows from \eqref{eq:proof_derv_cumu_4} and perturbation theory of eigenvalues \cite{eig_perturb}.
\end{proof}

\begin{proof}[Proof of Theorem \ref{thm:grad_F}] 
We first prove the result for self-adjoint FCF based on free kurtosis.
Set $\bm{X} = [\bm{X_1},\cdots,\bm{X_\nc}] = \widetilde{\bm{W}}^T \bm{Y}$. Recall Definition \ref{def:dot_operator} and \eqref{eqn:dot_sum}, for $\widehat F(\cdot) = -\left \lvert \widehat{\cumu}_4(\cdot) \right \rvert$, we have that
$$
       \sum \widehat F_{\cdot} \left(\widetilde{\bm{W}}^T \bm{Y}\right) = - \sum_{i = 1}^\nc  \left \lvert \widehat{\cumu}_4\left(\bm{X}_i\right)\right \rvert
$$
As only $X_\ell$ explicitly depends on $\bm{W}_{k\ell}$,
\begin{equation}
\label{eq:proof_grad_sym_1}
     \partial_{\bm{W}_{k\ell}} \sum \widehat F_{\cdot} \left(\widetilde{\bm{W}}^T \bm{Y}\right) = -\partial_{\bm{W}_{k\ell}} \left \lvert \widehat{\cumu}_4\left(\bm{X}_\ell\right)\right \rvert
\end{equation}
Further notice that $\bm{X}_\ell = \widetilde{\bm{w}}_\ell^T \bm{Y}$ with $\bm{w}_\ell = [\bm{W}_{1\ell}, \cdots, \bm{W}_{\nc \ell}]^T$, thus
\begin{equation}
\label{eq:proof_grad_sym_2}
    \begin{aligned}
        \partial_{\bm{W}_{k\ell}} \left \lvert \widehat{\cumu}_4\left(\bm{X}_\ell\right)\right \rvert & = \mathrm{sign}(\widehat{\cumu}_4\left(\bm{X}_\ell\right)) \times \partial_{\bm{W}_{k\ell}}  \widehat{\cumu}_4\left(\bm{X}_\ell\right) \\
        & = \mathrm{sign}(\widehat{\cumu}_4\left(\bm{X}_\ell\right)) \left(\frac4N \Tr(\bm{Y}_k\bm{X}_\ell^3) - \frac{8}{N^2} \Tr(\bm{X}_\ell^2) \Tr(\bm{Y}_k\bm{X}_\ell)\right),
    \end{aligned}
\end{equation}
where we used Lemma \ref{lemma:devr_cumu} for the last equality.
The proof is then completed by plugging \eqref{eq:proof_grad_sym_2} into \eqref{eq:proof_grad_sym_1}.
The result for self-adjoint FCF based on free entropy can be proved in a similar manner by repeating the process from \eqref{eq:proof_grad_sym_1} to \eqref{eq:proof_grad_sym_2}, where we replace $-\abs*{\hat \cumu_4(\cdot)}$ with $\ent(\cdot)$ and Lemma \ref{lemma:devr_cumu} with Lemma \ref{lemma:devr_ent}.

We omit the proofs for the rectangular FCF column since these are straightforward modifications of the proofs of Lemma \ref{lemma:devr_cumu}, \ref{lemma:devr_ent} and proofs of the self-adjoint FCF case.
\end{proof}

\section{Matrix Embeddings} \label{sec:mat embed}
One restriction of ICA is that it only operates on vector-valued components (see Section \ref{sec:algica}).
In contrast, FCF applies to data whose matrix-valued components that can be of arbitrary dimensions.
Thus, one can embed components into new dimensions potentially obtain a better performance with FCA.
In this section, we list several matrix embedding algorithms.

For $\bm{Z} = [\bm{Z}_1, \cdots, \bm{Z}_N]^T$ where the $\bm Z_i$ are rectangular matrices, Algorithm \ref{alg:embed_sym} embeds $Z_i$ in the upper diagonal parts of a $N'\times N'$ self-adjoint matrices.
In practice, the target dimension $N'$ should be picked such that there no loss of information while also avoiding too many artificial zeros. 
To embed $\bm Z_i$ into rectangular matrices of other dimensions, we introduce Algorithm \ref{alg:embed_rec}. Putting the above embeddings and appropriate FCF algorithms together, we get Algorithm \ref{alg:embed_sym_fcf} and Algorithm \ref{alg:embed_rec_fcf}.
One easily state the analogs of the above algorithms for data containing self-adjoint matrices; for the sake of brevity, we omit them here.

If the $\bm Z_i$ are vectors, one can use the STFT to embed them into matrices.
The STFT matrices of a vector is the alignment of the discrete Fourier transform of a sliding window. The outcome is a complex rectangular matrix to which we can apply rectangular FCFs. This is summarized in Algorithm \ref{alg:spec_fcf}.

\begin{algorithm}
\caption{Symmetric Embedding}
\label{alg:embed_sym}
\begin{spacing}{1.5}
        \textbf{Input:} $\bm{Z} = [\bm{Z}_1, \cdots, \bm{Z}_N]^T$ where $\bm{Z}_i \in \C^{N \times M}$.\\
        \textbf{Input:} Target dimension $N' \times N'$.  
        \\
        1. Draw $S$ uniformly from all subsets of $\{1,\cdots, \frac{N'(N' - 1)}{2}\}$ with size $\frac{N'(N' - 1)}{2} - NM$.\\
        2. \textbf{for} $i = 1, \cdots, \nc$ \\
        3. $~~~~$ Construct $z' \in \R^{\frac{N'(N' - 1)}{2}}$ by setting $z'[S] = 0$ and $z'[S^c] = vec(\bm{Z}_i)$. \\
        4. $~~~~$ Fill upper diagonal part of zero matrix $\bm{Z}' \in \C^{N'\times N'}$ with $(\bm{Z}'_{kl})_{l>k} = z'$.\\
        5. $~~~~$ Construct self-adjoint matrix $\bm{Z}_i' = \bm{Z}' + (\bm{Z}')^H$.\\
        6. \textbf{end for}\\
        7. \textbf{return:} $\bm{Z}' = [\bm{Z}_1', \cdots, \bm{Z}_\nc']^T$.
\end{spacing}
\end{algorithm}

\begin{algorithm}
\caption{Symmetric Embedding FCF}
\label{alg:embed_sym_fcf}
\begin{spacing}{1.5}
        \textbf{Input:} $\bm Z = [\bm Z_1, \cdots, \bm Z_N]^T \in \C_{\nc N \times M}$ where $\bm Z_i \in \C^{N \times M}$. \\
        \textbf{Input:} Target dimension $N'$ such that $ \frac{N'(N'-1)}{2}\geq NM$. \\
        1. Apply Algorithm \ref{alg:embed_sym} to $\bm Z$ and find $\bm Z'$.\\
        2. Apply Algorithm \ref{alg:fcf_pro} to $\bm Z'$ and find estimated mixing matrix $\widehat{\bm A}$.\\
        3. Compute $\widehat{\bm X} = (\widehat{\bm{A}}^{-1}\otimes \bm{I}_N)\bm{Z}$ such that $\bm{Z} = (\widehat{\bm{ A}} \otimes \bm{I}_N) \widehat{\bm{X}}$.\\
        4. \textbf{return:} $\widehat{\bm{A}}$ and $\widehat{\bm{X}}$. 
\end{spacing}
\end{algorithm}

\begin{algorithm}
\caption{Rectangular Embedding}
\label{alg:embed_rec}
\begin{spacing}{1.5}
        \textbf{Input:} $\bm Z = [\bm Z_1, \cdots, \bm Z_N]^T \in \C_{\nc N \times M}$ where $\bm Z_i \in \C^{N \times M}$. \\
                \textbf{Input:} Target dimension $N'$ and $M'$ such that $ N'M'\geq NM$. \\
        1. Draw $S$ uniformly from all subsets of $\{1,\cdots, N'M'\}$ with size $N'M' - NM$.\\
        2. \textbf{for} $i = 1, \cdots, \nc$ \\
        4. $~~~~$ Construct $z' \in \R^{N'M'}$ by setting $z'[S] = 0$ and $z'[S^c] = vec(\bm{Z}_i)$. \\
        5. $~~~~$ Reshape $z'$ to $\bm{Z}_i' \in \C^{N'\times M'}$.\\
        6. \textbf{end for}\\
        7. \textbf{return:} Return $\bm Z' = [\bm Z_1', \cdots, \bm Z'_\nc]^T$
\end{spacing}
\end{algorithm}

\begin{algorithm}
\caption{Rectangular Embedding FCF}
\label{alg:embed_rec_fcf}
\begin{spacing}{1.5}
         \textbf{Input:} $\bm Z = [\bm Z_1, \cdots, \bm Z_N]^T \in \C_{\nc N \times M}$ where $\bm Z_i \in \C^{N \times M}$. \\
         \textbf{Input:} Target dimension $N'$ and $M'$ such that $ N'M'\geq NM$ \\
        1. Apply Algorithm \ref{alg:embed_rec} to $\bm Z$ and get $\bm Z'$.\\
        2. Apply Algorithm \ref{alg:fcf_pro} to $\bm Z' = [\bm Z_1',\cdots, \bm Z_N']^T$ and get the estimated mixing matrix $\widehat{\bm A}$.\\
        3. Compute $\widehat{\bm X} = (\widehat {\bm A}^{-1}\otimes \bm I_N) \bm{Z}$ such that $\bm{Z} = (\widehat{\bm{A}} \otimes \bm{I}_N) \widehat{\bm{X}}$.\\
        4. \textbf{return:} $\widehat{\bm{A}}$ and $\widehat{\bm{X}}$. 
\end{spacing}
\end{algorithm}

\begin{algorithm}
\caption{Short Time Fourier Transform Embedding FCF}
\label{alg:spec_fcf}
\begin{spacing}{1.5}
         \textbf{Input:} $\bm Z = [\bm Z_1, \cdots, \bm Z_N]^T \in \C_{\nc \times N}$ where $\bm Z_i \in \C^{1 \times N}$. \\
         \textbf{Input:} Necessary parameters for STFT \\
        1. For each $\bm Z_i$, for $i = 1, \cdots, \nc$, compute the STFT matrices $\bm Z_i'$.\\
        2. Apply Algorithm \ref{alg:fcf_pro} to $\bm Z' = [\bm Z_1',\cdots, \bm Z_N']^T$ and get the estimated mixing matrix $\widehat{\bm A}$.\\
        3. Compute $\widehat{\bm X} = (\widehat{\bm A}^{-1}\otimes \bm I_N) \bm Z$ such that $\bm Z = (\widehat{ \bm A} \otimes \bm {I}_N) \widehat {\bm X}$.\\
        4. \textbf{return:} $\widehat{\bm A}$ and $\widehat{\bm X}$. 
\end{spacing}
\end{algorithm}

\section{Independent Component Factorization}
\label{sec:algica}
We would like to numerically compare FCA with ICA, and begin by providing a summary of the ICA algorithm.
Given data whose components are rectangular matrices, we first vectorize them and then apply ICA. We once again perform a whitening process (see Algorithm \ref{alg:whiten_ica}) and solve an optimization problem.

Here, we present Algorithm \ref{alg:icf_kurt} whose optimization problem is based on the empirical (scalar) kurtosis $\widehat \clc_4(\cdot)$ or the empirical (scalar) negentropy $\widehat{\mathcal E}(\cdot)$.
We call them kurtosis-based ICF and entropy-based ICF respectively.
Given a centered and whitened vector $x \in \R^T$, its empirical kurtosis $\widehat \clc_4(x)$ can be expressed as
\begin{equation} \label{eq:def_clc}
   \widehat \clc_4(x) = \frac{1}{T}\sum_{i = 1}^{T} x_i^4 - 3 \left(\frac1T \sum_{i = 1}^T x_i^2\right)^2.
\end{equation}

The negentropy ${\mathcal E}(x)$ is defined as 
\begin{equation}
    \mathcal E(x) = h(g_x) - h(x),
\end{equation}
where $h(x)$ denotes the entropy of random variable $x$ (see \eqref{eq:scalar ent}) and $g_x$ denote the Gaussian random variable with the same mean and variance as $x$.
It is used as a measure of distance to normality.
The empirical negentropy $\widehat{\mathcal E}(x)$ involves the empirical distribution of $x$, which is computationally difficult.
Fortunately, it can also be expressed as a infinite sum of cumulants.
Thus in practice, $\widehat{\mathcal E}(x)$ can be approximated by a finite truncation of that sum  \citep[Theorem 14 and (3.2) pp. 295]{comon_1994}.

In the simulation of this paper, we adapt the following approximation (see Section 5 of \citep{hyvarinen2004independent}):
\begin{equation}\label{eq:entropy cumulant}
  \widehat{\mathcal E}(x) = \frac{1}{12}\left(\frac1T\sum_{i = 1}^T x_i^3\right)^2 + \frac{1}{48}  \widehat \clc_4(x)  = \textrm{also cumulants}
\end{equation}

\begin{algorithm}[H]
\caption{Reshape and whitening}
\label{alg:whiten_ica}
\begin{spacing}{1.5}
        \textbf{Input:} $\bm Z = [\bm Z_1, \cdots, \bm Z_N]^T \in \C_{\nc N \times M}$ where $\bm Z_i \in \C_{N \times M}$. \\
        1. For $\bm z = [z_1, \cdots, z_\nc]^T$, where $z_i = vec(\bm Z_i)$, Compute $\mu_z = \text{mean}(z, 2)$ and $\widetilde {\bm z} = \bm z - \mu_z 1_{NM}^T$.
        \\
        3. Compute $\bm C = \frac{1}{NM}\widetilde {\bm z} \widetilde{\bm {z}}^H$ and the eigenvalue decomposition $\Re \bm C = \bm U \bm \Sigma^2 \bm U^T$.\\
        4. Compute $\bm y = \bm U \bm \Sigma^{-1}  \bm U^T \widetilde{\bm z}$.\\
        5. \textbf{return:} $\bm y, \bm \Sigma, \bm U$. 
\end{spacing}
\end{algorithm}

\begin{algorithm}[H]
    \caption{Prototypical ICF}
    \label{alg:icf_kurt}
\begin{spacing}{1.5}
    \textbf{Input}: $\bm Z = [\bm Z_1, \cdots, \bm Z_n]^T \in \C_{\nc N \times M}$ where $\bm Z_i \in \C_{N \times M}$ \\
    1. Compute $\bm y, \bm \Sigma, \bm U$ by applying Algorithm \ref{alg:whiten_ica} to $Z$. \\
    2. Compute 
    $$\widehat{\bm W} = \mathop{\text{arg min}}_{\bm W \in O(n)} \sum \widehat F_{\cdot} \left(\bm W^T \bm y\right),$$
    where $\widehat F(\cdot)$ is equal to $-\abs*{\widehat \clc_4(\cdot)}$ for kurtosis-based ICF or $-\widehat {\mathcal E}(\cdot)$ for entropy-based ICF.
    \\
    3. Compute $\widehat{\bm A} = \bm U \bm \Sigma \bm U^T \widehat{\bm W}$ and $\widehat{\bm X} = (\widehat{\bm A}^{-1}\otimes \bm I_N)\bm Z$.
        \\
    4. Sorting components of $\widehat{\bm X}$ by kurtosis or entropy. Permute the columns of $\widehat{\bm A}$ correspondingly. \\
    5. \textbf{return:} $\widehat{\bm A}$ and $\widehat{\bm X}$.
\end{spacing}
\end{algorithm}




\bibliography{main}

\end{document}